\documentclass[11pt]{article}

\usepackage{epsfig,epsf,fancybox}
\usepackage{amsmath}
\usepackage{mathrsfs}
\usepackage{amssymb}
\usepackage{graphicx}
\usepackage{color}
\usepackage{multirow}
\usepackage{paralist}
\usepackage{verbatim}
\usepackage{galois}
\usepackage{algorithm}
\usepackage{algorithmic}
\usepackage{boxedminipage}
\usepackage{booktabs}
\usepackage{float}
\usepackage{subfigure}
\usepackage{stmaryrd}
\usepackage{wrapfig}
\usepackage[bottom]{footmisc}

\usepackage{hyperref}

\usepackage{nicefrac}
\usepackage[left=1in,top=1in,right=1in,bottom=1in,letterpaper]{geometry}

\newtheorem{theorem}{Theorem}[section]
\newtheorem{corollary}[theorem]{Corollary}
\newtheorem{lemma}[theorem]{Lemma}
\newtheorem{proposition}[theorem]{Proposition}
\newtheorem{definition}[theorem]{Definition}
\newtheorem{remark}[theorem]{Remark}
\newtheorem{assumption}[theorem]{Assumption}

\numberwithin{equation}{section}

\providecommand{\keywords}[1]{\textbf{\textit{Keywords---}} #1}
\newcommand{\norm}[1]{\left\lVert#1\right\rVert}

\newcommand{\BE}{\mathbb{E}}

\newcommand{\x}{\mathbf x}

\newcommand{\e}{\mathbf e}

\newcommand{\argmin}{\mathop{\rm argmin}}

\newcommand{\br}{\mathbb{R}}

\newcommand{\be}{\begin{equation}}
\newcommand{\ee}{\end{equation}}
\newcommand{\ba}{\begin{array}}
\newcommand{\ea}{\end{array}}
\newcommand{\bad}{\begin{aligned}}
\newcommand{\ead}{\end{aligned}}

\newcommand{\zero}{0}

\newcommand{\logt}{\iota}
\newcommand{\la}{\langle}
\newcommand{\ra}{\rangle}
\newcommand{\modify}[1]{\tilde{#1}}
\newcommand{\defeq}{:=}
\newcommand{\dif}[1]{\hat{#1}}
\newcommand{\cXs}{\mathcal{X}}
\newcommand{\ball}{\mathbb{B}}

\def\gB{{\mathcal{B}}}

\def\gD{{\mathcal{D}}}

\def\gO{{\mathcal{O}}}

\def\gS{{\mathcal{S}}}

\newcommand{\ufun}{\mathscr{F}}
\newcommand{\uspace}{\mathscr{S}}
\newcommand{\utime}{\mathscr{T}}

\allowdisplaybreaks

\begin{document}

\title{Efficiently Escaping Saddle Points in Bilevel Optimization\footnote{The first two authors contributed equally to this work.}}

\author{Minhui Huang\thanks{Department of Electrical and Computer Engineering, University of California, Davis}
\and Xuxing Chen\thanks{Department of Mathematics, University of California, Davis}
\and Kaiyi Ji\thanks{Department of Computer Science and Engineering, University at Buffalo}
\and Shiqian Ma\thanks{Department of Computational Applied Mathematics and Operations Research, Rice University}
\and Lifeng Lai\footnotemark[2]}
\date{\today}
\maketitle

\begin{abstract}
Bilevel optimization is one of the fundamental problems in machine learning and optimization. Recent theoretical developments in bilevel optimization focus on finding the first-order stationary points for nonconvex-strongly-convex cases. In this paper, we analyze algorithms that can escape saddle points in nonconvex-strongly-convex bilevel optimization. Specifically, we show that the perturbed approximate implicit differentiation (AID) with a warm start strategy finds $\epsilon$-approximate local minimum of bilevel optimization in $\tilde{O}(\epsilon^{-2})$ iterations with high probability. Moreover, we propose an inexact NEgative-curvature-Originated-from-Noise Algorithm (iNEON), a pure first-order algorithm that can escape saddle point and find local minimum of stochastic bilevel optimization. As a by-product, we provide the first nonasymptotic analysis of perturbed multi-step gradient descent ascent (GDmax) algorithm that converges to local minimax point for minimax problems. 
\end{abstract}


\section{Introduction}\label{intro}

Bilevel optimization has become a powerful tool in various machine learning fields including reinforcement learning \cite{hong2020two}, hyperparameter optimization \cite{franceschi2018bilevel, feurer2019hyperparameter}, meta learning \cite{franceschi2018bilevel, ji2020convergence} and signal processing \cite{kunapuli2008classification}. A general formulation of bilevel optimization problem can be written as 
\begin{align}
&\min_{x\in\mathbb{R}^{d}}\;\; \Phi(x):=f(x, y^*(x)), \nonumber
\\&\;\;\mbox{s.t.} \quad y^*(x)= \argmin_{y\in\mathbb{R}^{n}} g(x,y).\label{eq:objective}
\end{align}
In this paper, we focus on the nonconvex-strongly-convex case where the lower level function $g(x,y)$ is smooth and strongly convex with respect to $y$ and the overall objective function $\Phi(x)$ is smooth but possibly nonconvex. One crucial but challenging task in the bilevel optimization is the computation of the hypergradient $\nabla \Phi(x)$, which, via chain rule, can be written as
\be\label{phi-grad}
\nabla \Phi(x) =  \nabla_x f(x, y^*(x)) + \frac{\partial y^*(x)}{\partial x} \cdot  \nabla_y f(x, y^*(x)),
\ee
where $ \frac{\partial y^*(x)}{\partial x}  \in \br^{d \times n}.$ Note that the differentiability of  $y^*(x)$ is a direct result of the Implicit Function Theorem, as mentioned in Lemma 2.1 of \cite{ghadimi2018approximation}. By taking derivative with respect to $x$ on the optimality condition: $\nabla_y g(x, y) = 0,$ we have the relation
\be\bad\label{eq:opt-deri1}
\nabla_{xy}^2 g(x, y^*(x)) + \frac{\partial y^*(x)}{\partial x}  \nabla^2_y g(x, y^*(x))= 0,
\ead\ee
 which implies
\be\label{solve-partial-y*-partial-x}
\frac{\partial y^*(x)}{\partial x} = -\nabla^2_{xy} g(x, y^*(x))\cdot \nabla^2_y g(x, y^*(x))^{-1}.
\ee
Substituting \eqref{solve-partial-y*-partial-x} to \eqref{phi-grad}, we get
\begin{align}
\nabla& \Phi(x) =  \nabla_x f(x, y^*(x)) -  \nabla_{xy}^2 g(x, y^*(x))\cdot \nabla^2_y g(x, y^*(x))^{-1} \nabla_y f(x, y^*(x)).\label{eq:bilevel-grad}
\end{align}
Note that the above hypergradient $\nabla \Phi(x)$ involves computationally intractable components such as the exact solution $y^*(x)$ and the Hessian inverse $\nabla^2_y g(x, y^*(x))^{-1}$. To address such difficulties, various computing approaches have been proposed, which include popular Approximate Implicit Differentiation (AID)~\cite{domke2012generic, pedregosa2016hyperparameter, gould2016differentiating,ghadimi2018approximation, grazzi2020iteration, ji2021bilevel} and Iterative Differentiation (ITD)~\cite{domke2012generic, maclaurin2015gradient, shaban2019truncated, grazzi2020iteration, ji2021bilevel}. Among them, \cite{ghadimi2018approximation} and \cite{ji2021bilevel} further analyze the computational complexities of these two types of approaches in finding a stationary point. {Besides these nested-loop approaches, \cite{hong2020two, chen2021single} propose single-loop algorithms with convergence analysis to stationary points.}

However, it still remains unknown how to provably find a {\bf local minimum} for bilevel optimization. This type of study is important as it has been widely shown that saddle points (which are also stationary points) can seriously undermine the quality of solutions~\cite{choromanska2015loss, dauphin2014identifying}. To address this issue, this paper focuses on escaping saddle points for bilevel optimization. We are interested in finding an approximate local minimum for $\Phi(x)$ defined as follows.
%
\begin{definition}[$\epsilon$-local minimum]\label{def:local-min}
We say $x$ is an $\epsilon$-local minimum for bilevel optimization \eqref{eq:objective} if 
\be
\|\nabla \Phi(x)\| \le \epsilon, \quad \lambda_{\min}\left(\nabla^2 \Phi(x) \right) \ge - \sqrt{\rho_\phi \epsilon},
\ee 
where $\lambda_{\min}\left(Z\right)$ denotes the minimum eigenvalue of a matrix $Z$ and $\rho_\phi$ is the Lipschitz constant of $\nabla^2 \Phi (x)$, i.e.,  
\be\label{def-rho-phi-intro}
\left\| \nabla^2 \Phi (x) - \nabla^2 \Phi (x') \right\| \le \rho_\phi \|x - x'\|, \quad \forall x, x' \in \br^d.
\ee
\end{definition}


Motivated by the recent demand in solving online or large-scale bilevel optimization problems, we also generalize our technique to the following stochastic bilevel optimization:
\begin{align}\label{objective}
&\min_{x\in\mathbb{R}^{d}} \Phi(x)=f(x, y^*(x))= \mathbb{E}_{\xi} \left[F(x,y^*(x);\xi)\right] \nonumber \\
& \;\mbox{s.t.} \;y^*(x)= \argmin_{y\in\mathbb{R}^s} g(x,y)= \mathbb{E}_{\zeta} \left[G(x,y;\zeta)\right],
\end{align}
where $f(x,y)$ and $g(x,y)$ take the expectation form with respect to the random variables $\xi$ and $\zeta$. There is a line of work studying stochastic bilevel algorithms that converge to the stationary point \cite{ghadimi2018approximation, ji2021bilevel, hong2020two}. Comparing with these results, we are interested in providing new stochastic algorithms that provably converge to the local minimum.

\subsection{Our Contributions}
 
In this paper, we derive a framework of adding perturbation to gradient sequence for bilevel optimization and design various new bilevel algorithms that provably escape saddle points and find local minimum. Our approach is mostly inspired by existing works for nonconvex minimization and minimax problems \cite{jin2021nonconvex,xu2018first,allen2018neon2}. 
Our main contributions are summarized below.
\begin{itemize}
\item[(i)] For deterministic bilevel optimization, we propose the perturbed AID with warm start strategy. We prove that the proposed algorithm achieves $\epsilon$-local minimum of $\Phi(x)$ in at most $\tilde{\gO}(\epsilon^{-2})$ iterations. Here the notation $\tilde{\gO}(\cdot)$ hides logorithmic terms and absolute constants.
\item[(ii)] For the minimax problem, which is a special case of bilevel optimization, we prove that the strict local minimum of $\Phi(x)$ is equivalent to strict local minimax point \cite{pmlr-v119-jin20e} and propose the perturbed GDmax algorithm with a nonasymptotic convergence rate to local minimax point. To the best of our knowledge,  this is the first nonasympototic analysis for gradient algorithms escaping saddle point in minimax problem.

\item[(iii)] For stochastic bilevel optimization, we propose inexact NEgative-curvature-Originated-from-Noise Algorithm (iNEON), a deterministic algorithm that extracts negative curvature descent direction with high probability. Combining iNEON with stocBiO \cite{ji2021bilevel}, we obtain a stochastic first-order algorithm with a gradient complexity of $\tilde{\gO}(\epsilon^{-4})$. To the best of our knowledge, our algorithms: perturbed AID and stochBiO+iNEON are the first ones that provably converge to local minimum of bilevel optimization.
\end{itemize}
 


\subsection{Related Work}
 \textbf{Escaping Saddle Point.} Most existing works for finding local minimum focus on classical optimization problems (i.e., minimization problems) {and derive the complexity for reaching an $\epsilon$-local minimum. \cite{nesterov2006cubic, curtis2017trust} proposed second-order methods for obtaining an $\epsilon$-local minimum. To avoid Hessian computation required in \cite{nesterov2006cubic,curtis2017trust}, \cite{carmon2018accelerated} and \cite{agarwal2017finding} proposed to use Hessian-vector product and achieved convergence rate of $O(\epsilon^{-7/4})$. Recently, the complexity results of pure first-order methods for obtaining local minimum have been studied (see, e.g., \cite{ge2015escaping,daneshmand2018escaping,jin2021nonconvex,fang2019sharp}). \cite{lee2016gradient} provided asymptotic results showing that gradient descent (GD) method converges to a local minimizer almost surely. \cite{jin2017escape,jin2021nonconvex} proved that the perturbed GD can converge to a local minimizer in a number of iterations that depends poly-logarithmically on the dimension, reaching a nonasymptotic iteration complexity of $\tilde{O}(\epsilon^{-2}\log (d)^4)$ for nonconvex minimization. For stochastic optimization problems, \cite{jin2021nonconvex, jin2018accelerated, fang2019sharp} provided nonasymptotic rate for finding local minimizers. How to escape saddle points for constrained problems and nonsmooth problems are also studied in the literature. In particular, \cite{lu2020finding, criscitiello2019efficiently, sun2019escaping} studied escaping saddle points for constrained optimization. \cite{davis2021proximal, davis2021escaping, huang2021escaping} studied escaping saddle points for nonsmooth problems. All these algorithms are for solving the minimization problems, and to the best of our knowledge, how to escape saddle points in bilevel optimization has not been addressed in the literature. 
 
\textbf{Minimax Optimization.} Motivated by its applications in adversarial learning \cite{goodfellow2014explaining, sinha2018certifying}, training GANs \cite{goodfellow2014generative, arjovsky2017wasserstein} and optimal transport \cite{lin2020projection, huang2021convergence, huang2021projection, huang2021riemannian}, the convergence theory of nonconvex minimax problems has been extensively studied in the literature. Specifically, \cite{nouiehed2019solving, pmlr-v119-jin20e} studied the complexity of multistep gradient descent ascent (GDmax). \cite{lin2020gradient, lu2020hybrid} provided the first convergence analysis for the single loop gradient descent ascent (GDA) algorithm. More recently, \cite{luo2020stochastic} applied the stochastic variance reduction technique to the nonconvex-strongly-concave case and achieved the best known stochastic gradient complexity. \cite{zhang2020single} proposed smoothed GDA, which stabilizes GDA algorithm and helps achieve a better complexity for the nonconvex-concave case. However, all the previous works targeted finding stationary point of $\Phi(x)$. Very recently, \cite{ chen2021escaping, luo2021finding} proposed cubic regularized GDA, a second-order algorithm that provably converges to a local minimum. \cite{fiez2021global} provided asymptotic results showing that GDA converges to local minimax point almost surely. To the best of our knowledge, the convergence rate of first-order methods for obtaining a local minimax point has been missing in the literature.

\textbf{Bilevel Optimization.} The bilevel optimization has a long history and dates back to \cite{bracken1973mathematical}. Recently, bilevel programming has been  successfully applied to meta-learning \cite{snell2017prototypical, rajeswaran2019meta, franceschi2018bilevel,ji2022theoretical} and hyperparameter optimization \cite{pedregosa2016hyperparameter, franceschi2018bilevel, shaban2019truncated,sow2021based}. Theoretically, \cite{ghadimi2018approximation} provided the first convergence rate for the AID approach. \cite{ji2021bilevel} further improved their complexity dependence on the condition number and analyzed the convergence of the ITD approach. Both AID and ITD have an iteration complexity of $\gO(\epsilon^{-2})$. \cite{ji2021lower} provided lower bounds for a class of AID and ITD-based bilevel algorithms. For stochastic bilevel problems, \cite{ghadimi2018approximation, ji2021bilevel} proposed BSA and stocBiO methods respectively, which are both double-loop algorithms inspired by AID. \cite{hong2020two} proposed TTSA, a provable single-loop algorithm that updates two variables in an alternating way with a convergence rate of $\gO(\epsilon^{-5})$.  \cite{chen2021closing} proposed ALSET, a simple SGD type approach, and improved the convergence rate to $\gO(\epsilon^{-4})$. Very recently, \cite{khanduri2021a, yang2021provably, chen2021single, guo2021randomized} studied stochastic algorithms with variance reduction and momentum techniques, and provided the cutting-edge first-order oracle complexity, which is $\gO(\epsilon^{-3})$. It is worth noting that extending the single-agent bilevel optimization to distributed settings has also been studied \cite{tarzanagh2022fednest, chen2022decentralized, yang2022decentralized, chen2022decentralizeddsbo, huang2023achieving}. All these previous analyses have focused on finding stationary points and algorithm for finding a local minimum is still missing.

\textbf{Notation.} 
 Let $\hat{\Phi}(x), \widehat{\nabla}{\Phi}(x), \widehat{\nabla}^2{\Phi}(x)$ be the inexact function value, gradient and Hessian respectively. Denote $G(f, \epsilon), JV(f, \epsilon), HV(f, \epsilon)$ as the complexity of gradient evaluations, Jacobian-vector product evaluations, and Hessian-vector product evaluations of function $f$, respectively. In particular, for matrix-vector product oracles, say Hessian-vector products, $HV(f, \epsilon)$ represents the total number of (deterministic or stochastic) $(\nabla^2f)\cdot v$ computation in our algorithm.
 Typically computing a Hessian-vector product is as cheap as computing a gradient \cite{pearlmutter1994fast}. Let $\kappa$ be the condition number of the lower-level problem. We use notation $\gO(\cdot)$ to hide only absolute constants which do not depend on any problem parameters and $\tilde{\gO}(\cdot)$ to further hide
additional $\log$ factors.

\section{Escape Saddle Points in General Bilevel Optimization} \label{sec:bilevel}
In this section, we propose novel algorithms for general bilevel optimization \eqref{eq:objective} that are guaranteed to converge to local minimum. We consider one of the popular approaches AID to estimate the hypergradient $\nabla \Phi(x)$. The AID approach is a nested-loop algorithm, which first update the lower-level variable $y$ with $D$ steps of gradient descent, and then construct an estimate of the upper-level hypergradient. To efficiently approximate the Hessian inverse in the hypergradient \eqref{eq:bilevel-grad}, AID solves the linear system: 
\be\bad\label{eq:aid-step1}
\nabla_y^2 g(x_k,y_k^D) v = \nabla_y f(x_k,y^D_k)
\ead\ee
using $N$ steps of the conjugate gradient (CG) method. The resulting vector $v_k^N$ is used as an approximation to the solution of \eqref{eq:aid-step1}: $\nabla_y^2 g(x_k,y_k^D)^{-1} \nabla_y f(x_k,y^D_k)$. The hypergradient is then constructed as 
\be\bad\label{eq:aid-step2}
\widehat\nabla \Phi(x_k)= \nabla_x f(x_k,y_k^D) -\nabla_{xy}^2 g(x_k,y_k^D)v_k^N.
\ead\ee
However, current AID-based approach can only guarantee the convergence to the first-order stationary point. In Algorithm \ref{alg:perturbedalg}, we propose perturbed AID (i.e., Algorithm \ref{alg:perturbedalg} with option \textbf{AID} in step 9) 
for solving bilevel optimization \eqref{eq:objective} with convergence guarantee to second-order stationarity. In the proposed algorithms, we update variable $x$ with the hypergradient $\widehat{\nabla}\Phi(x)$ estimated by AID. When the norm of $\widehat{\nabla}\Phi(x)$ is small, we add random noise sampled from a uniform ball and keep running AID for at least $\utime$ steps (see steps 10-13 of Algorithm \ref{alg:perturbedalg}). If the current point is a saddle point of $\nabla \Phi(x)$, we show that with high probability the function value $\Phi(x)$ has  sufficient decrease after $\utime$ steps so it can escape the current saddle point. 

\begin{algorithm}[t] 
\small 
\caption{Perturbed Algorithms for Minimax and Bilevel Optimization Problems}
\label{alg:perturbedalg}
\begin{algorithmic}[1]
\STATE \textbf{Input:} Iteration Numbers $K, D, N$, Step Sizes $\tau, \eta$, Accuracy $\epsilon$, Radius $r$, Perturbation Time $\utime$.
\STATE \textbf{Initialization:}  $x_0, y_0, v_0$.
\STATE Set $k_{perturb} = 0$
\FOR{$k = 0, 1, 2, \ldots, {K-1}$}
\STATE{Set $y_k^0 = y_{k-1}^D$ if $k > 0$, otherwise $y_0$}
\FOR{$t = 0, 1, 2, \ldots, D$}
\STATE $y_k^t = y_k^{t-1} - \tau \cdot \nabla_y g(x_k, y_k^{t-1})$
\ENDFOR
 \STATE{Estimating Hypergradient:
 \vspace{0.1cm} \\\hspace{0cm} {\bf Option 1 (for minimax problem) GDmax}: compute $\widehat\nabla \Phi(x_k) = \nabla_x f(x_k, y^D_k)$
\vspace{0.1cm}
              \\\hspace{0cm} {\bf Option 2 (for bilevel optimization) AID}: \\\hspace{0.7cm}
              1) set $v_k^0 = v_{k-1}^{N} \mbox{ if }\; k> 0$ and $v_0$ otherwise
              \vspace{0.1cm}    \\\hspace{0.8cm}2) solve $v_k^N$ from $\nabla_y^2 g(x_k,y_k^D) v = \nabla_y f(x_k,y^D_k)$ via $N$ steps of CG starting from $v_k^0$
 \vspace{0.1cm} \\\hspace{0.8cm}3) get Jacobian-vector product {\small$\nabla^2_{xy} g(x_k,y_k^D)v_k^N$} via automatic differentiation
 \vspace{0.1cm} \\\hspace{0.8cm}4) {\small$\widehat\nabla \Phi(x_k)= \nabla_x f(x_k,y_k^D) -\nabla^2_{xy} g(x_k,y_k^D)v_k^N$} 
                    }
\IF {$\|\widehat{\nabla} \Phi(x_k) \| \le \frac{4}{5}\epsilon$ and $k - k_{perturb} > \utime$}
\STATE $x_k = x_k - \eta \cdot u, \quad (u \sim \text{Uniform}(\mathbb{B}(r)))$
\STATE $k_{perturb} = k$
\ENDIF
\STATE $x_{k+1} = x_k - \eta \cdot \widehat{\nabla} \Phi(x_k) $
\ENDFOR
\STATE \textbf{Output:} $x_K$.
\end{algorithmic}
\end{algorithm}

\subsection{Convergence Analysis}

We first state assumptions needed for our analysis. 
\begin{assumption}\label{assu:bilevel}
Assume the upper level function $f(x, y)$ and the lower level function $g(x, y)$ satisfy the following assumptions:
\begin{itemize}
\item[(i)] Function $g(x, y)$ is three times differentiable and $\mu$-strongly convex with respect to $y$ for any fixed $x$.
\item[(ii)] Function $f(x,y)$ is twice differentiable and $f(x,y)$ is $M$-Lipschitz continuous with respect to $x$ and $y$. 
\item[(iii)] Gradients $\nabla f(x,y)$ and $\nabla g(x,y)$ are $\ell$-Lipschitz continuous with respect to $x$ and $y$.
\item[(iv)] Jacobian matrices $\nabla_x^2 f(x, y)$, $ \nabla^2_{xy} f(x, y)$,  $ \nabla_y^2 f(x, y)$, $\nabla^2_{xy} g(x, y)$ and $\nabla_y^2 g(x, y)$
are $\rho$-Lipschitz continuous with respect to $x$ and  $y$.
\item[(v)] Third-order  derivatives $\nabla^3_{xyx} g(x, y)$, $\nabla^3_{yxy} g(x, y)$ and $\nabla_y^3 g(x, y)$ are $\nu$-Lipschitz continuous with respect to $x$ and  $y$.
\end{itemize}
\end{assumption}

\begin{remark}
Compared with assumptions in recent bilevel optimization literature \cite{ghadimi2018approximation, ji2021bilevel}, we further assume the Lipschitz continuity of the Hessian of $f(x, y)$ and the third-order derivative of $g(x, y)$. These assumptions are required to prove the Hessian Lipschitz continuity of $\Phi(x)$, which is a common condition required in the literature of escaping saddle points \cite{jin2021nonconvex}.
\end{remark}

\begin{remark}
It should be noted that although we assume the third-order partial derivatives of $g(x, y)$ to be Lipschitz continuous, this assumption is for the theoretical analysis only, we do not compute any third-order derivatives in our algorithms.
\end{remark}

One of the key elements in our proof technique is to show that under Assumption \ref{assu:bilevel}, function $\Phi(x)$ is Hessian Lipschitz continuous, as shown in the following lemma. 

\begin{lemma}\label{hess-Lipschitz}
Suppose Assumption \ref{assu:bilevel} holds, then $\Phi (x)$ is $\rho_\phi$-Hessian Lipschitz continuous, i.e., \eqref{def-rho-phi-intro} holds, 
where $\rho_\phi = \gO(\kappa^5)$ and is defined in \eqref{def-rho-phi}. 
\end{lemma}

For the AID approach, the main results are in the following theorem.

\begin{theorem}[\textbf{Convergence of Perturbed AID}]\label{thm:aid}
Suppose that Assumption \ref{assu:bilevel} holds. Set parameters of Algorithm \ref{alg:perturbedalg} as in \eqref{eq:para_minimax} and  
\be\label{eq:aid-ND-thm}
D=\gO\left(\kappa \log\left(\epsilon^{-1}\right) \right), \quad N=\gO\left(\sqrt{\kappa} \log\left(\epsilon^{-1}\right) \right).
\ee
With probability at least $1 - \delta$, the  iteration number of the perturbed AID algorithm for visiting an $\epsilon$-local minimum of $\Phi(x)$  is 
\be\label{iter-number-AID}
K=\tilde{\gO} \left( \kappa^3 \epsilon^{-2}\right).
\ee
\end{theorem}

\begin{corollary}
The gradient complexities of the perturbed AID algorithm for finding an $\epsilon$-local minimum of $\Phi(x)$ are 
\[
G(f, \epsilon) = \tilde{\gO}(\kappa^3 \epsilon^{-2}), \quad G(g, \epsilon) = \tilde{\gO}(\kappa^4 \epsilon^{-2}).
\]
The  Jacobian- and Hessian-vector product complexities are
\[
JV(g, \epsilon) = \tilde{\gO}(\kappa^3 \epsilon^{-2}), \quad HV(g, \epsilon) = \tilde{\gO}(\kappa^{3.5} \epsilon^{-2}).
\]
\end{corollary}

\begin{remark}
Though the complexities of the perturbed AID method are worse than the results in \cite{ji2021bilevel} by a $\log$ factor, it should be noted that the algorithms converge to different points. Specifically, our perturbed AID method converges to a local minimum of $\Phi(x)$, whereas the algorithms in \cite{ji2021bilevel} are only  guaranteed to converge to first-order stationarity.
\end{remark}

\subsection{Proof sketch} 
We briefly describe the main elements in proving the above theorems. The main ideas follow \cite{jin2021nonconvex}. 
However, in contrast to the problem studied in \cite{jin2021nonconvex}, we do not have access to the exact hypergradient of $\Phi(x)$ in bilevel optimization problems. Therefore, we need to deal with the error introduced by this approximation. We first provide the inexact descent lemma.

\begin{lemma}[Inexact Descent Lemma]\label{lem:descent}
Suppose Assumption \ref{assu:bilevel} holds and set $\eta = 1/ L_\phi$, then the inexact gradient sequence $\{x_k\}$ satisfies:
\be\bad
\Phi(x_{k+1}) - \Phi(x_{k}) \le& -\frac{\eta}{4} \left\|\widehat{\nabla} \Phi(x_k) \right\|^2 +  \eta \|\nabla \Phi(x_k) - \widehat{\nabla} \Phi(x_k) \| ^2.
\ead\ee
\end{lemma}

Secondly, the following lemma shows that with high probability, adding random noise sampled uniformly from a ball helps escape saddle points of $\Phi(x)$.

\begin{lemma}[Escaping Saddle Points]\label{lem:escapesaddle}
Assume Assumption \ref{assu:bilevel} holds. Assume $\tilde{x}$ satisfies $\norm{\nabla \Phi(\tilde{x})} \le \epsilon$, and $\lambda_{\min}(\nabla^2 \Phi(\tilde{x})) \le -\sqrt{\rho_{\phi}\epsilon}$, where $\rho_\phi = \gO(\kappa^5)$ and is defined in \eqref{def-rho-phi}. Let $x_0 = \tilde{x} + \eta u ~(u \sim \text{Uniform}(B_0(r)))$. With parameters given in \eqref{eq:para_minimax}, as long as the following inequality holds in each iteration,
\be\bad\label{eq:error-required}
\|\nabla \Phi(x_k) - \widehat{\nabla} \Phi(x_k) \| \le \min \left\{ \frac{\sqrt{17} }{80\logt^2}, \frac{1}{16\logt^2 2^{\logt}} \right\}\cdot \epsilon,
\ead\ee
with probability at least $1 - \delta$, it holds that 
\be\bad\label{eq:phi-dec-highprob-main-body}
\Phi(x_\utime) - \Phi(\tilde{x}) \le -\ufun/2,
\ead\ee
where $x_\utime$ is the $\utime^{\textrm{th}}$ gradient descent iterate starting from $x_0$, $\logt$ satisfies \eqref{iota-condition} and 
\be\label{ufun-def}
\ufun = \frac{1}{100\logt^3}\sqrt{\frac{\epsilon^3}{\rho_\phi}}. 
\ee
\end{lemma}


Finally, by Lemma \ref{lem:descent}, Lemma \ref{lem:escapesaddle} and with a proper choice of parameters $D$ and $N$, we can bound the total iteration number of Algorithm \ref{alg:perturbedalg} by 
\[
K = \frac{ (\Phi(x_{0}) - \Phi* )\utime}{ \ufun}+ \frac{ L_\phi (\Phi(x_{0}) - \Phi* )}{\epsilon^2}.
\]

\section{Escape Saddle Points for Minimax Problem}\label{subsec:minimax}


In this section, we consider the following nonconvex-strongly-concave minimax problem:
\be\bad\label{eq:minimax-formu}
&\min_{x\in \br^d} \max_{y \in \br^n}& f(x, y),
\ead\ee
where $f(x, y)$ is nonconvex with respect to $x$ and $\mu$-strongly concave with respect to $y$. By defining the function $\Phi(x) = \max_y f(x, y)$, \eqref{eq:minimax-formu} reduces to a smooth nonconvex minimization problem $\min_{x\in\br^d} \Phi(x)$. Note that this is also a special case of the bilevel optimization problem by setting $g(x, y) = - f (x, y)$ in \eqref{eq:objective}, which leads to the following problem:
\[
    \min_x f(x, y^*(x)),\ \text{s.t. } y^*(x) = \argmin g(x,y) := -f(x,y).
\]
The minimax problem \eqref{eq:minimax-formu} seeks the Nash equilibrium of $f(x, y)$. However, when considering nonconvex minimax problem, the global Nash equilibrium does not exist in general. Instead, one is more interested in finding the local Nash equilibrium \cite{daskalakis2018limit, mazumdar2019finding}  and the local minimax point \cite{pmlr-v119-jin20e}. Therefore, the following question arises naturally: 
\begin{itemize}
\item \textit{What is the relationship between the local minimum of $\Phi(x)$ and the local optimality of the minimax problem \eqref{eq:minimax-formu}?}
\end{itemize}
The answer to this question is still ambiguous thus far. 
We first discuss the relationship between the local minimum of $\Phi(x)$ and the local Nash equilibrium of \eqref{eq:minimax-formu}. The Nash equilibrium and its local alternative are defined as below.
\begin{definition}\cite{mazumdar2019finding}[Local Nash Equilibrium]\label{def:pureNash}
We say $(x^\star, y^\star)$ is a \textbf{Nash equilibrium} of function $f$, if for any $(x, y)$:
\begin{equation*}
f(x^\star, y) \le f(x^\star, y^\star) \le f(x, y^\star).
\end{equation*}
Point $(x^\star, y^\star)$ is a \textbf{local Nash equilibrium} of $f$ if there exists $\delta>0$ such that for any $(x, y)$ satisfying $\norm{x - x^\star} \le \delta$ and $\norm{y - y^\star} \le \delta$ we have:
\begin{equation*}
f(x^\star, y) \le f(x^\star, y^\star) \le f(x, y^\star).
\end{equation*}
\end{definition}
The local Nash equilibrium can be characterized in terms of first-order and second-order conditions. Specifically, when assuming $f$ is twice-differentiable, any stationary point (i.e., $\nabla f = \zero$) is a \textbf{strict local Nash equilibrium} if and only if 
\[
\nabla^2_{yy} f(x, y) \prec \zero, \text{~and~} \nabla^2_{xx} f(x, y) \succ \zero.
\]
We have the following proposition, showing that the local minimum of $\Phi(x)$ is indeed superior to its saddle point regarding whether it is a local Nash equilibrium or not. 
\begin{proposition}\label{pro:minimax}
For any smooth nonconvex-strongly-concave function $f(x, y)$, define $\Phi(x) = \max_{y \in \br^n} f(x, y)$. Then we have
\begin{itemize}
\item[(i)] A saddle point of $\Phi(x)$ cannot be a strict local Nash equilibrium of $f(x, y)$.
\item[(ii)] A strict local Nash equilibrium of $f(x, y)$ must be a local minimum of $\Phi(x)$.
\end{itemize}
\end{proposition}

Moreover, \cite{pmlr-v119-jin20e} introduced the concept of local minimax point, which is a weakened notion of the local Nash equilibrium. Compared with the local Nash equilibrium, the local minimax point alleviates the non-existence issue\footnote{In the Proposition 6 of \cite{pmlr-v119-jin20e}, the authors constructed a two dimensional function showing that the Nash equilibria may not exist, and this is known as the ``non-existence issue''.} and is the first proper mathematical definition of local optimality for the two-player sequential games. 

\begin{definition}\cite{pmlr-v119-jin20e}[Strict Local Minimax Point]
For any twice differentiable function $f(x, y)$, a point $(x, y)$ is a strict local minimax point if it satisfies $\nabla f(x, y) = \zero$, $\nabla^2_{yy} f(x, y) \prec \zero$ and
\[
\nabla^2_{xx} f(x, y) -  \nabla^2_{xy} f(x, y)\nabla^2_{yy} f(x, y)^{-1}\nabla^2_{xy} f(x, y)\succ \zero.
\]
\end{definition}
The following proposition shows the equivalence between a strict local minimax point and a strict local minimum of $\Phi(x)$. 
\begin{proposition}\label{lem:equivalence}
For nonconvex-strongly-concave minimax problem \eqref{eq:minimax-formu}, suppose $\Phi(x)$ has a strict local minimum, then a strict local minimax point of \eqref{eq:minimax-formu} always exists and is equivalent to a strict local minimum of $\Phi(x)$. 
\end{proposition}
Most existing convergence theory for minimax problems focuses on finding $\epsilon$-stationary point. Recently, \cite{chen2021escaping, luo2021finding} proposed second-order algorithms for minimizing $\Phi(x)$, which is guaranteed to converge to local minimum. Second-order methods enjoy faster convergence rate than gradient methods, but require solving nonconvex subproblems in each iteration. Moreover, second-order methods are difficult to be implemented in large-scale problems due to the heavy computation of Hessian matrices.  \cite{fiez2021global} proved that GDA asymptotically converges to strict local minimax point almost surely. However, no convergence rate for finding a local minimax point was given in \cite{fiez2021global}. 

The above facts motivate us to propose the perturbed GDmax Algorithm (Algorithm \ref{alg:perturbedalg} with option \textbf{GDmax} in step 9), a first-order nested-loop algorithm that provably escapes saddle points in minimax problems. In the inner loop, the perturbed GDmax runs $D$ steps of gradient ascent for solving the $y$-subproblem inexactly. With the warm start strategy, we set the initial point in $k$-th iteration $y_k^0$ to be the output of the inner loop in the previous iteration $y_{k-1}^D$. In the outer loop, we estimate the hypergradient $\nabla \Phi(x)$ by
\[
\widehat{\nabla} \Phi(x) = \nabla_x f(x_k, y^D_k),
\]
and update $x$ by one step of inexact gradient descent. When the first-order stationary condition is satisfied (step 10 in Algorithm \ref{alg:perturbedalg}), we add a random noise vector sampled uniformly from a ball with radius of $r$ and centered at the current iterate.

Now we analyze the convergence of the perturbed GDmax algorithm. We first list our assumptions.  
\begin{assumption}\label{assu:minimax}
For the minimax problem \eqref{eq:minimax-formu}, $f(x, y)$ satisfies the following assumptions:
\begin{itemize}
\item[(i)] $f(x, y)$ is twice differentiable, $\mu$-strongly concave with respect to $y$ and non-convex with respect to $x$.
\item[(ii)] Denote $z = (x, y)$. $f(z)$ is $\ell$-smooth, i.e., for any $z, z'$, it holds:
\[
\|\nabla f(z) - \nabla f(z') \|\le \ell \|z - z'\|.
\]
\item[(iii)] The Hessian and Jacobian matrices $\nabla_x^2 f(x, y)$, $\nabla_{xy}^2 f(x, y)$, and $ \nabla_y^2 f(x, y)$ are $\rho$-Lipschitz continuous.
\item[(iv)] Function $\Phi(x)$ is bounded below and has compact sub-level sets.
\end{itemize}
\end{assumption}
\begin{remark}
Compared with assumptions for general bilevel optimization problem (Assumption \ref{assu:bilevel} in Section \ref{sec:bilevel}), we do not require any third-order derivative information for the minimax problem. 
\end{remark}

Our perturbed GDmax algorithm is the first pure gradient algorithm with a nonasymptotic convergence rate for finding a local minimax point. The main results for the perturbed GDmax algorithm are given in the following theorem.
\begin{theorem}[\textbf{Convergence of Perturbed GDmax}] \label{thm:minimax}
Suppose $f(x, y)$ satisfies Assumption \ref{assu:minimax}.
Set parameters as 
\be\label{eq:para_minimax}
\tau = \frac{1}{\ell}, \quad \eta = \frac{1}{L_\phi}, \quad r = \frac{\epsilon}{400 \iota^3}, \quad \utime = \frac{L_\phi}{ \sqrt{\rho_{\phi}\epsilon}} \cdot \logt
\ee
and $D=\gO\left(\kappa \log\left(\epsilon^{-1}\right)\right)$, 
with probability at least $1 - \delta$, the perturbed GDmax Algorithm (i.e., Algorithm  \ref{alg:perturbedalg} with option \textbf{GDmax} in step 9) obtains an $\epsilon$-local minimum of $\Phi(x)=\max_y f(x,y)$ in
\[
K=\tilde{\mathcal{O}} \left( \frac{L_\phi (\Phi(x_0) - \Phi(x^*))}{\epsilon^2} \right)=\tilde{\mathcal{O}}(\kappa\epsilon^{-2})
\]
iterations. Here $\delta\in(0,1)$, $L_\phi$ is the Lipschitz constant of $\nabla\Phi$, $\rho_\phi$ is the Lipschitz constant of $\nabla^2\Phi$ (see Lemma \ref{lemC.1}), and $\logt$ is a constant satisfying 
\be\label{iota-condition}
\logt > 1, \mbox{ and } \delta \ge \frac{L_\phi \sqrt{d}}{\sqrt{\rho_{\phi}\epsilon}}\cdot \logt^2 2^{8-\logt}.
\ee
\end{theorem}

\begin{remark}
Throughout this paper, we also assume
\be\label{L-phi-big}
L_\phi/\sqrt{\rho_{\phi}\epsilon} \ge 1.
\ee
We make this assumption because if \eqref{L-phi-big} does not hold, finding the $\epsilon$-local minimum is straightforward, see \cite{jin2021nonconvex}. 
\end{remark}
\begin{remark}
Note that in practice, we may choose $\logt$ sufficiently large so that $\delta$ in \eqref{iota-condition} can be small, which leads to the fact that Theorem 17 holds with probability at least $1-\delta$.
\end{remark}
\begin{corollary}\label{cor:minimax}
The complexity of the gradient evaluations of the perturbed GDmax algorithm for  finding an $\epsilon$-local minimum of $\Phi(x)=\max_y f(x,y)$ is 
\[
G(f, \epsilon) = D\cdot K = \tilde{\mathcal{O}} \left( \kappa^2 \epsilon^{-2} \right).
\]
\end{corollary}

\begin{remark}
Compared with results of second-order methods escaping saddle points for minimax problem \cite{chen2021escaping, luo2021finding}, our perturbed GDmax algorithm is purely first order, which means we do not need to compute Hessian-vector product. Moreover, algorithms in \cite{chen2021escaping, luo2021finding} require solving a nonconvex cubic sub-problem and multiple linear systems in each iteration. All these expensive computations are avoided in our perturbed GDmax algorithm, which makes it practical in real applications.
\end{remark}

\begin{remark}
Compared with asymptotic results in  \cite{fiez2021global}, we provide nonasymptotic convergence rate for finding a local minimax point for minimax problems.
\end{remark}

\begin{remark}
The dependence on the conditional number $\kappa$ for the perturbed GDmax algorithm is $\tilde \gO(\kappa^2)$, which matches the order in \cite{lin2020gradient} and is better than the results in \cite{chen2021escaping}.
\end{remark}

\section{Inexact NEON and Stochastic Bilevel Algorithms}\label{sec:sto-bilevel}
In this section, we consider escaping saddle points for stochastic bilevel optimization problem \eqref{objective}. Inspired by recent work \cite{xu2018first} and \cite{allen2018neon2}, we propose inexact NEON (iNEON) that helps escape saddle points in stochastic bilevel optimization \eqref{objective}.

\subsection{Inexact NEON}
Recently, \cite{xu2018first} and \cite{allen2018neon2} proposed NEgative curvature Originated from Noise (NEON) and NEON2, two pure first-order methods that extract negative curvature descent direction. NEON turns almost all stationary-point finding algorithms into local-minimum finding algorithms. The work of \cite{xu2018first} was inspired by the connection between perturbed gradient descent method \cite{jin2017escape} and the power method, while the idea of \cite{allen2018neon2} is based on the result of Oja's algorithm \cite{allen2017follow}. Compared with classical optimization problems, bilevel optimization no longer has access to the exact gradient, which motivates us to propose the inexact NEON (iNEON). 
The proposed iNEON update is
\be\bad\label{eq:u-update-neon}
u_{k+1} = u_k - \eta (\widehat{\nabla} \Phi(\tilde{x} + u_k) - \widehat{\nabla} \Phi(\tilde{x}) ),
\ead\ee
where $\widehat{\nabla} \Phi(\tilde{x} + u_k)$ and $ \widehat{\nabla} \Phi(\tilde{x})$ are the hypergradient estimates. Our iNEON algorithm is described in Algorithm \ref{alg:iNEON}. Intuitively, the iNEON can be viewed as an approximate power method. More specifically, note that \eqref{eq:u-update-neon} is equivalent to
\be\bad\label{eq:neon-update}
u_{k+1} =& u_k - \eta (\nabla \Phi(\tilde{x} + u_k) - \nabla \Phi(\tilde{x}) ) + \delta_{ k}\\
\approx &(I - \eta \nabla^2 \Phi(\tilde{x})) u_k + \delta_{k},
\ead\ee
where $\delta_{k} = \eta (\widehat{\nabla} \Phi(\tilde{x} + u_k) - \widehat{\nabla} \Phi(\tilde{x})  - \nabla \Phi(\tilde{x} + u_k) + \nabla \Phi(\tilde{x}) )$ is the gradient estimation error and in the last step we use the approximation: $\nabla \Phi(\tilde{x} + u_k)  - \nabla \Phi(\tilde{x}) \approx  \nabla^2 \Phi(\tilde{x}) u_k  $ as long as $\|u_k\|$ is small. Therefore, \eqref{eq:u-update-neon} is equivalent to applying approximate power method to the matrix $I - \eta \nabla^2 \Phi(\tilde{x})$ starting with initial vector $u_0$.

\begin{algorithm}[t] 
\caption{Inexact NEgative-curvature-Originated-from-Noise Algorithm (iNEON)}
\label{alg:iNEON}
\begin{algorithmic}[1]
\STATE \textbf{Input:} Iteration Numbers $\utime, D$, Step Sizes $\tau, \eta$, Accuracy $\epsilon$, Radius $r$, Potential Saddle Point $\tilde{x}$, Initial Point $y_0$
\STATE Select $u_0 \sim Uniform(\mathbb{B}(\eta r))$
\STATE Set $\tilde{y}^0 = y_0$
\FOR{$t = 0, 1, 2, \ldots, D$}
\STATE $\tilde{y}^t = \tilde{y}^{t-1} - \tau \cdot \nabla_y g(\tilde{x}, \tilde{y}^{t-1})$
\ENDFOR
\STATE Compute $\widehat{\nabla}  \Phi(\tilde{x})$ using AID and $\tilde{y}^D$
\FOR{$k = 0, 1, 2, \ldots, \utime$}
\STATE{Set $y_k^0 = y_{k-1}^D$ if $k > 0$, otherwise $y_0$}
\FOR{$t = 0, 1, 2, \ldots, D$}
\STATE $y_k^t = y_k^{t-1} - \tau \cdot \nabla_y g(\tilde{x} + u_k, y_k^{t-1})$
\ENDFOR
\STATE Compute $\widehat{\nabla} \Phi(\tilde{x} + u_k)$ using AID and $y_k^D$
\STATE $u_{k+1} = u_k - \eta (\widehat{\nabla} \Phi(\tilde{x} + u_k) - \widehat{\nabla} \Phi(\tilde{x}) ) $
\IF{$\hat{\Phi} (\tilde{x} + u_{k+1}) - \hat{\Phi} (\tilde{x}) - \widehat{\nabla} \Phi(\tilde{x}) ^\top u_{k+1}  \le - \frac{11519}{12800}\ufun$}
\STATE return $u_{out} = u_{k+1} / \|u_{k+1}\|$
\ENDIF
\ENDFOR
\STATE return 0
\end{algorithmic}
\end{algorithm}

%

We next show that iNEON can extract negative gradient descent direction with high probability.
\begin{lemma}\label{lem:neon}
Suppose Assumption \ref{assu:bilevel} holds. Choose parameters 
\be\label{eq:para_neon}
\tau = \frac{1}{\ell}, \quad \eta = \frac{1}{L_\phi}, \quad r = \frac{\epsilon}{400 \iota^3}, \quad \utime = \frac{L_\phi}{ \sqrt{\rho_{\phi}\epsilon}} \cdot \frac{\logt}{4},
\ee
\be\label{eq:para_neon-2}
\ufun = \frac{1}{25 \logt^3} \sqrt{\frac{\epsilon^3}{\rho_{\phi}}}, 
\ee
and
$D = \gO(\kappa \log(\epsilon^{-1}))$ in Algorithm \ref{alg:iNEON}. Let $\tilde{x}$ satisfy $\norm{\nabla \Phi(\tilde{x})} \le \epsilon$, and $\lambda_{\min}(\nabla^2 \Phi(\tilde{x})) \le -\sqrt{\rho_{\phi}\epsilon}$, where $\rho_\phi = \gO(\kappa^5)$ and is defined in \eqref{def-rho-phi}. Denote $u_{out}$ as the output of Algorithm \ref{alg:iNEON}. If $u_{out} \ne 0$, we have with probability at least $1 - \delta$ that 
\be\bad\label{eq:phi-dec-highprob}
\frac{u_{out}^\top \nabla^2 \Phi(\tilde{x}) u_{out}}{\|u_{out}\|^2} \le -\frac{1}{40\logt} \sqrt{\rho_\phi\epsilon},
\ead\ee 
where $\logt$ is a constant satisfying
\be\label{iota-condition-2}
\logt > 1, \mbox{ and } \delta > \frac{\ell  \sqrt{d}}{\sqrt{\rho\epsilon}} \cdot \logt^2 2^{8-\logt/4}.
\ee
If $u_{out} = 0$, then we conclude that $\lambda_{\min} (\nabla^2 \Phi(x)) \ge - \sqrt{\rho_\phi\epsilon}$ with high probability $1 - O(\delta).$
\end{lemma}
\begin{remark}
Compared with results in \cite{xu2018first}, we provide a simplified proof that can handle the gradient estimation error.
\end{remark}
So far, we treat iNEON as a deterministic algorithm that extracts the descent direction for a deterministic objective function $\Phi(x)$. We will show how to apply iNEON to stochastic bilevel algorithms in the next section.

\subsection{StocBiO Escapes Saddle Point}

\begin{algorithm}[t] 
\caption{StocBiO with iNEON}
\label{alg:INEONstocbio}
\begin{algorithmic}[1]
\STATE {\bfseries Input:} $K,D,Q$, batch size $S$, stepsizes $\alpha$ and $\beta$, initializations $x_0$ and $y_0$.
\FOR{$k=1,\ldots,K-1$}
\STATE{Set $y_k^0 = y_{k-1}^D$ if $k > 0$, otherwise $y_0$}
\FOR{$t=1,....,D$}
\STATE{Draw a sample batch $\gS_{t-1}$ with batch size $S$}  
\vspace{0.05cm}
\STATE{Update $y_k^t = y_k^{t-1}-\alpha \nabla_y G(x_k,y_k^{t-1}; \gS_{t-1}) $}
\ENDFOR
\STATE{Draw sample batches $\gD_F,\gD_H$ and $\gD_G$ }
\STATE{Compute $  \widehat \nabla \Phi(x_k) $  by \eqref{est-vq} - \eqref{estG}     }
\STATE{Update $x_{k+1}=x_k- \beta \widehat \nabla \Phi(x_k) $}
\STATE {$k \gets k+1$;}
\STATE {Compute $\widehat \nabla \Phi_\gD(x_k) $ via AID}
\IF{$ \|\widehat\nabla \Phi_{\gD}(x_{k})\| \le \frac{4}{5}\epsilon$}
\STATE $u = iNEON (x_{k}, \utime, r, f_{\gD_F}, g_{\gD_G})$
\IF{$u = 0$}
\STATE Return $x_{k}$;
\ELSE
\STATE Select Rademacher variable $\bar{\xi} \in \{1, -1\}$
\STATE $x_{k+1} =  x_{k}  - \frac{\bar{\xi}}{80} \sqrt{\frac{ \epsilon}{\rho_\phi}}u$
\STATE {$k \gets k+1$;}
\ENDIF
\ENDIF
\ENDFOR
\end{algorithmic}
\end{algorithm}

In this section, we apply iNEON to a popular algorithm for stochastic bilevel optimization, StocBiO \cite{ji2021bilevel}. StocBiO is a double-loop batch stochastic algorithm, which has similar structure as AID. In its inner loop, it runs $D$ steps of stochastic gradient descent (SGD) for an estimated solution $y_k^D$. Let $\Pi_{Q+1}^Q (\cdot) = I$. In the outer loop, StocBiO  samples data batches $\gD_F,\gD_H = \{ \gB_j, j = 1, ..., Q \}$ and $\gD_G$ and constructs $v_Q$ as an approximate solution of the linear system \eqref{eq:aid-step1} as follows:
\be\bad\label{est-vq}
v_0 = &  \nabla_y F(x_k,y_k^D;\gD_F), \\
v_Q =& \eta \sum_{q=-1}^{Q-1}\prod_{j=Q-q}^Q (I - \eta \nabla_y^2G(x_k,y_k^D;\gB_j)) v_0. 
\ead\ee

The stochastic hypergradient can be constructed as 
 \be\bad\label{estG}
 \widehat \nabla \Phi(x_k) =&  \nabla_x F(x_k,y_k^D;\gD_F)-\nabla^2_{xy} G(x_k,y_k^D;\gD_G)v_Q.
 \ead\ee

When the norm of the batch gradient is small (see step 12 of Algorithm \ref{alg:INEONstocbio}), we fix sample batches $\gD_F, \gD_G$ and call iNEON. Denote $f_{\gD_F} (x, y) = \frac{1}{D_f} \sum_{i = 1}^{D_f} F(x, y; \xi_i)$, $g_{\gD_G} (x, y) = \frac{1}{D_g} \sum_{i = 1}^{D_g} G(x, y; \zeta_i)$ and $\Phi_\gD(x) = f_{\gD_F} (x, y^*_{\gD_G}(x))$. iNEON finds the descent direction for $\Phi_\gD(x)$ at saddle points with high probability. We list the assumptions for Algorithm \ref{alg:INEONstocbio} as following.

\begin{assumption}\label{ass:stoc}
For the stochastic case, Assumption \ref{assu:bilevel} holds for $F(x, y; \xi)$ and $G(x, y; \zeta)$ for any given $\xi$ and $\zeta$.
\end{assumption}
\begin{assumption}\label{ass:variance-bound}
The variance of gradient $\nabla G(x, y;\zeta)$ is bounded:
\[
\BE_{\zeta} \|\nabla G(x, y;\zeta) - \nabla g(x, y) \|^2 \le \sigma^2.
\]
\end{assumption}

The following theorem provides our main results on stochastic bilevel optimization. 
\begin{theorem}\label{thm:stocneon}
Suppose Assumptions \ref{ass:stoc} and \ref{ass:variance-bound} hold. Set parameters as \eqref{eq:para_neon} and \eqref{eq:para_neon-2}, and let $\alpha = \frac{2}{\ell + \mu}$, $\beta = \frac{1}{4L_{\phi}}, D = \gO\left(\kappa\log (\epsilon^{-1})\right), Q = \gO\left(\kappa \log (\epsilon^{-1})\right), |\gB_{Q + 1 - j}| = BQ (1-\eta\mu)^{j-1}$, for $j = 1, ..., Q,$ where $ B = \gO\left(\kappa^2 \cdot\epsilon^{-2}\right)$, and $ S = \gO\left(\kappa^5 \cdot \epsilon^{-2}\right),  D_f = \tilde{\gO}\left(\kappa^2 \cdot\epsilon^{-2}\right), D_g = \tilde{\gO}\left(\kappa^6 \cdot \epsilon^{-2}\right) $ in Algorithm \ref{alg:INEONstocbio}. With high probability, the total iteration number of the Algorithm \ref{alg:INEONstocbio} for visiting an $\epsilon$-local minimum of \eqref{objective} can bounded by
\[
K=\tilde{\mathcal{O}} \left( \frac{L_\phi (\Phi(x_0) - \Phi(x^*))}{\epsilon^2} \right)= \tilde{\mathcal{O}}(\kappa^3\epsilon^{-2}),
\]
where $L_\phi$ is the Lipschitz constant of $\nabla \Phi(x)$, which is defined in Lemma \ref{grad-Lipschitz}. 
\end{theorem}

\begin{corollary}\label{cor:stocneon}
For Algorithm \ref{alg:INEONstocbio}, the complexities of gradient evaluations for finding an $\epsilon$-local minimum of \eqref{objective} are
\[
G(f, \epsilon) = \gO(\kappa^5 \epsilon^{-4}), \quad G(g, \epsilon) = \tilde{\gO}(\kappa^{10} \epsilon^{-4}),
\]
and the Jacobian- and Hessian-vector product complexities are
\[
JV(g, \epsilon) = \tilde{\gO}(\kappa^{9} \epsilon^{-4}), \quad HV(g, \epsilon) = \tilde{\gO}(\kappa^{9.5} \epsilon^{-4}).
\]
\end{corollary}

\begin{remark}
Compared with the StocBiO complexity results in \cite{ji2021bilevel}, our results have a worse dependence on the condition number $\kappa$. This is because we set a larger sample size $D_g$ in order to obtain a high probability result.
\end{remark}

\section{Numerical Experiments}

In this section we present the experimental results to demonstrate the efficiency of our algorithm. We reformulate the problem in \cite{du2017gradient} as a bilevel optimization problem \eqref{eq:objective} and then compare our Algorithm \ref{alg:perturbedalg} with AID-BiO in \cite{ji2021bilevel}. More precisely, we consider the following bilevel optimization problem:
\begin{align}
&\min_{x\in\mathbb{R}^{d}}\;\; \Phi(x):=f_1(x, y^*(x)), \nonumber
\\&\;\;\mbox{s.t.} \quad y^*(x)= \argmin_{y\in\mathbb{R}} f_2(x,y).\label{eq:exp_objective}
\end{align}
Motivated by \cite{du2017gradient}, we construct the following functions.
For the upper level function we have
\begin{equation}\label{eq: tube_f}
    f_1(x, y) = \left\{
    \begin{array}{ll}
       f_{i,1}(x, y) & x_1, ..., x_{i-1}\in [2\tau, 6\tau],\ x_i\in [0,\tau],\ x_{i+1},..., x_d\in [0,\tau],\ 1\leq i\leq d-1\\
       f_{i,2}(x,y) &  x_1, ..., x_{i-1}\in [2\tau, 6\tau],\ x_i\in [\tau, 2\tau],\ x_{i+1},..., x_d\in [0,\tau],\ 1\leq i\leq d-1\\
       f_{d,1}(x,y) & x_1, ..., x_{d-1}\in [2\tau, 6\tau],\ x_d\in [0,\tau] \\
       f_{d,2}(x,y) & x_1, ..., x_{d-1}\in [2\tau, 6\tau],\ x_d\in [\tau, 2\tau] \\
       f_{d+1, 1}(x,y) & x_1, ..., x_{d}\in [2\tau, 6\tau],
    \end{array}
    \right.
\end{equation}
where
\begin{align}
    f_{i,1}(x, y) &= \sum_{j=1}^{i-1}L(x_j - 4\tau)^2 -\gamma x_i^2 + \sum_{j=i+1}^{d}Lx_j^2 - (i-1)\nu,\ 1\leq i\leq d-1, \\
    f_{i,2}(x,y) &= \sum_{j=1}^{i-1}L(x_j - 4\tau)^2 + y + \sum_{j=i+2}^{d}Lx_j^2 - (i-1)\nu,\ 1\leq i \leq d-1, \\
    f_{d,1}(x, y) &= \sum_{j=1}^{d-1}L(x_j - 4\tau)^2 -\gamma x_d^2 - (d-1)\nu, \\
    f_{d,2}(x,y) &= \sum_{j=1}^{d-1}L(x_j - 4\tau)^2 + y - (d-1)\nu, \\
    f_{d+1, 1}(x, y) &= \sum_{j=1}^{d}L(x_j - 4\tau)^2 - d\nu.
\end{align}
The lower level function is defined as
\begin{equation}
    f_2(x, y) = \frac{y^2}{2} - g(x)y,
\end{equation}
where
\begin{align}
    g(x) &= \left\{
    \begin{array}{ll}
      h_1(x_i) + h_2(x_i)x_{i+1}^2 & x_1, ..., x_{i-1}\in [2\tau, 6\tau],\ x_i\in [\tau, 2\tau],\ x_{i+1},..., x_d\in [0,\tau],\\ 
      &1\leq i\leq d-1\\
      h_1(x_d) & x_1, ..., x_{d-1}\in [2\tau, 6\tau],\ x_d\in [\tau, 2\tau]\\
      0 & \text{elsewhere}
    \end{array}
    \right. \\
    h_1(c) &= -\gamma c^2 + \frac{(-14L + 10\gamma)(c-\tau)^3}{3\tau} + \frac{(5L - 3\gamma)(c-\tau)^4}{2\tau^2}\\
    h_2(c) &= -\gamma - \frac{10(L+\gamma)(c-2\tau)^3}{\tau^3} - \frac{15(L+\gamma)(c-2\tau)^4}{\tau^4} - \frac{6(L+\gamma)(c-2\tau)^5}{\tau^5}
\end{align}
The constants satisfy
\[
    L > 0,\ \gamma >0,\ \tau = e,\ \nu = -h_1(2\tau) + 4L\tau^2.
\]
Note that from \eqref{eq: tube_f} we know the function $\Phi(x)$ is only defined on the following domain (see also Eq. (5) in \cite{du2017gradient}):
\begin{equation}
    D_0 = \bigcup_{i=1}^{d+1}\left\{x\in \mathbb{R}^d: 6\tau \geq x_1,...,x_{i-1}\geq 2\tau, 2\tau\geq x_i\geq 0, \tau\geq x_{i+1},...,x_d\geq 0\right\}.
\end{equation}
By Lemma A.3 in \cite{du2017gradient} we know there are $d$ saddle points in $D_0$:
\[
    (0,...,0)^{\top}\ ,(4\tau, 0,..., 0)^{\top}\ ,..., (4\tau, ..., 4\tau,0)^{\top}.
\]
Moreover, the only local optimum is $(4\tau,...,4\tau)^{\top}.$ One can follow Steps 2 and 3 in Section A.1 of \cite{du2017gradient} to extend the domain to $\mathbb{R}^d$. For simplicity we omit the extension here. We refer the interested readers to Section 4 and Appendix of \cite{du2017gradient} for details of the motivation for constructing these functions. In our experiments, we choose the total number of iterations to be 1000 and all stepsizes to be 0.05 in both Algorithm \ref{alg:perturbedalg} and AID-BiO. 
Following \cite{du2017gradient}, we conduct the comparison using different choices of problem parameters. In Figure \ref{Figure: exp} we plot the learning curves of $\Phi(x) - \min \Phi(x)$ vs. Iteration number. Our algorithm is denoted as ``PBO'', and AID-BiO is denoted as ``BO'' in Figure \ref{Figure: exp}. Note that each learning curve is nearly a step function which consists of vertical and horizontal line segments. The horizontal segment indicates that the function value does not change and thus we may deduce that the iterates are stuck at a saddle point. Each vertical segment indicates that a perturbation successfully helps the iterate escape the saddle point. We observe that under different parameter choices our Algorithm \ref{alg:perturbedalg} escapes saddle points more efficiently than standard bilevel optimization algorithm.

\begin{figure*}[ht]
	\centering  
	\subfigure[]{\label{figure: 115}\includegraphics[width=0.24\textwidth]{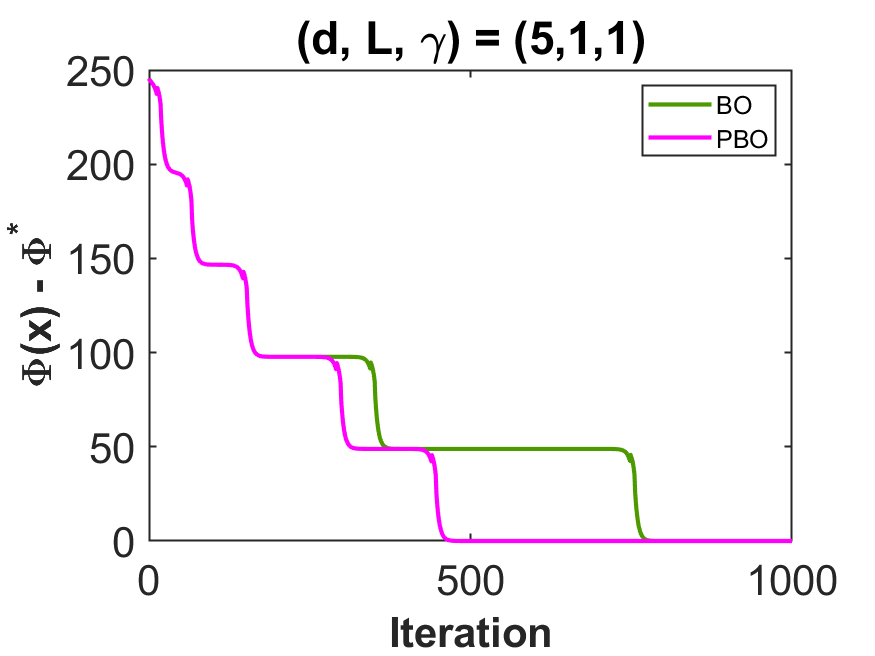}}
	\subfigure[]{\label{figure: 1515}\includegraphics[width=0.24\textwidth]{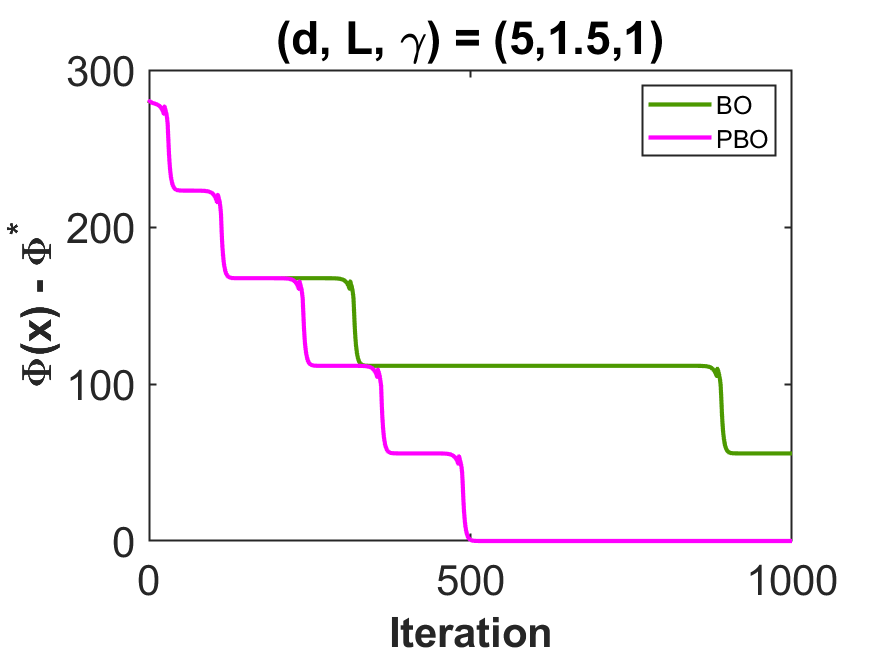}}
	\subfigure[]{\label{figure: 215}\includegraphics[width=0.24\textwidth]{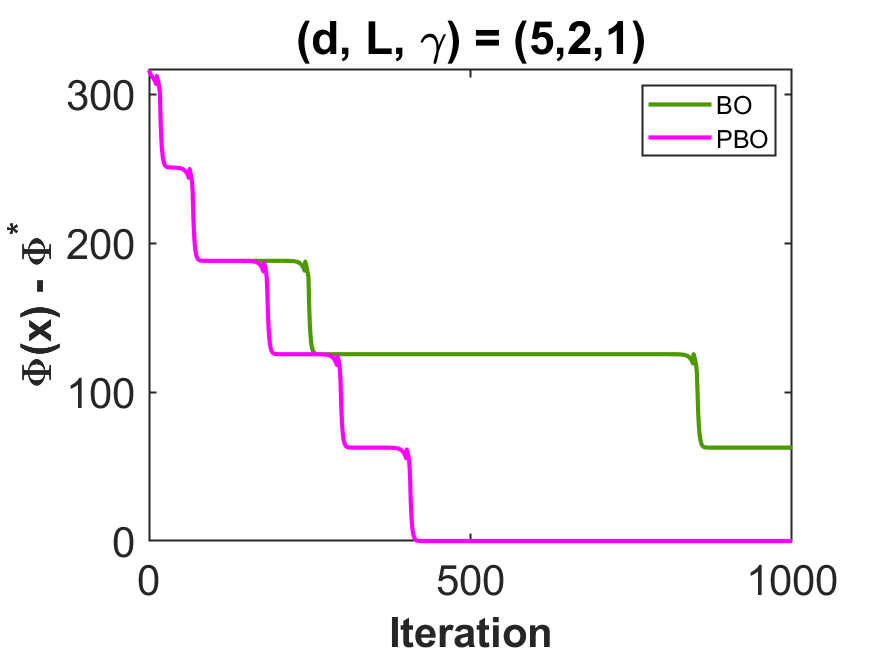}}
    \subfigure[]{\label{figure: 315}\includegraphics[width=0.24\textwidth]{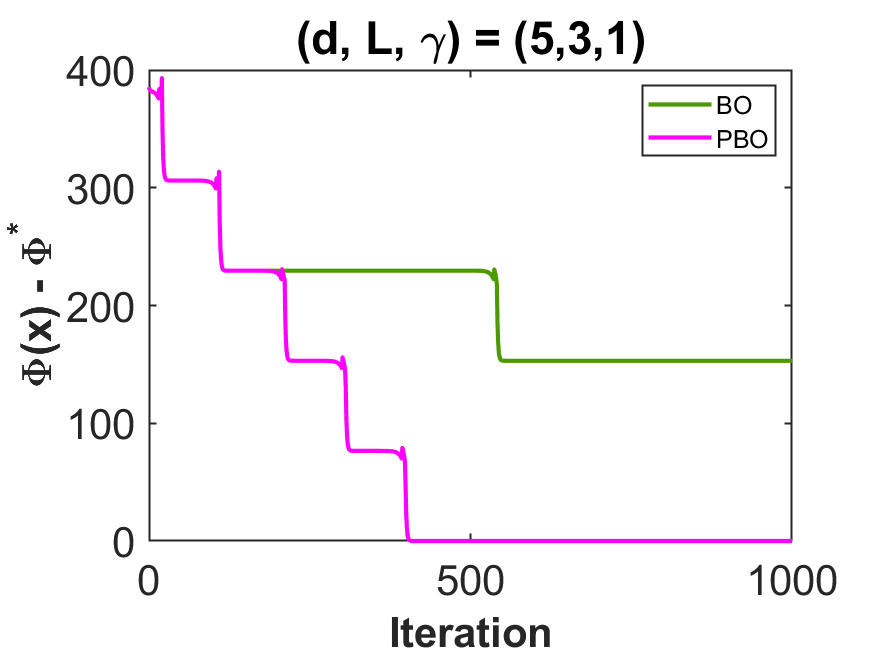}}
	
	\subfigure[]{\label{figure: 1110}\includegraphics[width=0.24\textwidth]{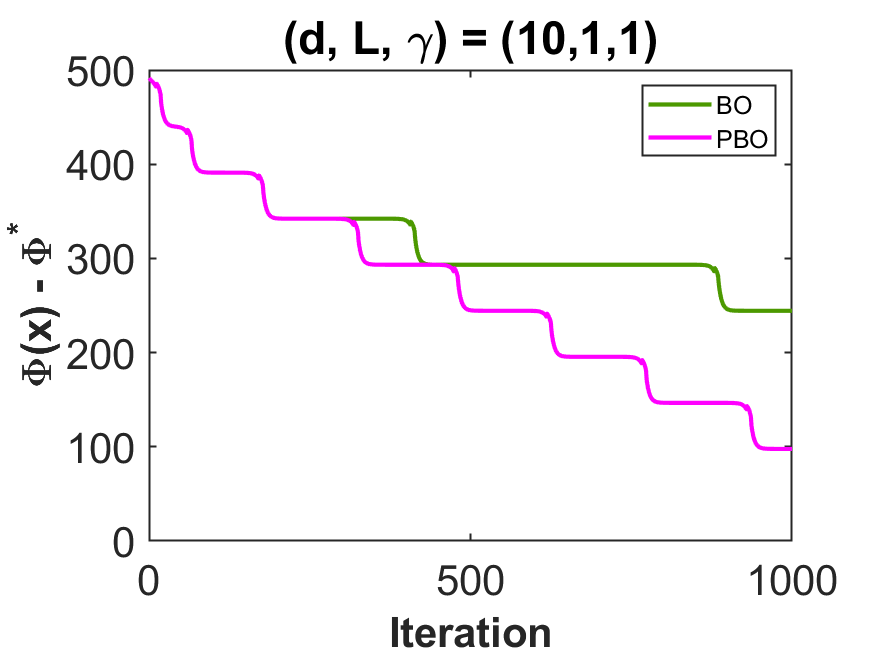}}
	\subfigure[]{\label{figure: 15110}\includegraphics[width=0.24\textwidth]{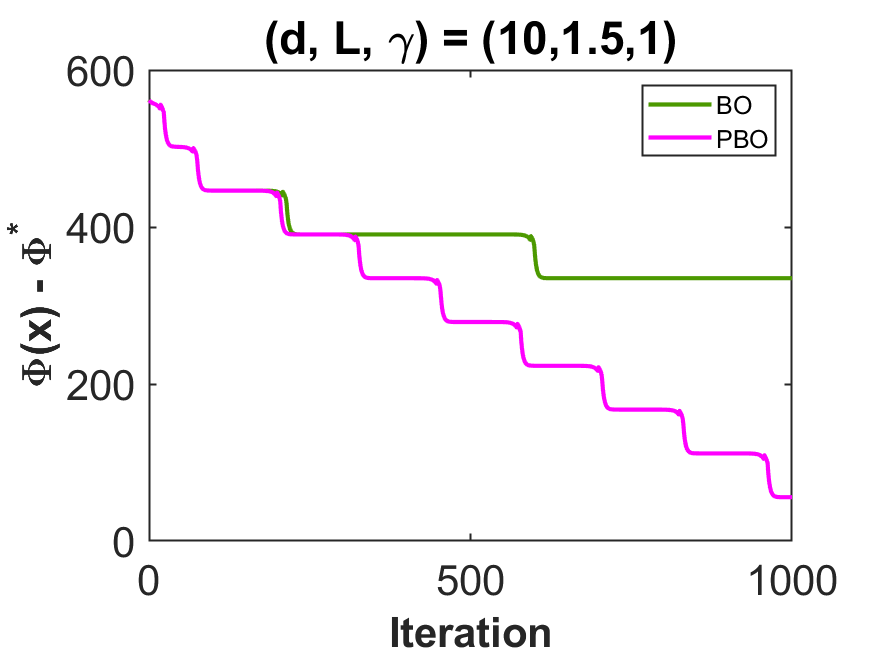}}
	\subfigure[]{\label{figure: 2110}\includegraphics[width=0.24\textwidth]{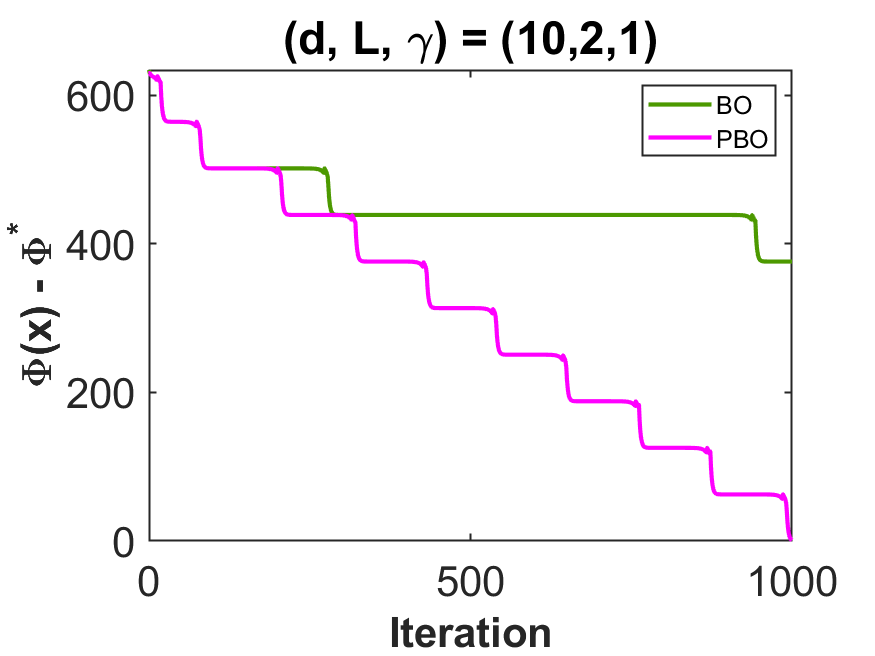}}
    \subfigure[]{\label{figure: 3110}\includegraphics[width=0.24\textwidth]{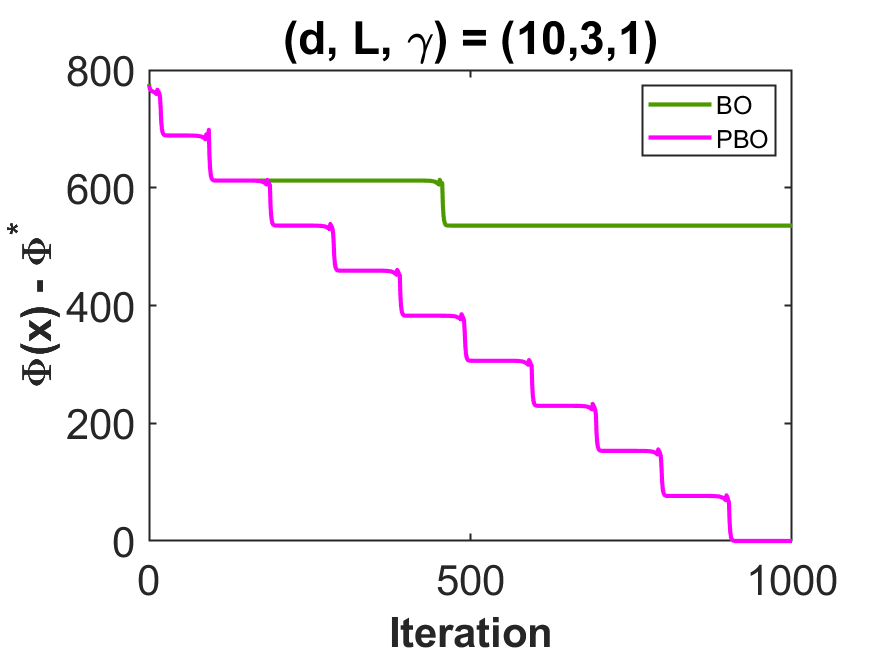}}

    \subfigure[]{\label{figure: 1120}\includegraphics[width=0.24\textwidth]{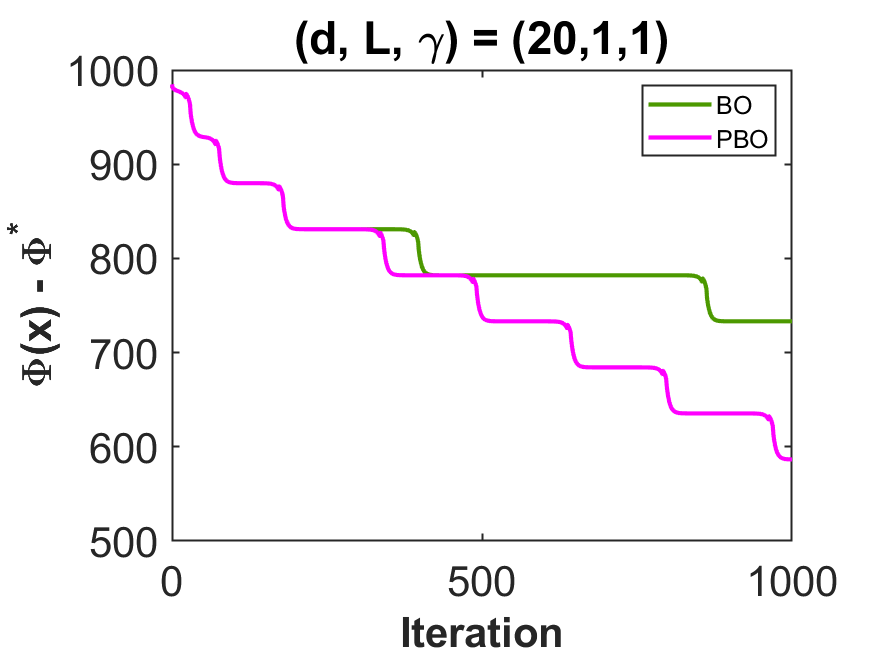}}
	\subfigure[]{\label{figure: 15120}\includegraphics[width=0.24\textwidth]{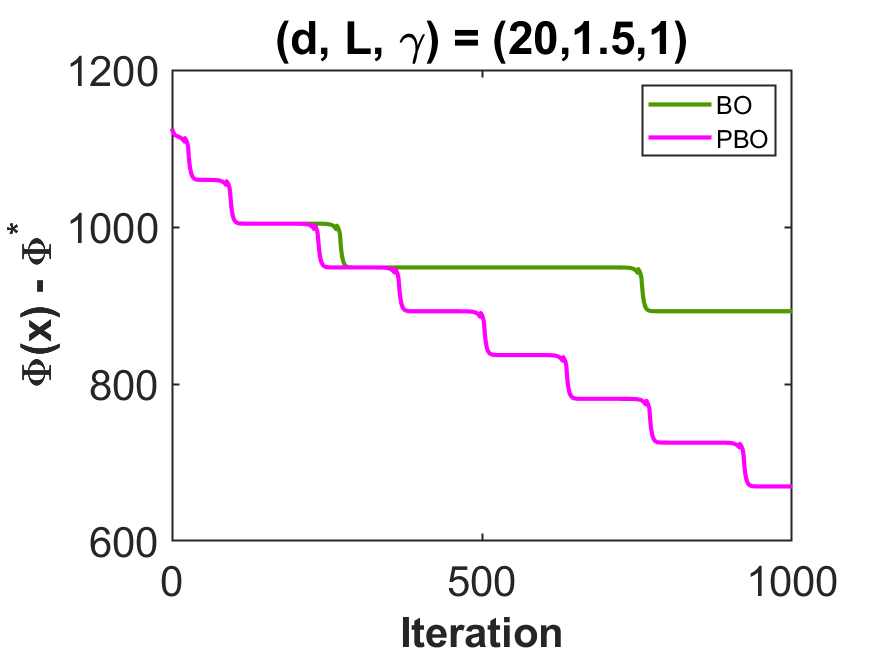}}
	\subfigure[]{\label{figure: 2120}\includegraphics[width=0.24\textwidth]{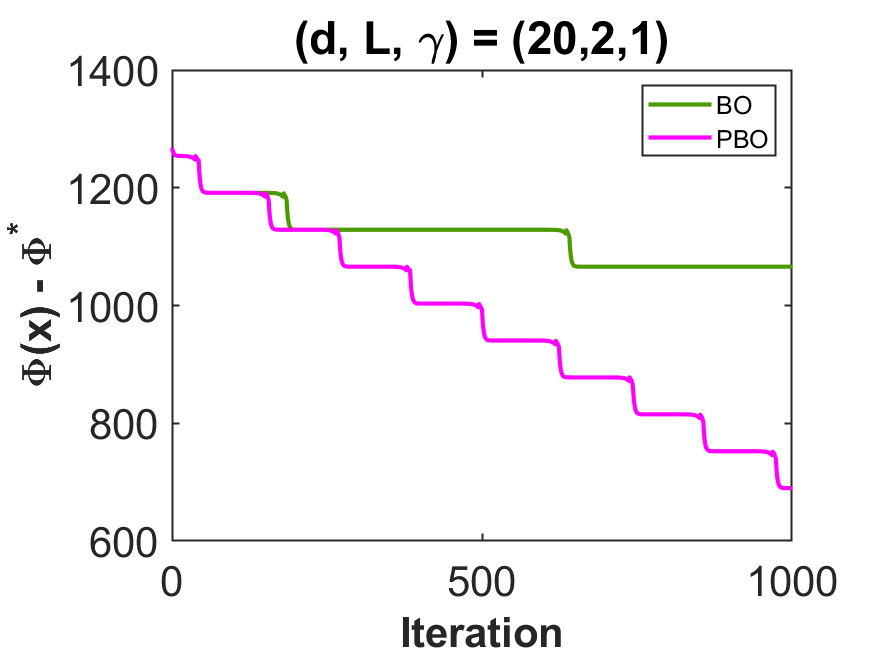}}
    \subfigure[]{\label{figure: 3120}\includegraphics[width=0.24\textwidth]{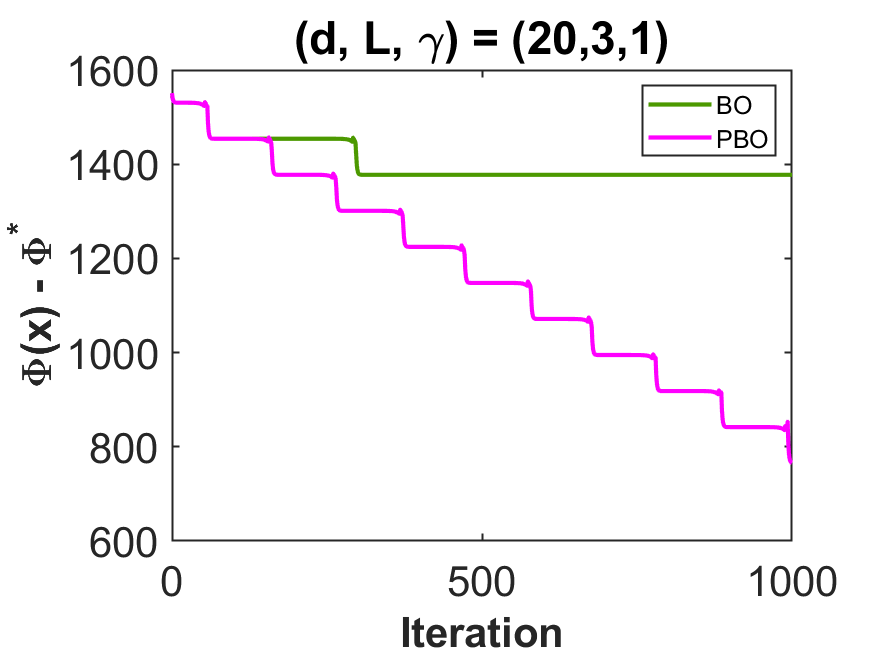}}
    
	\vspace{-0.2cm}
	\caption{Comparison between Algorithm \ref{alg:perturbedalg} and AID-BiO in \cite{ji2021bilevel}. PBO and BO represent Algorithm \ref{alg:perturbedalg} and AID-BiO respectively. $d$ is the dimension of the upper level function in \eqref{eq:objective}. $L$ and $\gamma$ are problem parameters used to generate the functions in both levels.}\label{Figure: exp}
	\vspace{-0.3cm}
\end{figure*}

\section{Conclusion}

In this paper, we have proposed the perturbed AID algorithm that provably converges to an $\epsilon$-local minimum in bilevel optimization. As a byproduct, we have provided the first nonasymptotic convergence rate for minimax problem converging to local minimax point with first-order method. Moreover, we have proposed inexact NEON that can extract negative gradient descent direction at saddle points. By combining the inexact NEON with StocBiO, we have proposed the first algorithm that converges to local minimum for stochastic bilevel optimization.

\bibliography{bilevel.bib}
\bibliographystyle{plain}


\newpage
\appendix
\onecolumn

\section{Preliminaries}

Under Assumption \ref{assu:bilevel}, we have the following proposition. The proof can be found in \cite{chen2021escaping}[Lemma 1]. 
\begin{proposition}\label{pro:bounds}
Suppose Assumption \ref{assu:bilevel} holds. We have the following bounds hold 
\begin{align*}
&\|\nabla_y^2 g(x, y) ^{-1}\| \le \frac{1}{\mu},  \\
&\max\left\{\|\nabla_x f(x, y)\|, \|\nabla_y f(x, y)\|, \|\nabla_y g(x, y)\|\right\} \le M, \\
&\max\left\{\|\nabla_x^2 f(x, y)\|, \|\nabla_y^2 f(x, y)\|, \|\nabla_{xy}^2 f(x, y)\|, \|\nabla_{xy}^2 g(x, y)\|, \|\nabla_y^2 g (x, y)\|\right\} \le \ell,  \\
 &\max\left\{\|\nabla_{xxy}^3 g(x, y)\|, \|\nabla_{yxy}^3 g(x, y)\|, \|\nabla_y^3 g (x, y)\|\right\} \le \rho. \end{align*}
\end{proposition}

Under Assumption \ref{assu:bilevel}, the gradient of $\Phi(x)$ is Lipschitz continuous.

\begin{lemma}\cite{ghadimi2018approximation}[Lemma 2.2] \label{grad-Lipschitz}
Suppose Assumption \ref{assu:bilevel} holds for $f(x, y)$ and $g(x, y)$, Then $\Phi (x)$ is $L_\phi$ smooth and the following inequality holds for any $x, x' \in \br^n$:
\be\bad\label{eq:phi-smooth}
\left\| \nabla \Phi (x) - \nabla \Phi (x') \right\| \le L_\phi \|x - x'\|, 
\ead\ee
where
\be\bad
L_\phi  = \ell + \frac{2\ell^2 + \rho M^2}{\mu} + \frac{\ell^3 + 2 \rho \ell M}{\mu^2} + \frac{\rho \ell^2M}{\mu^3}.
\ead\ee
\end{lemma}

We postpone the proof of Theorem \ref{thm:minimax} to Section \ref{sec:C} and focus on the general bilevel optimization first.

\section{Proofs of Results in Section \ref{sec:bilevel}}

\subsection{Proof of Lemma \ref{hess-Lipschitz}}

\begin{proof}
For general bilevel optimization problem \eqref{eq:objective}, the Hessian of function $\Phi(x)$ can be computed as 
\be\bad \label{phi-hessian}
\nabla^2 \Phi(x) = & \nabla_x^2 f(x, y^*(x)) + \frac{\partial y^*(x)}{\partial x} \cdot  \nabla_{yx}^2 f(x, y^*(x)) + \frac{\partial^2 y^*(x)}{\partial^2 x} \cdot  \nabla_y f(x, y^*(x)) \\ & +  \frac{\partial y^*(x)}{\partial x} \cdot  \nabla_{xy}^2 f(x, y^*(x))^\top +  \frac{\partial y^*(x)}{\partial x}\cdot \nabla_y^2 f(x, y^*(x))  \cdot  \frac{\partial y^*(x)}{\partial x}^\top,
\ead\ee
which is obtained by taking derivative on \eqref{phi-grad}. 
By further taking the derivative with respect to $x$ on \eqref{eq:opt-deri1}, we obtain:
\be\bad\label{eq:opt-deri2}
\nabla_{xxy}^3 g(x, y^*(x)) + \frac{\partial y^*(x)}{\partial x} \nabla_{yxy}^3 g(x, y^*(x)) + \frac{\partial^2 y^*(x)}{\partial^2 x} \cdot  \nabla_y^2 g(x, y^*(x)) \\
 +  \frac{\partial y^*(x)}{\partial x} \nabla_{xyy}^3 g(x, y^*(x)) +  \frac{\partial y^*(x)}{\partial x}\cdot \nabla_y^3 g(x, y^*(x))  \cdot  \frac{\partial y^*(x)}{\partial x}^\top = 0.
\ead\ee
By \eqref{phi-hessian}, we have 
\begin{align}
&\| \nabla^2 \Phi (x) -\nabla^2 \Phi (x')\| \nonumber \\
= & \left\| \nabla_x^2 f(x, y^*(x)) + \frac{\partial y^*(x)}{\partial x} \cdot  \nabla_{yx}^2 f(x, y^*(x)) + \frac{\partial^2 y^*(x)}{\partial^2 x} \cdot  \nabla_y f(x, y^*(x)) \right.\nonumber\\ 
& +  \frac{\partial y^*(x)}{\partial x} \cdot  \nabla^2_{xy} f(x, y^*(x))^\top + \frac{\partial y^*(x)}{\partial x}\cdot \nabla_y^2 f(x, y^*(x))  \cdot  \frac{\partial y^*(x)}{\partial x}^\top \nonumber\\
& - \left( \nabla_x^2 f(x', y^*(x')) + \frac{\partial y^*(x')}{\partial x'} \cdot  \nabla^2_{yx} f(x', y^*(x')) + \frac{\partial^2 y^*(x')}{\partial^2 x'} \cdot  \nabla_y f(x', y^*(x')) \right.  \nonumber\\
&\left. \left. +  \frac{\partial y^*(x')}{\partial x'} \cdot  \nabla^2_{xy} f(x', y^*(x'))^\top +  \frac{\partial y^*(x')}{\partial x'}\cdot \nabla_y^2 f(x', y^*(x'))  \cdot  \frac{\partial y^*(x')}{\partial x'}^\top\right)\right\|  \nonumber \\
\le &  \underbrace{\left\| \nabla_x^2 f(x, y^*(x)) -  \nabla_x^2 f(x', y^*(x')) \right\| }_{(I)}\nonumber \\
&  + 2 \underbrace{ \left\| \frac{\partial y^*(x)}{\partial x} \cdot  \nabla_{yx}^2 f(x, y^*(x))  -  \frac{\partial y^*(x')}{\partial x'} \cdot  \nabla_{yx}^2 f(x', y^*(x'))\right\|}_{(II)}  \nonumber\\
& +  \underbrace{\left\| \frac{\partial^2 y^*(x)}{\partial^2 x} \cdot  \nabla_y f(x, y^*(x))  - \frac{\partial^2 y^*(x')}{\partial^2 x'} \cdot  \nabla_y f(x', y^*(x')) \right\|}_{(III)}\\\label{lemma3.4-long}
& + \underbrace{ \left\|  \frac{\partial y^*(x)}{\partial x}\cdot \nabla_y^2 f(x, y^*(x))  \cdot  \frac{\partial y^*(x)}{\partial x}^\top  - \frac{\partial y^*(x')}{\partial x'}\cdot \nabla_y^2 f(x', y^*(x'))  \cdot  \frac{\partial y^*(x')}{\partial x'}^\top \right\| }_{(IV)},\nonumber
\end{align}
where the inequality is due to the triangle inequality and the fact that $\nabla_{yx}^2 f(x, y^*(x)) = \nabla_{xy}^2 f(x, y^*(x))^ \top$ for any smooth functions $f(x, y)$. We then bound the terms $(I) - (IV)$. By \cite{ghadimi2018approximation}[Lemma 2.2], we know that $y^*(x)$ is $\frac{\ell}{\mu}$-Lipschitz continuous. Therefore, we can bound the first term as 
\be\bad
(I) = & \left\| \nabla_x^2 f(x, y^*(x)) -  \nabla_x^2 f(x', y^*(x')) \right\|\\
\le &  \left\| \nabla_x^2 f(x, y^*(x)) -  \nabla_x^2 f(x', y^*(x)) \right\| +  \left\| \nabla_x^2 f(x', y^*(x)) -  \nabla_x^2 f(x', y^*(x')) \right\|\\
\le & \rho \left( \|x - x'\| + \|y^*(x) - y^*(x') \| \right)\le \rho \left(1 + \frac{\ell}{\mu}\right) \|x - x'\|,
\ead\ee
where the second inequality applies Assumption \ref{assu:bilevel} and the last inequality is obtained by the Lipschitz continuity of $y^*(x)$. For the second term $(II)$, we have the following bound:

\begin{align}
    \bad
        (II) = & 2\left\| \frac{\partial y^*(x)}{\partial x} \cdot  \nabla_{yx}^2 f(x, y^*(x))  -  \frac{\partial y^*(x')}{\partial x'} \cdot  \nabla_{yx}^2 f(x', y^*(x'))\right\|\\
\le & 2\left\| \frac{\partial y^*(x)}{\partial x} \cdot  \nabla_{yx}^2 f(x, y^*(x))  -  \frac{\partial y^*(x)}{\partial x} \cdot  \nabla_{yx}^2 f(x', y^*(x'))\right\|\\
& + 2\left\| \frac{\partial y^*(x)}{\partial x} \cdot  \nabla_{yx}^2 f(x', y^*(x'))  -  \frac{\partial y^*(x')}{\partial x'} \cdot  \nabla_{yx}^2 f(x', y^*(x'))\right\|\\
\le & 2  \left\| \frac{\partial y^*(x)}{\partial x} \right\| \left \|\nabla_{yx}^2 f(x, y^*(x))  -  \nabla_{yx}^2 f(x', y^*(x'))\right\|\\
& +  2\left\| \frac{\partial y^*(x)}{\partial x}  -  \frac{\partial y^*(x')}{\partial x'} \right\| \left\| \nabla_{yx}^2 f(x', y^*(x'))\right\|\\
\le & \frac{2\ell}{\mu} \left \|\nabla_{yx}^2 f(x, y^*(x))  -  \nabla_{yx}^2 f(x', y^*(x'))\right\| + 2\ell \left\| \frac{\partial y^*(x)}{\partial x}  -  \frac{\partial y^*(x')}{\partial x'} \right\|\\
 \le & \frac{2\ell}{\mu} \rho   \left( 1 + \frac{\ell}{\mu} \right)\left \|x - x'\right\| + 2\ell \left\| \frac{\partial y^*(x)}{\partial x}  -  \frac{\partial y^*(x')}{\partial x'} \right\|,
    \ead
\end{align}


where the second inequality is by the Cauchy-Schwarz inequality, and the last two  inequalities are obtained by Assumption \ref{assu:bilevel}. Moreover, we can bound $\left\| \frac{\partial y^*(x)}{\partial x}  -  \frac{\partial y^*(x')}{\partial x'} \right\|$ as 
\be\bad\label{eq:y-grad-diff}
&\left\| \frac{\partial y^*(x)}{\partial x}  -  \frac{\partial y^*(x')}{\partial x'} \right\|\\
 \le & \left\| \nabla_{xy}^2 g(x, y^*(x))\cdot \nabla^2_y g(x, y^*(x))^{-1} - \nabla_{xy}^2 g(x', y^*(x'))\cdot \nabla^2_y g(x', y^*(x'))^{-1} \right\|\\
\le &  \left\| \nabla_{xy}^2 g(x, y^*(x))\cdot \nabla^2_y g(x, y^*(x))^{-1} - \nabla_{xy}^2 g(x, y^*(x))\cdot \nabla^2_y g(x', y^*(x'))^{-1} \right\|\\
& +  \left\| \nabla_{xy} g(x, y^*(x))\cdot \nabla^2_y g(x', y^*(x'))^{-1} - \nabla_{xy}^2 g(x', y^*(x'))\cdot \nabla^2_y g(x', y^*(x'))^{-1} \right\|\\
\le & \ell  \left\| \nabla^2_y g(x, y^*(x))^{-1} - \nabla^2_y g(x', y^*(x'))^{-1} \right\| + \frac{1}{\mu} \left\| \nabla_{xy}^2 g(x, y^*(x)) - \nabla_{xy}^2 g(x', y^*(x'))\right\|\\
\le & \frac{\ell}{\mu^2}  \left\| \nabla^2_y g(x, y^*(x)) - \nabla^2_y g(x', y^*(x')) \right\| + \frac{1}{\mu} \left\| \nabla_{xy}^2 g(x, y^*(x)) - \nabla_{xy}^2 g(x', y^*(x'))\right\|\\
\le & \left(\frac{\ell}{\mu^2} + \frac{1}{\mu}\right) \rho \left( \|x - x'\| + \|y^*(x) -y^*(x') \| \right)\\
\le &\frac{\rho}{\mu} \left(1 + \frac{\ell}{\mu}\right)^2 \|x - x'\|,
\ead\ee
where the first inequality is by equation \eqref{eq:opt-deri1}, the third inequality applies Proposition \ref{pro:bounds} and the forth inequality follows the fact that $\|X^{-1} - Y^{-1}\| \le \|X^{-1}\| \|X - Y\| \|Y^{-1}\|$ and Proposition \ref{pro:bounds}.
Therefore, we have the following bound for the second term:
\be\bad
(II)  \le &\left( \frac{2\ell\rho}{\mu}    \left( 1 + \frac{\ell}{\mu} \right) +  \frac{2\ell\rho}{\mu} \left( 1 + \frac{\ell}{\mu} \right)^2 \right)\left \|x - x'\right\|.
\ead\ee
The third term $(III)$ can be bounded by:
\be\bad\label{eq:termIII}
(III) = &  \left\| \frac{\partial^2 y^*(x)}{\partial^2 x} \cdot  \nabla_y f(x, y^*(x))  - \frac{\partial^2 y^*(x')}{\partial^2 x'} \cdot  \nabla_y f(x', y^*(x')) \right\| \\
\le &  \left\| \frac{\partial^2 y^*(x)}{\partial^2 x} \right\| \left\| \nabla_y f(x, y^*(x))  - \nabla_y f(x', y^*(x')) \right\| \\
&+  \|\nabla_y f(x', y^*(x')) \| \cdot \left\|\frac{\partial^2 y^*(x)}{\partial^2 x}  - \frac{\partial^2 y^*(x')}{\partial^2 x'} \right\|.
\ead\ee
Therefore, we need to give upper bounds for both $\left\|\frac{\partial^2 y^*(x)}{\partial^2 x} \right\| $ and $\left\|\frac{\partial^2 y^*(x)}{\partial^2 x}  -\frac{\partial^2 y^*(x')}{\partial^2 x'} \right\| $. Note that equation \eqref{eq:opt-deri2} yields
\be\bad\label{eq:hess-y-expression}
\frac{\partial^2 y^*(x)}{\partial^2 x} = &- \left[ \nabla_{xxy}^3 g(x, y^*(x)) + \frac{\partial y^*(x)}{\partial x} \nabla_{yxy}^3 g(x, y^*(x))  +  \frac{\partial y^*(x)}{\partial x} \nabla_{xyy}^3 g(x, y^*(x)) \right.\\ 
& \left. +  \frac{\partial y^*(x)}{\partial x}\cdot \nabla_y^3 g(x, y^*(x))  \cdot  \frac{\partial y^*(x)}{\partial x}^\top \right]  \nabla_y^2 g(x, y^*(x))^{-1},
\ead\ee
which, combined with Proposition \ref{pro:bounds}, leads to
\be\bad\label{eq:y-hess-bound}
\left\|\frac{\partial^2 y^*(x)}{\partial^2 x} \right\| \le  \frac{1}{\mu} \left[ \rho + 2 \frac{\ell}{\mu} \cdot \rho +  \left(\frac{\ell}{\mu} \right)^2 \cdot \rho\right] = \frac{\rho}{\mu} \left( 1 + \frac{\ell}{\mu} \right)^2.
\ead\ee
Furthermore, $\left\|\frac{\partial^2 y^*(x)}{\partial^2 x}  -\frac{\partial^2 y^*(x')}{\partial^2 x'} \right\|$ can be upper bounded by 
\begin{align}\label{eq:y-hess-diff}
\tiny
&\left\|\frac{\partial^2 y^*(x)}{\partial^2 x}  -\frac{\partial^2 y^*(x')}{\partial^2 x'} \right\| \nonumber\\
 \le & \left\|  \left[ \nabla_{xxy}^3 g(x, y^*(x)) + \frac{\partial y^*(x)}{\partial x} \nabla_{yxy}^3 g(x, y^*(x))  +  \frac{\partial y^*(x)}{\partial x} \nabla_{xyy}^3 g(x, y^*(x)) \right.\right. \nonumber \\ 
&  \left.+  \frac{\partial y^*(x)}{\partial x}\cdot \nabla_y^3 g(x, y^*(x))  \cdot  \frac{\partial y^*(x)}{\partial x}^\top \right]  \nabla_y^2 g(x, y^*(x))^{-1} \nonumber\\
& - \left[ \nabla_{xxy}^3 g(x', y^*(x')) + \frac{\partial y^*(x')}{\partial x'} \nabla_{yxy}^3 g(x', y^*(x')) +  \frac{\partial y^*(x')}{\partial x'} \nabla_x \nabla_y  \nabla_y g(x', y^*(x'))  \right. \nonumber\\
&\left.\left.  +  \frac{\partial y^*(x')}{\partial x'}\cdot \nabla_y^3 g(x', y^*(x'))  \cdot  \frac{\partial y^*(x')}{\partial x'}^\top \right]  \nabla_y^2 g(x', y^*(x'))^{-1}  \right\| \nonumber\\
 \le &  \left\| \left[ \nabla_{xxy}^3 g(x, y^*(x)) + \frac{\partial y^*(x)}{\partial x} \nabla_{yxy}^3 g(x, y^*(x))  +  \frac{\partial y^*(x)}{\partial x} \nabla_{xyy}^3 g(x, y^*(x)) \right. \right.  \nonumber\\
& \left. +  \frac{\partial y^*(x)}{\partial x}\cdot \nabla_y^3 g(x, y^*(x))  \cdot  \frac{\partial y^*(x)}{\partial x}^\top \right]  - \left[ \nabla_{xxy}^3 g(x', y^*(x')) + \frac{\partial y^*(x')}{\partial x'} \nabla_{yxy}^3 g(x', y^*(x')) \right. \nonumber\\
&\left.\left. +  \frac{\partial y^*(x')}{\partial x'} \nabla_{xyy}^3 g(x', y^*(x'))  +  \frac{\partial y^*(x')}{\partial x'}\cdot \nabla_y^3 g(x', y^*(x'))  \cdot  \frac{\partial y^*(x')}{\partial x'}^\top \right] \right\| \left\| \nabla_y^2 g(x', y^*(x'))^{-1}  \right\| \nonumber\\
&+ \left\|  \left[ \nabla_{xxy}^3 g(x, y^*(x)) + \frac{\partial y^*(x)}{\partial x} \nabla_{yxy}^3 g(x, y^*(x))  +  \frac{\partial y^*(x)}{\partial x} \nabla_{xyy}^3 g(x, y^*(x)) \right.\right. \nonumber \\ 
&  \left.\left.+  \frac{\partial y^*(x)}{\partial x}\cdot \nabla_y^3 g(x, y^*(x))  \cdot  \frac{\partial y^*(x)}{\partial x}^\top \right] \right\| \left\|  \nabla_y^2 g(x, y^*(x))^{-1} - \nabla_y^2 g(x', y^*(x'))^{-1} \right\| \nonumber\\
\le & \left[ \frac{1}{\mu} \|\nabla_{xxy}^3 g(x, y^*(x))  - \nabla_{xxy}^3 g(x', y^*(x')) \| \right.\nonumber\\
& +  \frac{2}{\mu} \left\| \frac{\partial y^*(x)}{\partial x} \nabla_{yxy}^3 g(x, y^*(x))  - \frac{\partial y^*(x')}{\partial x'} \nabla_{yxy}^3 g(x', y^*(x'))  \right\| \nonumber\\ 
& \left. + \frac{1}{\mu} \left\| \frac{\partial y^*(x)}{\partial x}\cdot \nabla_y^3 g(x, y^*(x))  \cdot  \frac{\partial y^*(x)}{\partial x}^\top - \frac{\partial y^*(x')}{\partial x'}\cdot \nabla_y^3 g(x', y^*(x'))  \cdot  \frac{\partial y^*(x')}{\partial x'}^\top  \right\| \right]\nonumber\\
& + \rho \left( 1 + \frac{\ell}{\mu} \right)^2 \left\|  \nabla_y^2 g(x, y^*(x))^{-1}  -  \nabla_y^2 g(x', y^*(x'))^{-1}  \right\|\nonumber\\
\le & \left[ \frac{\nu}{\mu} \left( 1 + \frac{\ell}{\mu} \right) \|x - x'\| +  \frac{2}{\mu} \left\| \frac{\partial y^*(x)}{\partial x} \nabla_{yxy}^3 g(x, y^*(x))  - \frac{\partial y^*(x')}{\partial x'} \nabla_{yxy}^3 g(x', y^*(x'))  \right\| \right.\nonumber\\ 
& \left. + \frac{1}{\mu} \left\| \frac{\partial y^*(x)}{\partial x}\cdot \nabla_y^3 g(x, y^*(x))  \cdot  \frac{\partial y^*(x)}{\partial x}^\top - \frac{\partial y^*(x')}{\partial x'}\cdot \nabla_y^3 g(x', y^*(x'))  \cdot  \frac{\partial y^*(x')}{\partial x'}^\top  \right\| \right] \nonumber\\
&+ \left(\frac{\rho}{\mu}\right)^2 \left( 1 + \frac{\ell}{\mu} \right)^3 \|x - x'\|,
\end{align}
where in the third inequality, we used Proposition \ref{pro:bounds} and the last inequality is due to \eqref{eq:y-hess-bound}, Assumption \ref{assu:bilevel}, and  $\|X^{-1} - Y^{-1}\| \le \|X^{-1}\| \|X - Y\| \|Y^{-1}\|$. To bound the term 
\[
\left\| \frac{\partial y^*(x)}{\partial x} \nabla_{yxy}^3 g(x, y^*(x))  - \frac{\partial y^*(x')}{\partial x} \nabla_{yxy}^3 g(x', y^*(x'))  \right\|,
\] 
we follow the computation of $(II)$ and get
\be\bad
&\left\| \frac{\partial y^*(x)}{\partial x} \nabla_{yxy}^3 g(x, y^*(x))  - \frac{\partial y^*(x)}{\partial x} \nabla_{yxy}^3 g(x, y^*(x))  \right\| \\
 \le & \left( \frac{\ell\nu}{\mu}  \left( 1 + \frac{\ell}{\mu} \right) +  \frac{2\rho^2}{\mu} \left( 1 + \frac{\ell}{\mu} \right)^2 \right)\left \|x - x'\right\|.
\ead\ee
Moreover, we have
\be\bad \label{eq:term-pgp}
&\left\| \frac{\partial y^*(x)}{\partial x}\cdot \nabla_y^3 g(x, y^*(x))  \cdot  \frac{\partial y^*(x)}{\partial x}^\top - \frac{\partial y^*(x')}{\partial x'}\cdot \nabla_y^3 g(x', y^*(x'))  \cdot  \frac{\partial y^*(x')}{\partial x'}^\top  \right\| \\
\le & \left\| \frac{\partial y^*(x)}{\partial x}\cdot \nabla_y^3 g(x, y^*(x))  \cdot  \frac{\partial y^*(x)}{\partial x}^\top - \frac{\partial y^*(x')}{\partial x'}\cdot \nabla_y^3 g(x, y^*(x))  \cdot  \frac{\partial y^*(x)}{\partial x}^\top  \right\| \\
& + \left\| \frac{\partial y^*(x')}{\partial x'}\cdot \nabla_y^3 g(x, y^*(x))  \cdot  \frac{\partial y^*(x)}{\partial x}^\top - \frac{\partial y^*(x')}{\partial x'}\cdot \nabla_y^3 g(x', y^*(x'))  \cdot  \frac{\partial y^*(x)}{\partial x}^\top  \right\| \\
& +\left\| \frac{\partial y^*(x')}{\partial x'}\cdot \nabla_y^3 g(x', y^*(x'))  \cdot  \frac{\partial y^*(x)}{\partial x}^\top - \frac{\partial y^*(x')}{\partial x'}\cdot \nabla_y^3 g(x', y^*(x'))  \cdot  \frac{\partial y^*(x')}{\partial x'}^\top  \right\| \\
\le & 2 \cdot \frac{\rho \ell}{\mu} \left\| \frac{\partial y^*(x)}{\partial x} -  \frac{\partial y^*(x')}{\partial x'} \right\| + \left( \frac{ \ell}{\mu}\right)^2 \left\| \nabla_y^3 g(x, y^*(x)) -  \nabla_y^3 g(x', y^*(x'))  \right\| \\
\le & \left( \frac{2\ell \rho^2 }{\mu^2} \left( 1 + \frac{\ell}{\mu} \right)^2 +  \nu \left( \frac{ \ell}{\mu}\right)^2\left( 1 + \frac{\ell}{\mu} \right) \right) \|x - x'\|,
\ead\ee
where the second inequality is due to $\left\| \frac{\partial y^*(x)}{\partial x}\right\| \le \frac{\ell}{\mu}$ and Proposition \ref{pro:bounds} and the last inequality is obtained by \eqref{eq:y-grad-diff} and Assumption \ref{assu:bilevel}. Combining \eqref{eq:y-hess-diff} - \eqref{eq:term-pgp} leads to
\begin{align}\label{eq:y-hess-bound-final}
& \left\|\frac{\partial^2 y^*(x)}{\partial^2 x}  -\frac{\partial^2 y^*(x')}{\partial^2 x'} \right\| \nonumber\\
 \le & \left[\left( \frac{\nu}{\mu} + \frac{2\ell\nu}{\mu^2}  + \frac{\ell^2\nu }{\mu^3}\right) \left( 1 + \frac{\ell}{\mu} \right)  + \left( \frac{4\rho^2}{\mu^2} + \frac{2\ell\rho^2}{\mu^3}   \right) \left( 1 + \frac{\ell}{\mu} \right)^2  + \frac{\rho^2}{\mu^2}  \left( 1 + \frac{\ell}{\mu} \right)^3\right] \|x - x' \|.
\end{align}
Combining \eqref{eq:termIII}, \eqref{eq:y-hess-bound}, and \eqref{eq:y-hess-bound-final} yields
\be\bad
(III) \le & \frac{\rho \ell}{\mu} \left( 1 + \frac{\ell}{\mu} \right)^3 \|x - x'\| + M \left\|\frac{\partial^2 y^*(x)}{\partial^2 x}  -\frac{\partial^2 y^*(x')}{\partial^2 x'} \right\|\\
\le &  \left[M \left( \frac{\nu}{\mu} + \frac{2\ell\nu}{\mu^2}  + \frac{\ell^2\nu }{\mu^3}\right) \left( 1 + \frac{\ell}{\mu} \right)  + M  \left( \frac{4\rho^2}{\mu^2} + \frac{2\ell\rho^2}{\mu^3}   \right) \left( 1 + \frac{\ell}{\mu} \right)^2 \right.\\
&\left.  +\left( \frac{M\rho^2}{\mu^2}  +   \frac{\rho\ell}{\mu}  \right) \left( 1 + \frac{\ell}{\mu} \right)^3\right] \|x - x' \|.
\ead\ee
Finally, similar to the computation in \eqref{eq:term-pgp}, we have 
\be\bad
(IV) = & \left\|  \frac{\partial y^*(x)}{\partial x}\cdot \nabla_y^2 f(x, y^*(x))  \cdot  \frac{\partial y^*(x)}{\partial x}^\top  - \frac{\partial y^*(x')}{\partial x'}\cdot \nabla_y^2 f(x', y^*(x'))  \cdot  \frac{\partial y^*(x')}{\partial x'}^\top \right\| \\
\le & 2 \cdot \frac{\ell^2}{\mu} \left\| \frac{\partial y^*(x)}{\partial x} -  \frac{\partial y^*(x')}{\partial x'} \right\| + \left( \frac{ \ell}{\mu}\right)^2 \left\| \nabla_y^2 f(x, y^*(x)) -  \nabla_y^2 f(x', y^*(x'))  \right\| \\
\le & \left( \frac{2\ell^2 \rho }{\mu^2} \left( 1 + \frac{\ell}{\mu} \right)^2 +  \rho \left( \frac{ \ell}{\mu}\right)^2\left( 1 + \frac{\ell}{\mu} \right) \right) \|x - x'\|.
\ead\ee
The bounds of $(I) - (IV)$ together lead to 
\begin{align}
\rho_\phi = & \left[\left(\rho +  \frac{2 \ell \rho + M\nu}{\mu} + \frac{2M \ell\nu + \rho \ell^2}{\mu^2}  + \frac{M\ell^2\nu }{\mu^3}\right) \left( 1 + \frac{\ell}{\mu} \right)  \right. \label{def-rho-phi}\\ 
&\left. +  \left( \frac{2\ell\rho}{\mu} +  \frac{4M\rho^2 + 2\ell^2 \rho}{\mu^2} + \frac{2M\ell\rho^2}{\mu^3}   \right) \left( 1 + \frac{\ell}{\mu} \right)^2  +\left( \frac{M\rho^2}{\mu^2}  +   \frac{\rho \ell}{\mu}  \right) \left( 1 + \frac{\ell}{\mu} \right)^3\right],\nonumber
\end{align}
which completes the proof. 
\end{proof}

In our proof of the main theorem, it is crucial to bound the estimation error for the hypergradient. For the AID method, we have the following lemma.

{\begin{lemma}\label{lem:grad-err-aid}
Suppose Assumption \ref{assu:bilevel} holds. For the AID approach (i.e., the \textbf{AID} option in Algorithm \ref{alg:perturbedalg}) with parameters $D = \gO(\kappa), N = \gO(\sqrt{\kappa})$ , we have:
\be\bad\label{eq:aid-error}
& \left\|\widehat \nabla \Phi(x_k)- \nabla \Phi(x_k)\right\| \le  \left(\left( 1 + \frac{\ell}{\mu}\left(1 +  2\sqrt{\kappa}\right)\right)\left( \ell +  \frac{\rho M}{\mu} \right)  \left(1-\frac{\mu}{\ell}\right)^{\frac{D}{2}} + 2\ell\sqrt{\kappa}\Big( \frac{\sqrt{\kappa}-1}{\sqrt{\kappa}+1} \Big)^N \right) \Gamma_1,\\
\ead\ee
where 
\be\label{def-gamma1-AID}
 \Gamma_1 = \widehat{\Delta} + 2\eta \left( \kappa^2 + 2\kappa + \frac{\rho M (1+ \kappa)}{\mu^2}\right)\left(M + \frac{\ell M}{\mu}\right),
\ee
\be\label{def-hatDelta-AID}
\widehat{\Delta} = \|y^0 - y^*(x_0)\| + \| v^0 - v^*_0 \|,
\ee
and $\widehat v_k=\nabla_y^2g(x_k,y^D_k)^{-1}\nabla_y f(x_k,y^D_k)$.
\end{lemma}
\begin{proof}
First note that the convergence rates of CG for the quadratic programming (see, e.g., eq. (17) in~\cite{grazzi2020iteration}) and GD for strongly convex optimization (see, e.g., Theorem 2.1.14 in \cite{NesterovConvexBook2004}) yield
\begin{align}
 \|v_k^N-\widehat v_k\| \leq 2\sqrt{\kappa}\Big( \frac{\sqrt{\kappa}-1}{\sqrt{\kappa}+1} \Big)^N\|v_k^0-\widehat v_k\|,\label{eq:cgrate}\\
 \|y_k^{D} -y^*(x_k)\| \leq \left(1-\frac{\mu}{\ell}\right)^{\frac{D}{2}} \|y^0_k-y^*(x_k)\|\label{eq:gdrate}.
\end{align} 
Denote $v_k^* = \nabla_y^2g(x_k,y^*(x_k))^{-1} \nabla_y f(x_k,y^*(x_k))$. Note that $\nabla \Phi(x_k)$ is defined in \eqref{eq:bilevel-grad}, i.e.,
\[\nabla\Phi(x_k) =  \nabla_x f(x_k, y^*(x_k)) - \nabla_{xy}^2 g(x_k, y^*(x_k)) v_k^*,\]
and $\widehat{\nabla}\Phi(x)$ is defined in step 9 of Algorithm \ref{alg:perturbedalg}, i.e., 
\[\widehat{\nabla}\Phi(x_k) = \nabla_x f(x_k,y_k^D) - \nabla_{xy}^2 g(x_k,y_k^D)v_k^N.\]
 We then have the following inequality holds
\begin{align}\label{eq:erroraid}
&\left\|\widehat \nabla \Phi(x_k)- \nabla \Phi(x_k)\right\| \nonumber\\
\leq& \|\nabla_x f(x_k,y^*(x_k))-\nabla_x f(x_k,y_k^D)\| + \|\nabla_{xy}^2 g(x_k,y_k^D)\|\|v_k^*-v_k^N\|  \\
&+ \|\nabla_{xy}^2 g(x_k,y^*(x_k))-\nabla_{xy}^2 g(x_k,y_k^D) \|  \|v_k^*\| \nonumber\\
\le &\ell \|y^*(x_k)-y_k^D\| + \ell \|v_k^*-v_k^N\|+ \rho \|v_k^*\|\|y_k^D-y^*(x_k)\|\nonumber\\
\leq&  \left( \ell +  \frac{\rho M}{\mu} \right) \|y^*(x_k)-y_k^D\| + \ell\|v_k^*-v_k^N\|\label{eq:erroraid_new}.
\end{align}
Here the last inequality follows from $\|v_k^*\|\leq\|(\nabla_y^2g(x_k,y^*(x_k)))^{-1}\|\|\nabla_y f(x_k,y^*(x_k))\|\leq \frac{M}{\mu}$ in which we use Proposition \ref{pro:bounds}. Next we give an upper bound of $\|v_k^*-v_k^N\|$:
 \begin{align}\label{eq:v-bound}
 &\|v_k^*-v_k^N\|\nonumber \leq \|v_k^*-\widehat v_k\| + \|v_k^N-\widehat v_k\|\nonumber\\
 \le & \|v_k^*-\widehat v_k\| + 2\sqrt{\kappa}\Big( \frac{\sqrt{\kappa}-1}{\sqrt{\kappa}+1} \Big)^N\|v_k^0-\widehat v_k\|\nonumber\\
 \le & \left(1 +  2\sqrt{\kappa}\Big( \frac{\sqrt{\kappa}-1}{\sqrt{\kappa}+1} \Big)^N\right) \|v_k^*-\widehat v_k\| + 2\sqrt{\kappa}\Big( \frac{\sqrt{\kappa}-1}{\sqrt{\kappa}+1} \Big)^N\|v_k^0- v_k^*\|\nonumber\\
 = & \left(1 +  2\sqrt{\kappa}\Big( \frac{\sqrt{\kappa}-1}{\sqrt{\kappa}+1} \Big)^N\right)\|\nabla_y^2g(x_k,y^D_k)^{-1}\nabla_y f(x_k,y^D_k) -\nabla_y^2g(x_k,y^*(x_k))^{-1}\nabla_y f(x_k,y^*(x_k))\|\nonumber\\
 & + 2\sqrt{\kappa}\Big( \frac{\sqrt{\kappa}-1}{\sqrt{\kappa}+1} \Big)^N\|v_k^0- v_k^*\|\nonumber\\
\le &\left(1 +  2\sqrt{\kappa}\Big( \frac{\sqrt{\kappa}-1}{\sqrt{\kappa}+1} \Big)^N\right)\Big(\frac{\ell}{\mu} +\frac{\rho M}{\mu^2} \Big)\|y^D_k-y^*(x_k)\| + 2\sqrt{\kappa}\Big( \frac{\sqrt{\kappa}-1}{\sqrt{\kappa}+1} \Big)^N\|v_k^0- v_k^*\|\nonumber\\
\le &\left(1 +  2\sqrt{\kappa}\right)\Big(\frac{\ell}{\mu} +\frac{\rho M}{\mu^2} \Big)\|y^D_k-y^*(x_k)\| + 2\sqrt{\kappa}\Big( \frac{\sqrt{\kappa}-1}{\sqrt{\kappa}+1} \Big)^N\|v_k^0- v_k^*\|,
 \end{align}
where the second inequality is due to \eqref{eq:cgrate} and in the second to last inequality we have used Assumption \ref{assu:bilevel}, Proposition \ref{pro:bounds}, and $\|X^{-1}-Y^{-1}\|\leq \|X^{-1}\|\cdot\|X-Y\|\cdot\|Y^{-1}\|$. Plugging \eqref{eq:gdrate} and \eqref{eq:v-bound} into \eqref{eq:erroraid_new} leads to
\be\bad\label{eq:aid-error1}
\left\|\widehat \nabla \Phi(x_k)- \nabla \Phi(x_k)\right\| \leq& \left( 1 + \frac{\ell}{\mu}\left(1 +  2\sqrt{\kappa}\right)\right)\left( \ell +  \frac{\rho M}{\mu} \right)  \left(1-\frac{\mu}{\ell}\right)^{\frac{D}{2}}\|y^0_k-y^*(x_k)\| \\
&+ 2\ell\sqrt{\kappa}\Big( \frac{\sqrt{\kappa}-1}{\sqrt{\kappa}+1} \Big)^N\|v_k^0- v_k^*\|.
\ead\ee
Secondly, by the warm-start strategy $y_k^0 = y_{k-1}^D, v_k^0 = v_{k-1}^N$ in the inner loop, we have 
\be\bad\label{eq:warmstart-y}
\left\|y_k^0 - y^*(x_k)\right\| \le & \left\|y_{k-1}^D - y^*(x_{k-1})\right\| + \left\|y^*(x_{k-1}) - y^*(x_k)\right\| \\ 
\le & \left(1-\frac{\mu}{\ell}\right)^{\frac{D}{2}} \left\|y_{k-1}^0 - y^*(x_{k-1})\right\| + \kappa \left\|x_{k-1} - x_k\right\| \\
\le & \left(1-\frac{\mu}{\ell}\right)^{\frac{D}{2}} \left\|y_{k-1}^0 - y^*(x_{k-1})\right\|  + \kappa \eta \left\| \widehat{\nabla} \Phi(x_{k-1})\right\|,
\ead\ee
where the second inequality is due to \eqref{eq:gdrate} and the fact that $y^*(x)$ is $\kappa$-Lipschitz continuous \cite{ghadimi2018approximation}[Lemma 2.2], the third inequality follows the update of $x_k$.
The next step is to bound $\left\|v_k^0 - v^*_k\right\|$. Before that, we prepare the following inequality:
\be\bad\label{eq:vkstar-diff}
&\left\|v_{k-1}^*- v_{k}^*\right\|\\ = & \left\| \nabla_y^2g(x_{k-1},y^*(x_{k-1}))^{-1} \nabla_y f(x_{k-1},y^*(x_{k-1})) - \nabla_y^2g(x_k,y^*(x_k))^{-1} \nabla_y f(x_k,y^*(x_k))\right\|\\
\le & M \left\| \nabla_y^2g(x_{k-1},y^*(x_{k-1}))^{-1} - \nabla_y^2g(x_k,y^*(x_k))^{-1} \right\| \\
&+  \frac{1}{\mu}\left\| \nabla_y f(x_{k-1},y^*(x_{k-1})) - \nabla_y f(x_k,y^*(x_k))  \right\|\\
 \le& \left( \kappa^2 + \kappa + \frac{\rho M (1+ \kappa)}{\mu^2}\right) \left\|x_{k}- x_{k-1}\right\|,
\ead\ee
 where in the first inequality we used Proposition \ref{pro:bounds} and the second inequality follows Assumption \ref{assu:bilevel}. We then have the bound of $\left\|v_k^0 - v^*_k\right\|$ as follows
\be\bad\label{eq:warmstart-v}
&\left\|v_k^0 - v^*_k\right\| = \left\|v_{k-1}^N - v^*_k\right\| \le  \left\|v_{k-1}^N - v_{k-1}^* \right\| + \left\|v_{k-1}^*- v_{k}^*\right\| \\ 
\le & \left(1 +  2\sqrt{\kappa}\right) \Big(\frac{\ell}{\mu} +\frac{\rho M}{\mu^2} \Big)\|y^D_{k-1}-y^*(x_{k-1})\| + 2\sqrt{\kappa}\Big( \frac{\sqrt{\kappa}-1}{\sqrt{\kappa}+1} \Big)^N\|v_{k-1}^0- v_{k-1}^*\|+ \left\|v_{k-1}^*- v_{k}^*\right\|\\
\le &\left(1 +  2\sqrt{\kappa}\right)  \Big(\frac{\ell}{\mu} +\frac{\rho M}{\mu^2} \Big)\left(1-\frac{\mu}{\ell}\right)^{\frac{D}{2}} \|y^0_{k-1}-y^*(x_{k-1})\| +2\sqrt{\kappa}\Big( \frac{\sqrt{\kappa}-1}{\sqrt{\kappa}+1} \Big)^N\|v_{k-1}^0- v_{k-1}^*\|\\ 
& + \left( \kappa^2 + \kappa + \frac{\rho M (1+ \kappa)}{\mu^2}\right) \left\|x_{k}- x_{k-1}\right\|\\
\le &\left(1 +  2\sqrt{\kappa}\right)  \Big(\frac{\ell}{\mu} +\frac{\rho M}{\mu^2} \Big)\left(1-\frac{\mu}{\ell}\right)^{\frac{D}{2}} \|y^0_{k-1}-y^*(x_{k-1})\| +2\sqrt{\kappa}\Big( \frac{\sqrt{\kappa}-1}{\sqrt{\kappa}+1} \Big)^N\|v_{k-1}^0- v_{k-1}^*\|\\ 
& + \eta\left( \kappa^2 + \kappa + \frac{\rho M (1+ \kappa)}{\mu^2}\right) \left\|\widehat{\nabla} \Phi(x_{k-1})\right\|,
\ead\ee
where the second inequality is by \eqref{eq:v-bound}, the third inequality is obtained by \eqref{eq:gdrate} and \eqref{eq:vkstar-diff}.
Summing \eqref{eq:warmstart-y} and \eqref{eq:warmstart-v}, we obtain 
\be\bad\label{eq:warmstart}
&\left\|y_k^0 - y^*(x_k)\right\| + \left\|v_k^0 - v^*_k\right\| \\ 
\le &\left(  
 \left(1 +  2\sqrt{\kappa}\right) \Big(\frac{\ell}{\mu} +\frac{\rho M}{\mu^2} \Big) + 1 \right) \cdot\left(1-\frac{\mu}{\ell}\right)^{\frac{D}{2}} \|y^0_{k-1}-y^*(x_{k-1})\|  \\ &+2\sqrt{\kappa}
   \Big( \frac{\sqrt{\kappa}-1}{\sqrt{\kappa}+1} \Big)^N\|v_{k-1}^0- v_{k-1}^*\|
 +  \left( \kappa^2 + 2\kappa + \frac{\rho M (1+ \kappa)}{\mu^2}\right) \eta \left\| \widehat{\nabla} \Phi(x_{k-1})\right\|.
\ead\ee
Set the parameters as 
 \be\bad
 D \ge & 2 \log \frac{1}{2\left(\left(1 +  2\sqrt{\kappa}\right) \Big(\frac{\ell}{\mu} +\frac{\rho M}{\mu^2} \Big) + 1 \right)} / \log(1 - \kappa^{-1})=  \gO(\kappa),\\ 
 N \ge & \log \frac{1}{4\sqrt{\kappa}} / \log\Big( \frac{\sqrt{\kappa}-1}{\sqrt{\kappa}+1} \Big) = \gO(\sqrt{\kappa}),
 \ead\ee
such that we can have 
\be\bad\label{eq:equvseq}
&\left\|y_k^0 - y^*(x_k)\right\| + \left\|v_k^0 - v^*_k\right\| \\ 
\le &\frac{1}{2} \left(\|y^0_{k-1}-y^*(x_{k-1})\| + \|v_{k-1}^0- v^*_{k-1}\|\right) +  \left( \kappa^2 + 2\kappa + \frac{\rho M (1+ \kappa)}{\mu^2}\right) \eta \left\| \widehat \nabla \Phi(x_{k-1})\right\| \\
\le & \left(\frac{1}{2}\right)^k \widehat{\Delta} +  \left( \kappa^2 + 2\kappa + \frac{\rho M (1+ \kappa)}{\mu^2}\right) \eta \sum_{j = 0}^{k - 1} \left(\frac{1}{2}\right)^{k-1-j} \left\|\widehat{\nabla} \Phi(x_j)\right\|\\
\le & \widehat{\Delta} + 2\eta \left( \kappa^2 + 2\kappa + \frac{\rho M (1+ \kappa)}{\mu^2}\right)\left(M + \frac{\ell M}{\mu}\right) = \Gamma_1,
\ead\ee
where the last inequality follows $\left\|\widehat{\nabla} \Phi(x)\right\| \le \left(M + \frac{\ell M}{\mu}\right)$ by Proposition \ref{pro:bounds}.  Combining  \eqref{eq:equvseq} with \eqref{eq:aid-error1}, we have 
\be\bad\label{eq:aid-error2}
& \left\|\widehat \nabla \Phi(x_k)- \nabla \Phi(x_k)\right\| \\
\le & \left(\left( 1 + \frac{\ell}{\mu}\left(1 +  2\sqrt{\kappa}\right)\right)\left( \ell +  \frac{\rho M}{\mu} \right)  \left(1-\frac{\mu}{\ell}\right)^{\frac{D}{2}} + 2\ell\sqrt{\kappa}\Big( \frac{\sqrt{\kappa}-1}{\sqrt{\kappa}+1} \Big)^N \right) \Gamma_1,
\ead\ee
where the inequality follows from the inequality $ab + cd \le (a + c)(b + d)$ for any positive $a,b,c,d$. This completes the proof.
\end{proof}}


\subsection{Proof of Lemma \ref{lem:descent}}
\begin{proof}
By Lemma \ref{grad-Lipschitz}, $\Phi(x)$ is $L_\phi$-smooth, which yields
\begin{align*}
\Phi(x_{k+1})  \le & \Phi(x_{k}) + \la \nabla \Phi(x_k), x_{k+1} - x_k \ra + \frac{L_\Phi}{2} \|x_{k+1} - x_{k}\|_2^2\\
\le & \Phi(x_{k}) + \la \widehat{\nabla} \Phi(x_k), x_{k+1} - x_k \ra + \la\nabla \Phi(x_k) - \widehat{\nabla} \Phi(x_k), x_{k+1} - x_k \ra \\
& + \frac{L_\Phi}{2} \|x_{k+1} - x_{k}\|_2^2\\
\le & \Phi(x_{k}) + \la \widehat{\nabla} \Phi(x_k), x_{k+1} - x_k \ra + \frac{1}{4\eta} \|x_{k+1} - x_k\|^2 +  \eta \|\nabla \Phi(x_k) - \widehat{\nabla} \Phi(x_k) \| ^2 \\
& + \frac{L_\Phi}{2} \|x_{k+1} - x_{k}\|_2^2\\
\le & \Phi(x_{k}) - \frac{\eta}{4}  \|\widehat{\nabla} \Phi(x_k)\|_2^2 + \eta\|\nabla \Phi(x_k) - \widehat{\nabla} \Phi(x_k) \| ^2,
\end{align*}
where the third inequality is obtained by Young's inequality and the last inequality uses $\eta = \frac{1}{L_\phi}$.
\end{proof}

\subsection{Proof of Lemma \ref{lem:escapesaddle}}
\begin{proof}
The proof of Lemma \ref{lem:escapesaddle} closely follows \cite{jin2021nonconvex}[Lemma 22]. We first define two sequences $\{x_k\}$, $\{x'_k\}$ that are generated by Algorithm \ref{alg:perturbedalg} with initial points $x_0$ and $x'_0$, respectively. That is, 
\[x_{k+1} = x_k - \eta\widehat{\nabla}\Phi(x_k), \qquad x'_{k+1} = x'_k - \eta\widehat{\nabla}\Phi(x'_k).\]
We require the two initial points to satisfy the following conditions: 
\begin{itemize}
\item {\bf Condition (i)}: $\max\{\norm{x_0 - \tilde{x}}, \norm{x'_0 - \tilde{x}}\} \le \eta r$; 
\item {\bf Condition (ii)}: $x_0 - x'_0 = \eta r_0 \e_1$, where $\e_1$ is the minimum eigenvector of $\nabla^2 \Phi(\tilde{x})$ with $\|\e_1\|=1$ and $r_0 > \omega :=2^{2-\logt} L_\phi \uspace$, and 
\be\label{uspace-def}
\uspace = \frac{1}{4\logt}\sqrt{\frac{\epsilon}{\rho_\phi}}.
\ee
\end{itemize} 
Note that the parameters are given in \eqref{eq:para_minimax}.
We show that for these two sequences, the following inequality must hold: 
\be \label{eq:phi-dec}
\min\{\Phi(x_\utime) - \Phi(x_0), \Phi(x'_\utime) - \Phi(x'_0)\} \le - \ufun,
\ee
where $\ufun$ is defined in \eqref{ufun-def}. 
We now prove \eqref{eq:phi-dec} by contradiction. Assume the contrary of \eqref{eq:phi-dec} holds, i.e.: 
\begin{align}\label{eq:contrary}
\min\{\Phi(x_\utime) - \Phi(x_0), \Phi(x'_\utime) - \Phi(x'_0)\}  > - \ufun.
\end{align}
First, by the update of $x_k$, we have for any $\tau \le k \le \utime$:
\begin{align}
\norm{x_\tau - x_0} & \le \sum_{t = 1}^k \norm{x_t - x_{t - 1}}
\le \left[k \sum_{t = 1}^k \norm{x_t - x_{t - 1}}^2\right]^{\frac{1}{2}} \nonumber \\
&= \left[\eta^2 k\sum_{t = 1}^k \norm{\widehat{\nabla} \Phi(x_{t - 1})}^2\right]^{\frac{1}{2}} \nonumber\\
& \le \left[\eta^2 \utime \sum_{t = 1}^\utime \norm{\widehat{\nabla} \Phi(x_{t - 1})}^2\right]^{\frac{1}{2}} \nonumber\\
&\le \sqrt{4\eta \utime \left(\Phi(x_0) - \Phi(x_\utime) + \eta\sum_{t = 1}^\utime \|\nabla \Phi(x_t) - \widehat{\nabla} \Phi(x_t) \| ^2\right)},\label{eq64}
\end{align}
where the last inequality is obtained by Lemma \ref{lem:descent}. 
We have for any $k \le \utime$:
\begin{align}\label{eq:localize}
\max\{\norm{x_k - \tilde{x}}, \norm{x'_k - \tilde{x}}\} 
\le& \max\{\norm{x_k - x_0}, \norm{x'_k - x'_0}\}  +\max\{\norm{x_0 - \tilde{x}}, \norm{x'_0 - \tilde{x}}\}\nonumber \\
\le& \sqrt{4\eta \utime \ufun + 4 \eta^2 \utime^2 \cdot \frac{17 }{6400\logt^4} \cdot \epsilon^2 }+ \eta r\nonumber\\
\le& \frac{9}{40\logt} \cdot\sqrt{ \frac{\epsilon}{\rho_{\phi}}} + \frac{1}{400\iota^3}\cdot \frac{\epsilon}{ L_\phi }\nonumber\\
\le& \frac{9}{40\logt} \cdot\sqrt{ \frac{\epsilon}{\rho_{\phi}}} + \frac{1}{400\iota}\cdot \sqrt{ \frac{\epsilon}{\rho_{\phi}}}\nonumber\\
\le& \uspace,
\end{align}
where the second inequality uses \eqref{eq64}, \eqref{eq:contrary}, \eqref{eq:error-required}, and {\bf Condition (i)}, the third inequality is due to \eqref{eq:para_minimax} and \eqref{ufun-def}, and the fourth inequality is due to \eqref{iota-condition} and \eqref{L-phi-big}. On the other hand, we can write the update equation for the difference $\dif{x}_k \defeq x_k - x'_k$ as:
\begin{align}
\dif{x}_{k+1} & = \dif{x}_k - \eta\left[ \widehat{\nabla} \Phi(x_k) -  \widehat{\nabla} \Phi(x'_k) \right]\nonumber\\
& = \dif{x}_k - \eta\left[ \nabla \Phi(x_k) -  \nabla \Phi(x'_k) \right] - \eta\left[ \widehat{\nabla} \Phi(x_k) -  \nabla \Phi(x_k)  +    \nabla \Phi(x'_k) - \widehat{\nabla} \Phi(x'_k) \right]\nonumber\\
& = (I - \eta\mathcal{H} ) \dif{x}_k  - \eta  \left(\Delta_{1,k} \dif{x}_k  + \Delta_{2,k} \right)\nonumber\\
& = \underbrace{(I - \eta\mathcal{H} )^{k+1} \dif{x}_0}_{p(k+1)} - \underbrace{ \eta \sum_{t = 0}^k (I - \eta\mathcal{H} )^{k-t} \left(\Delta_{1,t} \dif{x}_t  + \Delta_{2,t} \right)}_{q(k+1)},
\end{align}
where we denote 
\begin{align*}
&\mathcal{H} = \nabla^2 \Phi(\tilde{x}),\\
&\Delta_{1,k} = \int_0^1\left[ \nabla^2 \Phi(x'_k + \theta(x_k - x_k')) - \mathcal{H} \right] d\theta,\\
&\Delta_{2,k} = \left[ \widehat{\nabla} \Phi(x_k) -  \nabla \Phi(x_k)  +    \nabla \Phi(x'_k) - \widehat{\nabla} \Phi(x'_k) \right].
 \end{align*}
For the two parts $q(k), p(k)$, we show that $p(k)$ is the dominant term by proving 
\be\bad\label{eq:qprelation}
\|q(k)\| \le \|p(k)\| / 2, \quad \forall k \in [\utime].
\ead\ee
We now prove \eqref{eq:qprelation} by induction. First, \eqref{eq:qprelation} holds trivially when $k = 0$ becase $\|q(0)\| = 0$.
Denote $\lambda_{\min}(\nabla^2 \Phi(\tilde{x})) = -\gamma$, which implies $\gamma\geq\sqrt{\rho_\phi \epsilon}$. Assume \eqref{eq:qprelation} holds for any $t\leq k$. Since $\dif{x}_0=\eta r_0 \e_1$, we have for any $t \le k$:
\be\label{xhat_t}
\|\dif{x}_t\| \le \|p(t)\| + \|q(t)\| \le \frac{3}{2} \|p(t)\| = \frac{3}{2} \| (\textbf{I} - \eta \mathcal{H})^t \dif{x}_0 \| = \frac{3}{2} (1 + \eta\gamma)^t \eta r_0. 
\ee
Therefore, at step $k+1$ we have 
\be\bad \label{eq:qkinc1}
\|q(k+1)\|  = & \left\|\eta \sum_{t = 0}^k (I - \eta\mathcal{H} )^{k-t} \left(\Delta_{1,t}  \dif{x}_t  + \Delta_{2,t} \right) \right\|\\
\le & \left\|\eta \sum_{t = 0}^k (I - \eta\mathcal{H} )^{k-t} \Delta_{1,t}  \dif{x}_t \right \| + \left\|\eta \sum_{t = 0}^k (I - \eta\mathcal{H} )^{k-t}  \Delta_{2,t}  \right\|\\
\le & \eta \rho_{\phi} \uspace  \sum_{t = 0}^k \left\| (I - \eta\mathcal{H} )^{k-t}\right\| \left \|  \dif{x}_t \right\| + \eta \sum_{t = 0}^k \left\| (I - \eta\mathcal{H} )^{k-t} \right\| \left\|\Delta_{2,t}\right\|\\
\le & \frac{3}{2} \eta \rho_{\phi} \uspace  \sum_{t = 0}^k  (1 + \eta\gamma )^{k}\eta r_0 + 2 \eta  \sum_{t = 0}^k  (1 + \eta\gamma )^{k}  \|\nabla \Phi(x_k) - \widehat{\nabla} \Phi(x_k) \|\\
\le & \frac{3}{2} \eta \rho_{\phi} \uspace \utime  (1 + \eta\gamma )^{k}\eta r_0 + 2 \eta \utime (1 + \eta\gamma )^{k}  \|\nabla \Phi(x_k) - \widehat{\nabla} \Phi(x_k) \|\\
\le & 2 \eta \rho_{\Phi} \uspace \utime  (1 + \eta\gamma )^{k}\eta r_0\\
\le &   2 \eta \rho_{\Phi} \uspace \utime \|p(k+1)\|.
\ead\ee
Here the second inequality is by $\| \Delta_{1,k} \| \le \rho_\phi \max\{ \|x_k - \modify{x} \|, \|x_k' - \modify{x} \| \} \le \rho_\phi \uspace$ which uses \eqref{eq:localize}. The third inequality is due to \eqref{xhat_t} and the fact that $I - \eta\mathcal{H} \succeq 0$ which is because $\eta = 1/L_\phi$ and $\lambda_{\max}(\mathcal{H}) \leq L_\phi$. The fifth inequality applies \eqref{eq:error-required}, i.e., $\|\nabla \Phi(x_k) - \widehat{\nabla} \Phi(x_k) \| \le \frac{\epsilon}{16\logt^2 2^\logt} \le \frac{1}{4} \rho_\phi \uspace \eta r_0.$ By noting $2 \eta \rho_{\phi} \uspace \utime = 1/2$, we complete the proof of \eqref{eq:qprelation}. Finally, \eqref{eq:qprelation} implies
\begin{align*}
\max\{\norm{x_\utime - \tilde{x}}, \norm{x'_\utime - \tilde{x}}\} \ge& \frac{1}{2}\norm{\dif{x}_\utime} \ge \frac{1}{2}[\norm{p(\utime)}
-\norm{q(\utime)}] \ge \frac{1}{4}\norm{p(\utime)} \\
=& \frac{(1+\eta\gamma)^{\utime} \eta r_0}{4} \ge \frac{(1+\eta\sqrt{\rho_\phi \epsilon}    )^{\utime} \eta r_0}{4} 
{\ge} 2^{\logt-2} \eta r_0 > \uspace,
\end{align*}
where the second to last inequality uses the fact $(1+x)^{1/x} \ge 2$ for any $x \in (0, 1]$. This contradicts with \eqref{eq:localize}, which finishes the proof of \eqref{eq:phi-dec}. We then characterize the probability, which follows the ideas in \cite{jin2021nonconvex}. Recall $x_0 \sim \text{Uniform}(B_{\tilde{x}}(\eta r))$. We refer to $B_{\tilde{x}}(\eta r)$ the \emph{perturbation ball}, and define the \emph{stuck region} within the perturbation ball to be the set of points starting from which GD requires more than $\utime$ steps to escape:
\begin{equation*}
\cXs \defeq \{ x \in B_{\tilde{x}}(\eta r) ~|~ \{x_t\} \text{~is GD sequence with~} x_0 = x, \text{and~} \Phi(x_\utime) - \Phi(x_0) > -\ufun \}.
\end{equation*}
Although the shape of the stuck region can be very complicated, we know that the width of $\cXs$ along the $\e_1$ direction is at most $\eta \omega$. That is, $\text{Vol}(\cXs) \le \text{Vol}(\ball_0^{d-1}(\eta r)) \eta \omega$. Therefore, 
\begin{align*}
\Pr(\x_0 \in \cXs) = & \frac{\text{Vol}(\cXs)}{\text{Vol}(\ball^{d}_{\tilde{\x}}(\eta r))}
\le \frac{\eta \omega \times \text{Vol}(\ball^{d-1}_0(\eta r))}{\text{Vol} (\ball^{d}_0(\eta r))}\\
=& \frac{\omega}{r\sqrt{\pi}}\frac{\Gamma(\frac{d}{2}+1)}{\Gamma(\frac{d}{2}+\frac{1}{2})}
\le \frac{\omega}{r} \cdot \sqrt{\frac{d}{\pi}} \le  \frac{\ell  \sqrt{d}}{\sqrt{\rho\epsilon}} \cdot \logt^2 2^{8-\logt}. 
\end{align*}
On the event $\{x_0 \not \in \cXs\}$, due to our parameter choice in \eqref{eq:para_minimax}, \eqref{iota-condition}, \eqref{ufun-def} and \eqref{uspace-def}, we have:
\begin{equation*}
\Phi(x_\utime) - \Phi(\tilde{x}) = [\Phi(x_\utime) - \Phi(x_0)] + [\Phi(x_0)- \Phi(\tilde{x})]
\le -\ufun + \epsilon \eta r + \frac{L_\phi \eta^2 r^2}{2} \le -\ufun/2,
\end{equation*}
where the first inequality uses the $L_\phi$-smoothness of $\Phi(\cdot)$.
This finishes the proof.
\end{proof}

\subsection{Proof of Theorem \ref{thm:aid}}\label{sec:B.4}
\begin{proof}
For the AID method, we characterize the iteration complexity for $D$ and $N$ so that \eqref{eq:error-required} holds. By Lemma \ref{lem:grad-err-aid}, we require $D$ and $N$ to satisfy
\be\bad\label{eq:aid-bound}
&\Gamma_1\left( 1 + \frac{\ell}{\mu}\left(1 +  2\sqrt{\kappa}\right)\right)\left( \ell +  \frac{\rho M}{\mu} \right)  \left(1-\frac{\mu}{\ell}\right)^{\frac{D}{2}} + 2\ell\sqrt{\kappa}\Gamma_1\Big( \frac{\sqrt{\kappa}-1}{\sqrt{\kappa}+1} \Big)^N \\
 \le & \min \left\{ \frac{\sqrt{17} }{80\logt^2}, \frac{1}{16\logt^2 2^{\logt}} \right\}\cdot \epsilon.
\ead\ee
It is easy to verify that \eqref{eq:aid-bound} holds when 
\be\bad\label{eq:aid-ND}
D =& 2 \log \left(\frac{2 \Gamma_1\left( 1 + \frac{\ell}{\mu}\left(1 +  2\sqrt{\kappa}\right)\right)\left( \ell +  \frac{\rho M}{\mu} \right) }{ \min \left\{ \frac{\sqrt{17} }{80\logt^2}, \frac{1}{16\logt^2 2^{\logt}} \right\} \epsilon} \right) / \log\left(\frac{1}{1 - \kappa^{-1}}\right) = \gO\left(\kappa \log\left(\frac{1}{\epsilon}\right) \right),\\
N =& \log \left(\frac{4 \ell \sqrt{\kappa} \Gamma_1}{\min \left\{ \frac{\sqrt{17} }{80\logt^2}, \frac{1}{16\logt^2 2^{\logt}} \right\} \epsilon} \right) / \log\left(\frac{1 + \sqrt{\kappa}}{1 - \sqrt{\kappa}}\right) = \gO\left(\sqrt{\kappa} \log\left(\frac{1}{\epsilon}\right) \right).
\ead\ee
Moreover, it is easy to verify that the right hand side of \eqref{eq:aid-bound} is smaller than $\epsilon/5$. Therefore, by choosing $D$ and $N$ as in \eqref{eq:aid-ND}, we know that 
\be\label{AID-bound-nablaphi-nablahatphi}
\|\nabla\Phi(x_k) - \widehat{\nabla}\Phi(x_k)\| \leq \frac{\epsilon}{5}, \quad \forall k.
\ee
There are two possible cases to consider. 
\begin{itemize}
\item \textbf{Case 1:} $\left\| \widehat{\nabla} \Phi(x_k) \right\| > \frac{4}{5} \epsilon$ and $k - k_{perturb} > \utime$. In this case, combining \eqref{AID-bound-nablaphi-nablahatphi} and Lemma \ref{lem:descent} leads to 
\[
\Phi(x_{k+1})  \le \Phi(x_{k}) - \frac{4}{25L_\phi} \epsilon^2 + \frac{1}{25L_\phi}  \epsilon^2 =  \Phi(x_{k}) - \frac{3}{25L_\phi} \epsilon^2.
\]
Therefore, the total iteration number of \textbf{Case 1} can be bounded by 
\be\label{thm3.5-bound1}
\frac{25 L_\phi (\Phi(x_{0}) - \Phi^* )}{3\epsilon^2}.
\ee
\item \textbf{Case 2:} $k - k_{perturb} \le \utime$. This case means that we are within $\utime$ iterations of the last perturbation step, i.e., the step 10 in Algorithm \ref{alg:perturbedalg}. Suppose the last perturbation step happened at the $\bar{k}$-th iteration. Therefore, from the step 10 in Algorithm \ref{alg:perturbedalg} we know that $\|\widehat{\nabla}\Phi(x_{\bar{k}})\|\leq\frac{4}{5}\epsilon$. This together with \eqref{AID-bound-nablaphi-nablahatphi} implies $\|\nabla \Phi(x_{\bar{k}})\|\leq\epsilon$. Now there are two cases to further consider. \textbf{Case 2(i).} If $\lambda_{\min}(\nabla^2\Phi(x_{\bar{k}})) \leq -\sqrt{\rho_{\phi}\epsilon}$, then according to Lemma \ref{lem:escapesaddle} we know that with probability at least $1-\delta$ it holds that
\[\Phi(x_{\bar{k}+\utime}) - \Phi(x_{\bar{k}}) \leq -\ufun/2.\]
So the total iteration number in this case is bounded by
\be\label{thm3.5-bound2}
\frac{ (\Phi(x_{0}) - \Phi^* )\utime}{ \ufun / 2}.
\ee
\textbf{Case 2(ii).}  If $\lambda_{\min}(\nabla^2\Phi(x_{\bar{k}})) > -\sqrt{\rho_{\phi}\epsilon}$, then we have already found an $\epsilon$-local minimum of $\Phi(x)$. 
\end{itemize}
Therefore, combining \eqref{thm3.5-bound1} and \eqref{thm3.5-bound2} we know that the total iteration number before we visit an $\epsilon$-local minimum can be bounded by 
\[
K = \frac{ (\Phi(x_{0}) - \Phi^* )\utime}{ \ufun / 2} + \frac{25 L_\phi (\Phi(x_{0}) - \Phi^* )}{3\epsilon^2} = \tilde{\gO} \left( \kappa^3 \epsilon^{-2}\right).
\]
This completes the proof.
\end{proof}


\section{Proofs of Results in Section \ref{subsec:minimax}}\label{sec:C}

\subsection{Proof of Proposition \ref{pro:minimax}}

\begin{proof}
By Danskin's theorem, the gradient of $\Phi(x)$ is $\nabla \Phi(x) = \nabla_x f(x, y^*(x))$. Therefore the Hessian of $\Phi(x)$ is given by 
\be\bad\label{phi}
\nabla^2 \Phi(x) = \nabla^2_{xx}  f(x, y^*(x)) + \nabla^2_{xy}  f(x, y^*(x)) \cdot \frac{\partial y^*(x)}{\partial x}.  
\ead\ee
Note that the optimality condition for the max-player is $\nabla _y f(x, y^*(x)) = 0$, which leads to
\be \bad\label{y-opt-condition}
\nabla^2_{yx}  f(x, y^*(x)) + \nabla^2_{yy}  f(x, y^*(x)) \cdot \frac{\partial y^*(x)}{\partial x} = 0.  
\ead\ee
Combining \eqref{phi} and \eqref{y-opt-condition} yields
\be\bad\label{phi-2}
\nabla^2 \Phi(x) = \nabla^2_{xx}  f(x, y^*(x)) - \nabla^2_{xy}  f(x, y^*(x)) \nabla^2_{yy}  f(x, y^*(x))^{-1} \nabla^2_{yx}  f(x, y^*(x)).  
\ead\ee
Since $f(x,y)$ is $\mu$-strongly concave with respect to $y$, the second term on the right hand side of \eqref{phi-2}, i.e.,  $- \nabla^2_{xy}  f(x, y^*(x)) \nabla^2_{yy}  f(x, y^*(x))^{-1} \nabla^2_{yx}  f(x, y^*(x))$, is always positive definite. Therefore, we have the following conclusions. 
\begin{itemize}
\item A saddle point of $\Phi(x)$ satisfies $\lambda_{\min}(\nabla^2 \Phi(x)) < 0$, which together with \eqref{phi-2}, implies $\lambda_{\min}(\nabla^2_{xx}  f(x, y^*(x)) ) < 0.$ Therefore, it cannot be a strict local Nash equilibrium. 
\item A strict local Nash equilibrium of $f(x, y)$ satisfies $\lambda_{\min}(\nabla^2_{xx}  f(x, y^*(x)) ) > 0$, which yields $\lambda_{\min}(\nabla^2 \Phi(x)) > 0$. So it must be a local minimum of $\Phi(x)$. 
\end{itemize}
\end{proof}

\subsection{Proof of Proposition \ref{lem:equivalence}}
\begin{proof}
A local minimum of $\Phi(x)$ satisfies 
\be\label{proof-prop-2.4-eq1}
\nabla \Phi(x) = \zero, \quad \nabla^2 \Phi(x) \succ \zero.
\ee 
According to \eqref{phi-2}, the inequality in \eqref{proof-prop-2.4-eq1} is equivalent to
\be
\nabla^2_{xx}  f(x, y^*(x)) - \nabla^2_{xy}  f(x, y^*(x)) \nabla^2_{yy}  f(x, y^*(x))^{-1} \nabla^2_{yx}  f(x, y^*(x)) \succ \zero.
\ee
Moreover, for nonconvex-strongly-concave problems, it holds that $\nabla_{yy}^2f(x, y) \prec 0.$ Therefore, we only need to show that $\nabla \Phi(x) = \zero$ is equivalent to $\nabla f(x, y) = \zero$. Notice that for a pair $(x,y)$ satisfying $\nabla_x f(x, y) = \zero, \nabla_y f(x, y) = \zero$, we have $y = y^*(x)$ from the strongly convexity and $\nabla_x f(x, y^*(x)) = \nabla \Phi(x) = \zero$. Further more, when $\nabla \Phi(x) = \zero$, we can always choose $y = y^*(x)$ so that $\nabla_x f(x, y) = \zero, \nabla_y f(x, y) = \zero$. Therefore, these two conditions are equivalent to each other. 
When function $\Phi(x)$ has a strict local minimum, the local minimax point is guaranteed to exist.
\end{proof}

\subsection{Proof of Theorem \ref{thm:minimax}}

To prove Theorem \ref{thm:minimax}, we need the following lemmas. The first lemma shows that under Assumption \ref{assu:minimax}, the function $\Phi(x)$ is smooth and Hessian-Lipschitz continuous.
\begin{lemma}\label{lemC.1}\cite{chen2021escaping}[Proposition 1]
Suppose $f(x, y)$ satisfies Assumption \ref{assu:minimax}, we have
\begin{itemize}
\item $\Phi(x)$ is $L_\phi$-smooth, where $L_\phi = \ell(1+\kappa)$.
\item $\Phi(x)$ is $\rho_\phi$-Hessian Lipschitz continuous, i.e., \eqref{def-rho-phi-intro} holds, where $\rho_\phi = \rho(1+\kappa)^3$.
\end{itemize}
\end{lemma}

The second lemma gives an upper bound for the gradient estimation error with the warm start strategy.
\begin{lemma}\label{lem:grad-err-mini}
Suppose Assumption \ref{assu:minimax} holds. For the GDmax algorithm (i.e., the \textbf{GDmax} option in Algorithm \ref{alg:perturbedalg}) with parameters $D = \gO(\kappa)$, we have,
\be\bad\label{eq:gdmax-error}
\left\|\widehat{\nabla}\Phi(x_k) - \nabla \Phi(x_k)\right\| \le & \ell\left(\widehat{\Delta} + 2\eta\kappa \left(M + \frac{\ell M}{\mu}\right)\right) (1-\kappa^{-1})^{\frac{D}{2}} ,
\ead\ee
where $\widehat{\Delta} = \|y^0 - y^*(x_{0})\|.$
\end{lemma}
\begin{proof}
The gradient estimation error for minimax problem can be bounded by
\be\bad\label{eq:gdmax-error-bound-1}
\left\|\widehat{\nabla}\Phi(x_k) - \nabla \Phi(x_k)\right\|  = &\left\|\nabla_x f(x_k, y_k^D) -\nabla_x f(x_k, y^*(x_k))\right\| \le  \ell \|y_k^D - y^*(x_k)\|\\
 \le & \ell(1-\kappa^{-1})^{\frac{D}{2}} \|y_k^0 - y^*(x_k)\|,
\ead\ee
where the last inequality follows \eqref{eq:gdrate}. By the warm start strategy $y_k^0 = y_{k-1}^D$, we have
\be\bad\label{eq:gdmax-error-bound}
\|y_k^0 - y^*(x_k)\| \le &  \|y_{k-1}^D - y^*(x_{k-1})\| + \|y^*(x_{k-1}) - y^*(x_k)\|\\
\le & (1-\kappa^{-1})^{\frac{D}{2}} \|y_{k-1}^0 - y^*(x_{k-1})\| + \kappa\|x_k - x_{k-1}\|\\
\le & (1-\kappa^{-1})^{\frac{D}{2}} \|y_{k-1}^0 - y^*(x_{k-1})\| + \eta\kappa\|\widehat{\nabla} \Phi(x_{k-1})\|.
\ead\ee
By setting 
\be\bad\label{eq:gdmax-error-D}
D > 2 \log 2 / \log\left(\frac{1}{1-\kappa^{-1}}\right) = \gO(\kappa),
\ead\ee
we have 
\be\bad\label{eq:gdmax-error-bound2}
\|y_k^0 - y^*(x_k)\| \le &  \|y_{k-1}^D - y^*(x_{k-1})\| + \|y^*(x_{k-1}) - y^*(x_k)\|\\
\le & \frac{1}{2}\|y_{k-1}^0 - y^*(x_{k-1})\| + \eta\kappa\|\widehat{\nabla} \Phi(x_{k-1})\|\\
\le & \left(\frac{1}{2}\right)^k \|y^0 - y^*(x_{0})\|  + \eta\kappa \sum_{j=0}^{k-1} \left(\frac{1}{2}\right)^{k - 1 - j}\|\widehat{\nabla} \Phi(x_{k-1})\|\\
\le & \widehat{\Delta} + 2\eta\kappa \left(M + \frac{\ell M}{\mu}\right),
\ead\ee
where the last inequality uses $\|\widehat{\nabla} \Phi(x_{k-1})\| \le  \left(M + \frac{\ell M}{\mu}\right)$ for any $k.$ Combining \eqref{eq:gdmax-error-bound2} and \eqref{eq:gdmax-error-bound-1} yields
\be\bad\label{eq:gdmax-error-bound-3}
\left\|\widehat{\nabla}\Phi(x_k) - \nabla \Phi(x_k)\right\|  \le  \ell\left(\widehat{\Delta} + 2\eta\kappa \left(M + \frac{\ell M}{\mu}\right)\right) (1-\kappa^{-1})^{\frac{D}{2}} ,
\ead\ee
which completes the proof.
\end{proof}

We now give the proof of Theorem \ref{thm:minimax}.

\begin{proof}
By Lemma \ref{lem:grad-err-mini}, we require $D$ to satisfy 
\be\bad\label{eq:gdmax-bound}
 \ell\left(\widehat{\Delta} + 2\eta\kappa \left(M + \frac{\ell M}{\mu}\right)\right) (1-\kappa^{-1})^{\frac{D}{2}} \le  \min \left\{ \frac{\sqrt{17} }{80\logt^2}, \frac{1}{16\logt^2 2^{\logt}} \right\}\cdot \epsilon.
\ead\ee
It is easy to verify that 
\be\bad\label{eq:gdmax-D}
D =&  2 \log \left(\frac{\ell\left(\widehat{\Delta} + 2\eta\kappa \left(M + \frac{\ell M}{\mu}\right)\right) }{\min \left\{ \frac{\sqrt{17} }{80\logt^2}, \frac{1}{16\logt^2 2^{\logt}} \right\} \epsilon} \right) / \log\left(\frac{1}{1 - \kappa^{-1}}\right) = \gO\left(\kappa \log\left(\frac{1}{\epsilon}\right)\right),
\ead\ee
satisfies \eqref{eq:gdmax-bound}. The rest of the proof is the same as the proof of Theorem \ref{thm:aid} in Section \ref{sec:B.4}.
\end{proof}

\section{Proofs of Results in Section \ref{sec:sto-bilevel}}

\subsection{Proof of Lemma \ref{lem:neon}}
\begin{proof}
Define the following function:
\be\bad\label{eq:hatphi}
\hat{\Phi}_x (u) = \Phi (x + u) - \Phi (x) - \nabla \Phi(x) ^\top u.
\ead\ee
We first characterize the required estimation error, which is used in the later proof. Specifically, we choose the inner iteration number $D$ (step 9 of Algorithm \ref{alg:iNEON}) and $N$ (used in steps 6 and 12 of Algorithm \ref{alg:iNEON})  such that for any $k \le \utime$, the following inequalities hold:
\be\bad\label{eq:est-error-neon}
&\| \widehat{\nabla} \Phi(\tilde{x} + u_k) -  \nabla \Phi(\tilde{x} + u_k)  \| \le \min\left\{ \frac{\sqrt{17} }{40\logt^2},  \frac{1}{16\logt^2 2^{\logt/4}}, \frac{9^{1/3} - 2}{8\logt} , \frac{1}{750\logt^2}\right\} \epsilon,\\
& M\left\|  y_k^D(\tilde{x} + u_{k}) - y^*(\tilde{x} + u_{k})) \right\| \le \frac{1}{750\logt^3 L_\phi} \epsilon^2.
\ead\ee 
Note that both inequalities in \eqref{eq:est-error-neon} also imply the following inequality, which corresponds to the case $u_k = 0$:
\be\bad\label{eq:est-error-neon-2}
&\| \widehat{\nabla} \Phi(\tilde{x}) -  \nabla \Phi(\tilde{x})  \| \le \min\left\{ \frac{\sqrt{17} }{40\logt^2},  \frac{1}{16\logt^2 2^{\logt/4}}, \frac{9^{1/3} - 2}{8\logt} , \frac{1}{750\logt^2}\right\} \epsilon,\\
&M\left\|  y_k^D(\tilde{x} ) - y^*(\tilde{x} )) \right\| \le \frac{1}{750\logt^3 L_\phi} \epsilon^2.
\ead\ee
Since the lower level problem is strongly convex, combining with Lemma \ref{lem:grad-err-aid}, we can set $D$ and $N$ in Algorithm \ref{alg:iNEON} as
\be\bad\label{eq:aid-ND-neon}
D =& \max\left\{2 \log \left(\frac{2 \left( 1 + \frac{\ell}{\mu}\left(1 +  2\sqrt{\kappa}\right)\right)\left( \ell +  \frac{\rho M}{\mu} \right)\Gamma_1}{ \min\left\{ \frac{\sqrt{17} }{40\logt^2},  \frac{1}{16\logt^2 2^{\logt/4}}, \frac{9^{1/3} - 2}{8\logt} , \frac{1}{750\logt^2}\right\} \epsilon} \right) , 2\log\left( \frac{750ML_\phi\logt^3\Gamma_1}{\epsilon^2} \right) \right\}/ \log\left(\frac{1}{1 - \kappa^{-1}}\right)\\
 = &\gO\left(\kappa \log\left(\frac{1}{\epsilon}\right) \right),\\
N =& \log \left(\frac{4 \ell \sqrt{\kappa} \Gamma_1}{\min\left\{ \frac{\sqrt{17} }{40\logt^2},  \frac{1}{16\logt^2 2^{\logt/4}}, \frac{9^{1/3} - 2}{8\logt} , \frac{1}{750\logt^2}\right\} \epsilon} \right) / \log\left(\frac{1 + \sqrt{\kappa}}{1 - \sqrt{\kappa}}\right) = \gO\left(\sqrt{\kappa} \log\left(\frac{1}{\epsilon}\right) \right),
\ead\ee
such that \eqref{eq:est-error-neon} holds in each iteration. Next, we show that with probability at least $1-\frac{L_\phi \sqrt{d}}{\sqrt{\rho_{\phi}\epsilon}}\cdot \logt^2 2^{8-\logt/4}$, $\hat{\Phi}_{\tilde{x}}(u_\utime) - \hat{\Phi}_{\tilde{x}}(u_0) \le -\ufun$ holds,
where $\hat{\Phi}_x(u)$ is defined in \eqref{eq:hatphi}. Note that it is easy to verify that $\hat{\Phi}_x(u)$ is $L_\phi$-smooth, and it yields 
\begin{align}\label{eq:phi-smooth-neon}
    \bad
        \hat{\Phi}_{\tilde{x}}(u_{k+1}) \le & \hat{\Phi}_{\tilde{x}}(u_{k}) + \la \nabla \hat{\Phi}_{\tilde{x}}(u_{k}) , u_{k+1} - u_{k} \ra + \frac{L_\phi}{2} \|u_{k+1} - u_{k}\|_2^2\\
= & \hat{\Phi}_{\tilde{x}}(u_{k}) + \la \nabla\Phi (\tilde{x} + u_{k}) - \nabla\Phi (\tilde{x}),   u_{k+1} - u_{k}\ra + \frac{L_\phi}{2} \|u_{k+1} - u_{k}\|_2^2\\
= & \hat{\Phi}_{\tilde{x}}(u_{k}) + \la\widehat{\nabla}\Phi (\tilde{x} + u_{k}) - \widehat{\nabla}\Phi (\tilde{x}), u_{k+1} - u_{k}\ra + \la\nabla\Phi (\tilde{x} + u_{k}) - \widehat{\nabla}\Phi (\tilde{x} + u_{k}), u_{k+1} - u_{k}\ra\\
& + \la \widehat{\nabla}\Phi (\tilde{x}) - \nabla\Phi (\tilde{x}), u_{k+1} - u_{k}\ra + \frac{L_\phi}{2} \|u_{k+1} - u_{k}\|_2^2\\
\le &\hat{\Phi}_{\tilde{x}}(u_{k}) - \frac{1}{\eta} \| u_{k+1} - u_{k}\|^2 + \frac{1}{8\eta}\| u_{k+1} - u_{k}\|^2 +  2\eta \| \nabla\Phi (\tilde{x} + u_{k}) - \widehat{\nabla}\Phi (\tilde{x} + u_{k}) \|^2 \\
& + \frac{1}{8\eta}\| u_{k+1} - u_{k}\|^2 +  2\eta\|\widehat{\nabla}\Phi (\tilde{x}) - \nabla\Phi (\tilde{x})\|^2 + \frac{1}{2\eta} \|u_{k+1} - u_{k}\|_2^2\\
= & \hat{\Phi}_{\tilde{x}}(u_{k}) - \frac{1}{4\eta} \| u_{k+1} - u_{k}\|^2  +  2\eta \| \nabla\Phi (\tilde{x} + u_{k}) - \widehat{\nabla}\Phi (\tilde{x} + u_{k}) \|^2\\
& +  2\eta\|\widehat{\nabla}\Phi (\tilde{x}) - \nabla\Phi (\tilde{x})\|^2,
    \ead
\end{align}
where the first equality is from \eqref{eq:hatphi}, the second inequality is from \eqref{eq:u-update-neon} and Young's inequality. 
We then follow the same ideas as in the proof of Lemma \ref{lem:escapesaddle}. We design two coupling sequences $\{u_t\}$, $\{w_t\}$ generated by iNEON (Algorithm \ref{alg:iNEON}) with initial points $u_0$ and $w_0$, respectively. We require the two sequences to satisfy: \textbf{Condition (i).} $\max\{\norm{u_0}, \norm{w_0}\} \le \eta r$; and \textbf{Condition (ii).} $u_0 - w_0 = \eta r_0 \e_1$, where $\e_1$ is the minimum eigenvector of $\nabla^2 \Phi(\tilde{x})$ with $\|\e_1\|=1$ and $r_0 > \omega \defeq 2^{2-\logt/4} L_\phi \uspace $. The rest is to prove
\be \label{eq:phi-dec-noen2}
\min\{\hat{\Phi}_x(u_\utime) - \hat{\Phi}_x(u_0) , \hat{\Phi}_x(w_\utime) - \hat{\Phi}_x(w_0)\} \le - \ufun.
\ee
We prove \eqref{eq:phi-dec-noen2} by contradiction. Assume the contrary holds:
\begin{align}\label{eq:noen-contrary}
 \min\{\hat{\Phi}_{\tilde{x}}(u_\utime) - \hat{\Phi}_{\tilde{x}}(u_0) , \hat{\Phi}_{\tilde{x}}(w_\utime) - \hat{\Phi}_{\tilde{x}}(w_0)\}  > - \ufun.
\end{align}
First, by the update of $u_k$ (i.e., \eqref{eq:u-update-neon} or step 13 of Algorithm \ref{alg:iNEON}), we have for any $\tau \le k$:
\begin{align}
&\norm{u_\tau - u_0} \nonumber\\& \le \sum_{t = 1}^k \norm{u_t - u_{t - 1}}
{\le} \left[k \sum_{t = 1}^k \norm{u_t - u_{t - 1}}^2\right]^{\frac{1}{2}} \le \left[\utime \sum_{t = 1}^\utime \norm{u_t - u_{t - 1}}^2\right]^{\frac{1}{2}}\nonumber\\
&\le \sqrt{4\eta \utime \left(\hat{\Phi}_{\tilde{x}}(u_0) - \hat{\Phi}_{\tilde{x}}(u_\utime) + 2\eta \left(\sum_{t = 1}^\utime \| \nabla\Phi (\tilde{x} + u_{t}) - \widehat{\nabla}\Phi (\tilde{x} + u_{t}) \|^2 + \|\widehat{\nabla}\Phi (\tilde{x}) - \nabla \Phi (\tilde{x})\|^2\right)\right)} \nonumber\\
& \le \sqrt{4\eta \utime \ufun + 8 \eta^2 \utime^2 \cdot \frac{17 }{800\logt^4} \cdot \epsilon^2 }+ \eta r \le \uspace,\label{eq92}
\end{align}
where the fourth inequality is obtained by \eqref{eq:phi-smooth-neon}, the fifth inequality follows from \eqref{eq:noen-contrary}, \eqref{eq:est-error-neon} and \eqref{eq:est-error-neon-2}, and the last inequality is due to the parameter choice in \eqref{eq:para_neon}, and $\logt \ge 1$, $L_\phi/\sqrt{\rho_{\phi}\epsilon} \ge 1$. 
Therefore, from \eqref{eq92} we have for any $k \le \utime$:
\begin{align}\label{eq:localize-neon}
\max\{\norm{u_k }, \norm{w_k}\} 
\le \max\{\norm{u_k - u_0}, \norm{w_k - w_0}\}  +\max\{\norm{u_0}, \norm{w_0 }\} \le \uspace. 
\end{align}
On the other hand, we can write the update equation for the difference $v_k \defeq u_k - w_k$ as: 
\begin{align}
v_{k+1}  = & v_k - \eta\left[ \widehat{\nabla} \Phi(\tilde{x} + u_k) -  \widehat{\nabla} \Phi(\tilde{x} +w_k) \right]\nonumber\\
 =& v_k - \eta\left[ \nabla \Phi(\tilde{x} +u_k) -  \nabla \Phi(\tilde{x} +w_k) \right] \nonumber\\
& - \eta\left[ \widehat{\nabla} \Phi(\tilde{x} +u_k) -  \nabla \Phi(\tilde{x} +u_k)  +    \nabla \Phi(\tilde{x} +w_k) - \widehat{\nabla} \Phi(\tilde{x} +w_k) \right]\nonumber\\
= & \underbrace{(I - \eta\mathcal{H} )^{k+1} v_0}_{p(k+1)} - \underbrace{ \eta \sum_{t = 0}^k (I - \eta\mathcal{H} )^{k-t} \left(\Delta_{1,t} v_t  + \Delta_{2, t} \right)}_{q(k+1)},
\end{align}
where we denote 
\begin{align}
&\mathcal{H} = \nabla^2 \Phi(\modify{x}),\\
&\Delta_{1, k} = \int_0^1\left[ \nabla^2 \Phi(\tilde{x} + w_k + \theta(u_k - w_k)) - \mathcal{H} \right] d\theta,\\
&\Delta_{2, k} = \left[ \widehat{\nabla} \Phi(\tilde{x} + u_k) -  \nabla \Phi(\tilde{x} + u_k)  +    \nabla \Phi(\tilde{x} + w_k) - \widehat{\nabla} \Phi(\tilde{x} + w_k) \right].
\end{align}
We then prove $\|q(k)\| \le \|p(k)\| / 2, \forall k \in [\utime]$. We prove it by induction. It is easy to check that it holds at $k = 0.$ Assume it holds for any $t\leq k.$ Denote $\lambda_{\min}(\nabla^2 \Phi(\tilde{x})) = -\gamma$. Since $v_0$ lies in the direction of the minimum eigenvector of $\nabla^2 \Phi(\tilde{x}_0)$, we have for any $t \le k$:
\be\bad
\|v_t\| \le \|p(t)\| + \|q(t)\|  \leq \frac{3}{2} \|p(t)\| \le \frac{3}{2} (1 + \eta\gamma)^t \eta r_0. 
\ead\ee
At step $k+1$, similar to \eqref{eq:qkinc1}, we have
\be\bad\label{eq:qkinc-neon}
\|q(k+1)\|  \le   2 \eta \rho_{\Phi} \uspace \utime \|p(k+1)\|,
\ead\ee
where we used \eqref{eq:est-error-neon}. 
Combining \eqref{eq:qkinc-neon} with the choice of parameters in \eqref{eq:para_neon}
finishes the proof of $\|q(k)\| \le \|p(k)\| / 2$. Therefore, we have
\begin{align*}
\max\{\norm{u_\utime }, \norm{w_\utime}\} \ge& \frac{1}{2}\norm{v(\utime)} \ge \frac{1}{2}[\norm{p(\utime)}
-\norm{q(\utime)}] \ge \frac{1}{4}[\norm{p(\utime)} \\
=& \frac{(1+\eta\gamma)^{\utime} \eta r_0}{4}
{\ge} 2^{\logt / 4-2} \eta r_0 > \uspace,
\end{align*}
where the second to last inequality uses the fact $(1+x)^{1/x} \ge 2$ for any $x \in (0, 1]$. This contradicts with \eqref{eq:localize-neon}, which finishes the proof of \eqref{eq:phi-dec-noen2}. To characterize the probability, we define the \emph{stuck region}:
\begin{equation*}
\cXs \defeq \{ u \in \ball_{0}(\eta r) ~|~ \{u_t\} \text{~is the iNEON sequence with~} u_0 = u, \text{and~} \hat{\Phi}_{\tilde{x}}(u_\utime) - \hat{\Phi}_{\tilde{x}}(u_0)  > -\ufun \}.
\end{equation*}
Although the shape of the stuck region can be very complicated, we know that the width of $\cXs$ along the $\e_1$ direction is at most $\eta \omega$. That is, $\text{Vol}(\cXs) \le \text{Vol}(\ball_0^{d-1}(\eta r)) \eta \omega$. Therefore:
\begin{align*}
\Pr(u_0 \in \cXs) \le \frac{\eta \omega \times \text{Vol}(\ball^{d-1}_0(\eta r))}{\text{Vol} (\ball^{d}_0(\eta r))}
= \frac{\omega}{r\sqrt{\pi}}\frac{\Gamma(\frac{d}{2}+1)}{\Gamma(\frac{d}{2}+\frac{1}{2})}
\le \frac{\omega}{r} \cdot \sqrt{\frac{d}{\pi}} \le  \frac{\ell  \sqrt{d}}{\sqrt{\rho\epsilon}} \cdot \logt^2 2^{8-\logt/4}. 
\end{align*}
On the event $\{u_0 \not \in \cXs\}$, due to the choice of the parameters in \eqref{eq:para_neon}, we have with probability at least $1- \delta$, where $\delta > \frac{\ell  \sqrt{d}}{\sqrt{\rho\epsilon}} \cdot \logt^2 2^{8-\logt/4}$, that 
$$\hat{\Phi}_{\tilde{x}}(u_\utime) - \hat{\Phi}_{\tilde{x}}(u_0) < -\ufun.$$
Therefore, there exists some $k' \le \utime$ such that $\hat{\Phi}_{\tilde{x}}(u_{k'}) - \hat{\Phi}_{\tilde{x}}(u_0) < -\ufun$ and $\|u_\tau\| \le \uspace, \forall \tau < k'$. In other words, $k'$ is the first iteration that satisfies 
\be\label{k'-satisfy}
\hat{\Phi}_{\tilde{x}}(u_{k'}) - \hat{\Phi}_{\tilde{x}}(u_0) < -\ufun.
\ee 
By the update $u_{k'} = u_{k'-1} - \eta (\widehat{\nabla} \Phi(\tilde{x} + u_{k'-1}) - \widehat{\nabla} \Phi(\tilde{x}) ),$ we can bound the norm of $u_{k'}$:
\be\bad\label{bound-uk'}
\|u_{k'}\| \le& \| u_{k'-1}\| + \eta \| \nabla \Phi(\tilde{x} + u_{k'-1}) - \nabla \Phi(\tilde{x}) \| + \eta \| \widehat{\nabla} \Phi(\tilde{x} + u_{k'-1}) - \nabla \Phi(\tilde{x}+ u_{k'-1}) \|\\
& + \eta \| \widehat{\nabla} \Phi(\tilde{x} ) - \nabla \Phi(\tilde{x}) \|\\
\le& \| u_{k'-1}\| + \eta L_\phi \|  u_{k'-1} \| + 2 \eta \| \widehat{\nabla} \Phi(\tilde{x} + u_{k'-1}) - \nabla \Phi(\tilde{x}+ u_{k'-1}) \|\\
= & 2 \uspace + 2 \eta \| \widehat{\nabla} \Phi(\tilde{x} + u_{k'-1}) - \nabla \Phi(\tilde{x}+ u_{k'-1}) \|\\
\le&  9^{1/3}\uspace,
\ead\ee
where the second inequality is by the smoothness of $\Phi(x)$ given in Lemma \ref{lemC.1}, the equality is due to the parameter choice in \eqref{eq:para_neon}, and the last inequality is obtained by \eqref{eq:est-error-neon}. 
Moreover, by \eqref{k'-satisfy} and the smoothness of $\Phi(x)$, we have 
\be\bad\label{phixf-ineq}
 \Phi (\tilde{x} + u_{k'}) - \Phi (\tilde{x}) - \nabla \Phi(\tilde{x}) ^\top u_{k'} \le & \Phi (\tilde{x} + u_0) - \Phi (\tilde{x}) - \nabla \Phi(\tilde{x}) ^\top u_0 - \ufun \le  \frac{L_\phi}{2} \|u_0\|^2 - \ufun \\
 \le & \frac{L_\phi}{2} \cdot \frac{\epsilon^2}{160000\logt^6 L_\phi^2} -  \frac{1}{25\logt^3} \sqrt{\frac{\epsilon^3}{\rho_\phi}} \le \left(\frac{1}{320000} -  \frac{1}{25} \right)\frac{1}{\logt^3} \sqrt{\frac{\epsilon^3}{\rho_\phi}}, \\
 \ead\ee
where the third inequality is due to the parameter choice in \eqref{eq:para_neon}, and the last inequality applies $\logt>1$ and $L_\phi / \sqrt{\rho_\phi \epsilon} > 1.$ 
From \eqref{eq:est-error-neon} we can get
\be\bad\label{eq:u-bound}
&\left\| \hat{\Phi} (\tilde{x} + u_{k'}) - \hat{\Phi} (\tilde{x}) - \widehat{\nabla} \Phi(\tilde{x}) ^\top u_{k'}  - \left(  \Phi (\tilde{x} + u_{k'}) - \Phi (\tilde{x}) - \nabla \Phi(\tilde{x}) ^\top u_{k'}\right) \right\|\\
\le &\left\| \hat{\Phi} (\tilde{x} + u_{k'}) - \Phi (\tilde{x} + u_{k'}) \right\| +  \left\|\hat{\Phi} (\tilde{x}) - \Phi (\tilde{x})  \right\| +  \left\| \widehat{\nabla} \Phi(\tilde{x})  - \nabla \Phi(x) \right\| \left\| u_{k'} \right\|\\
\le &\left\| f(\tilde{x} + u_{k'}, y_k^D(\tilde{x} + u_{k'})) - f(\tilde{x} + u_{k'}, y^*(\tilde{x} + u_{k'})) \right\| +  \left\| f(\tilde{x}, y_k^D(\tilde{x})) - f(\tilde{x} , y^*(\tilde{x} ))  \right\|\\
& +  \left\| \widehat{\nabla} \Phi(\tilde{x})  - \nabla \Phi(x) \right\| \left\| u_{k'} \right\|\\
\le & M\left\|  y_k^D(\tilde{x} + u_{k'}) - y^*(\tilde{x} + u_{k'})) \right\| + M \left\|y_k^D(\tilde{x}) - y^*(\tilde{x} )  \right\| +  3\uspace \left\| \widehat{\nabla} \Phi(\tilde{x})  - \nabla \Phi(x) \right\| \\
\le & \frac{2}{750\logt^3 L_\phi} \epsilon^2 + \frac{3}{4\logt} \sqrt{\frac{\epsilon}{\rho_\phi}} \frac{1}{750\logt^2} \epsilon\\
\le & \frac{2}{750\logt^3 L_\phi} \epsilon^2 \cdot \frac{L_\phi}{\sqrt{\rho_\phi\epsilon}} + \frac{1}{750\logt^3} \sqrt{\frac{\epsilon^3}{\rho_\phi}} \\
= & \frac{1}{250\logt^3} \sqrt{\frac{\epsilon^3}{\rho_\phi}},
\ead\ee
where the third inequality is due to Assumption \ref{assu:bilevel} and \eqref{bound-uk'}, the fourth inequality uses \eqref{eq:est-error-neon}, \eqref{eq:est-error-neon-2} and the parameter choice in \eqref{eq:para_neon}, and the fifth inequality is due to $L_\phi > \sqrt{\rho_\phi\epsilon}$. 
Combining \eqref{phixf-ineq} and \eqref{eq:u-bound} we get
\be\bad
\hat{\Phi} (\tilde{x} + u_{k'}) - \hat{\Phi} (\tilde{x}) - \widehat{\nabla} \Phi(\tilde{x}) ^\top u_{k'} \le  \left(\frac{1}{320000} -  \frac{1}{25} + \frac{1}{250}\right) \frac{1}{\logt^3} \sqrt{\frac{\epsilon^3}{\rho_\phi}} = - \frac{11519}{12800} \ufun,
\ead\ee
i.e., the stopping criterion in step 14 of Algorithm \ref{alg:iNEON} is satisfied. Therefore, this shows that with high probability, the Algorithm \ref{alg:iNEON} terminates with the stopping criterion in step 14 being satisfied. When this happens, we have
\begin{align}
&\Phi (x + u_{k'}) - \Phi (x) - \nabla \Phi(x) ^\top u_{k'} \nonumber\\
\le &\hat{\Phi} (\tilde{x} + u_{k'}) - \hat{\Phi} (\tilde{x}) - \hat{\nabla} \Phi(\tilde{x}) ^\top u_{k'}  \nonumber\\&+ 
\left\| \hat{\Phi} (\tilde{x} + u_{k'}) - \hat{\Phi} (\tilde{x}) - \hat{\nabla} \Phi(\tilde{x}) ^\top u_{k'}  - \left(  \Phi (x + u_{k'}) - \Phi (x) - \nabla \Phi(x) ^\top u_{k'}\right) \right\|\nonumber\\
 \le &  \left(\frac{1}{320000} -  \frac{1}{25} + \frac{1}{250} + \frac{1}{250} \right) \frac{1}{\logt^3} \sqrt{\frac{\epsilon^3}{\rho_\phi}},\label{eq105}
\end{align}
where we used \eqref{eq:u-bound}. 
By the Hessian Lipschitz continuity of $\Phi(x)$, we have
\be\bad
\frac{1}{2}  u_{k'} \nabla^2 \Phi (x) u_{k'} \le& \Phi (x + u_{k'}) -  \Phi (x) - \nabla \Phi(x) ^\top u_{k'} + \frac{\rho_\phi}{6} \|u_{k'}\|^3  \\
\le & \left(\frac{1}{320000} -  \frac{1}{25}+ \frac{1}{250} + \frac{1}{250}  \right)\frac{1}{\logt^3} \sqrt{\frac{\epsilon^3}{\rho_\phi}}  + \frac{\rho_\phi}{6} \cdot 9 \uspace^3\\
\le & \left(\frac{1}{320000} -  \frac{1}{25}+ \frac{1}{250} + \frac{1}{250}  \right)\frac{1}{\logt^3} \sqrt{\frac{\epsilon^3}{\rho_\phi}}  + \frac{3}{128\logt^3}\sqrt{\frac{\epsilon^3}{\rho_\phi}}  \\
\le & -\frac{1}{5} \ufun,
\ead\ee
where the second inequality follows from \eqref{eq105} and \eqref{bound-uk'}.
Finally, by \eqref{bound-uk'} we have
\be\bad
 \frac{ u_{k'} \nabla^2 \Phi (x) u_{k'} }{\|u_{k'}\|^2}  \le & -\frac{2/5 \ufun} {9\uspace^2} = - \frac{32}{1125\logt} \sqrt{\rho_\phi \epsilon} \le - \frac{1}{40\logt} \sqrt{\rho_\phi \epsilon}.
\ead\ee
If NEON returns 0, by Bayes theorem, we have $\lambda_{\min} (\nabla^2 \Phi(x) ) \ge - \sqrt{\rho_\phi \epsilon}$ with high probability $1 - O(\delta)$ for a sufficiently small $\delta.$
\end{proof}

\subsection{Proof of Theorem \ref{thm:stocneon}}

We first provide several useful lemmas.

\begin{lemma}\cite{xu2018first}[Lemma 4] \label{lem:err-hess}
Suppose Assumption \ref{ass:stoc} holds. Let $\nabla^2 F_\gD (x) =\frac{1}{n} \sum_{i = 1}^{D_f} \nabla^2 F(x; \xi_i)$. For any $\epsilon, \delta \in (0,1)$, $x \in \br^d$ when $D_f \ge \frac{16L^2\log(2d/\delta)}{\epsilon^2}$, we have with probability at least $1-\delta$:
\be\bad
\|\nabla^2 F_\gD(x)  - \nabla^2 F(x) \| \le \epsilon.
\ead\ee
\end{lemma}

The next two lemmas mimic Lemma \ref{lem:err-hess} for first-order and third-order derivatives. Their proofs are mostly identical to that of Lemma \ref{lem:err-hess}, and hence we omit the details for brevity.  

\begin{lemma}\label{lem:err-grad}
Suppose Assumption \ref{ass:stoc} holds. Let $\nabla F_\gD (x) = \frac{1}{n}\sum_{i = 1}^{D_f} F(x; \xi_i)$. For any $\epsilon, \delta \in (0,1)$, $x \in \br^d$ when $D_f \ge \frac{16M^2\log(2d/\delta)}{\epsilon^2}$, we have with probability at least $1-\delta$:
\be\bad
\|\nabla F_\gD(x)  - \nabla F(x) \| \le \epsilon.
\ead\ee
\end{lemma}

\begin{lemma}\label{lem:err-third}
Suppose Assumption \ref{ass:stoc} holds. Let $\nabla^3 G_\gD (x) = \frac{1}{n} \sum_{i = 1}^{D_g} \nabla^3 G(x; \zeta_i)$. For any $\epsilon, \delta \in (0,1)$, $x \in \br^d$ when $D_g \ge \frac{16\rho^2\log(2d/\delta)}{\epsilon^2}$, we have with probability at least $1-\delta$:
\be\bad
\|\nabla^3 G_\gD(x)  - \nabla^3 G(x) \| \le \epsilon.
\ead\ee
\end{lemma}

The following two lemmas show that with sample batch sizes $D_f, D_g = \max\{\gO(\frac{1}{\rho_\phi \epsilon}), \gO(\frac{1}{\rho_\phi \epsilon}) \}$, we can bound the batch gradient and Hessian errors with high probability.
 
\begin{lemma}\label{lem:batch-hess-err}
Suppose Assumptions \ref{ass:stoc} and \ref{ass:variance-bound} hold. Set batch sizes $D_g = \tilde{\gO}(\frac{\kappa^{10}}{\rho_\phi \epsilon} + \frac{\kappa^{6}}{ \epsilon^2}), D_f = \tilde{\gO}(\frac{\kappa^6}{\rho_\phi \epsilon}+\frac{\kappa^2}{\epsilon^2}), $ with probability at least $1 - \delta$, we have the following inequality holds:
\be\bad\label{hess-err}
\|\nabla^2 \Phi_\gD(x)  - \nabla^2 \Phi(x) \| \le  \frac{1}{80\logt}\sqrt{\rho_\Phi \epsilon}.
\ead\ee
\end{lemma}

\begin{proof}
Note that the Hessian $\nabla^2\Phi(x)$ is given in \eqref{phi-hessian}.
The Hessian estimation error between $ \nabla^2 \Phi (x) $ and $ \nabla^2 \Phi _\gD(x) $ can be computed as

\begin{align}
    \bad
& \| \nabla^2 \Phi (x) - \nabla^2 \Phi_\gD (x) \|\\
\le &  \underbrace{\left\| \nabla_x^2 f(x, y^*(x)) -  \nabla_x^2 f_{\gD_F}(x, y_{\gD_G}^*(x)) \right\| }_{(I)} \\
& + \underbrace{ 2\left\| \frac{\partial y^*(x)}{\partial x} \cdot  \nabla^2_{yx} f(x, y^*(x))  -  \frac{\partial y_{\gD_G}^*(x)}{\partial x} \cdot  \nabla^2_{yx} f_{\gD_F}(x, y_{\gD_G}^*(x))\right\|}_{(II)}\\
& +  \underbrace{\left\| \frac{\partial^2 y^*(x)}{\partial^2 x} \cdot  \nabla_y f(x, y^*(x))  - \frac{\partial^2 y_{\gD_G}^*(x)}{\partial^2 x} \cdot  \nabla_y f_{\gD_F}(x, y_{\gD_G}^*(x)) \right\|}_{(III)}\\
& + \underbrace{ \left\|  \frac{\partial y^*(x)}{\partial x}\cdot \nabla_y^2 f(x, y^*(x))  \cdot  \frac{\partial y^*(x)}{\partial x}^\top  - \frac{\partial y_{\gD_G}^*(x)}{\partial x}\cdot \nabla_y^2 f_{\gD_F}(x, y_{\gD_G}^*(x))  \cdot  \frac{\partial y_{\gD_G}^*(x)}{\partial x}^\top \right\| }_{(IV)},\\
\ead
\end{align}
where the inequality is obtained by triangle inequality and $\nabla^2_{yx} f(x, y) = \nabla^2_{xy} f(x, y)^\top$ for any smooth functions. Similar to the proof of Lemma \ref{hess-Lipschitz}, we bound the terms $(I) - (IV)$. The first term can be bounded as follows: 
\begin{align}
    \bad\label{eq:term1-stoc}
(I) = & \left\| \nabla_x^2 f(x, y^*(x)) -  \nabla_x^2 f_{\gD_F}(x, y_{\gD_G}^*(x)) \right\|\\
\le &  \left\| \nabla_x^2 f(x, y^*(x)) -  \nabla_x^2 f_{\gD_F}(x, y^*(x)) \right\| +  \left\| \nabla_x^2 f_{\gD_F}(x, y^*(x)) -  \nabla_x^2 f_{\gD_F}(x, y_{\gD_G}^*(x)) \right\|\\
\le &  \left\| \nabla_x^2 f(x, y^*(x)) -  \nabla_x^2 f_{\gD_F}(x, y^*(x)) \right\| + \rho \left\|  y^*(x) -  y_{\gD_G}^*(x) \right\|,
\ead
\end{align}
where we used the Assumption \ref{assu:bilevel} in the last step. Secondly, we have  
\be\bad\label{eq:term2-stoc}
(II) = & 2\left\| \frac{\partial y^*(x)}{\partial x} \cdot  \nabla^2_{yx} f(x, y^*(x))  -  \frac{\partial y_{\gD_G}^*(x)}{\partial x} \cdot  \nabla^2_{yx} f_{\gD_F}(x, y_{\gD_G}^*(x))\right\|\\
\le & 2  \left\| \frac{\partial y^*(x)}{\partial x} \right\| \left \|\nabla^2_{yx} f(x, y^*(x))  -  \nabla^2_{yx} f_{\gD_F}(x, y_{\gD_G}^*(x))\right\| \\
& +  2\left\| \frac{\partial y^*(x)}{\partial x}  -  \frac{\partial y_{\gD_G}^*(x)}{\partial x} \right\| \left\| \nabla^2_{yx} f_{\gD_F}(x, y_{\gD_G}^*(x))\right\|\\
 \le & \frac{2\ell}{\mu}\left( \left\| \nabla^2_{yx} f(x, y^*(x)) -  \nabla^2_{yx} f_{\gD_F}(x, y^*(x)) \right\| + \rho \left\|  y^*(x) -  y_{\gD_G}^*(x) \right\| \right)\\
 & + 2\ell \left\| \frac{\partial y^*(x)}{\partial x}  -  \frac{\partial y^*(x')}{\partial x'} \right\|,
\ead\ee
where the first inequality is due to the triangle inequality and the Cauchy-Schwarz inequality and the last step uses Assumption \ref{assu:bilevel} as well as Proposition \ref{pro:bounds}. Furthermore, we bound $\left\| \frac{\partial y^*(x)}{\partial x}  -  \frac{\partial y_{\gD_G}^*(x)}{\partial x} \right\|$ as 
\begin{align}
    \bad\label{eq:pydiff-stoc}
&\hspace{-4mm}\left\| \frac{\partial y^*(x)}{\partial x}  -   \frac{\partial y_{\gD_G}^*(x)}{\partial x} \right\| \\= & \left\| \nabla^2_{xy} g(x, y^*(x))\cdot \nabla^2_y g(x, y^*(x))^{-1} - \nabla^2_{xy} g_{\gD_G}(x, y_{\gD_G}^*(x))\cdot \nabla^2_y g_{\gD_G}(x, y_{\gD_G}^*(x))^{-1} \right\|\\
\le &  \left\| \nabla^2_{xy} g(x, y^*(x))\cdot \nabla^2_y g(x, y^*(x))^{-1} - \nabla^2_{xy} g(x, y^*(x))\cdot \nabla^2_y g_{\gD_G}(x, y_{\gD_G}^*(x))^{-1} \right\|\\
& +  \left\| \nabla^2_{xy} g(x, y^*(x))\cdot \nabla^2_y g_{\gD_G}(x, y_{\gD_G}^*(x))^{-1} - \nabla^2_{xy} g_{\gD_G}(x, y_{\gD_G}^*(x))\cdot \nabla^2_y g_{\gD_G}(x, y_{\gD_G}^*(x))^{-1} \right\|\\
\le & \frac{\ell}{\mu^2}  \left( \left\| \nabla^2_y g(x, y^*(x)) - \nabla^2_y g_{\gD_G}(x, y^*(x)) \right\| + \rho \| y^*(x) -  y_{\gD_G}^*(x)\|  \right) \\
&+ \frac{1}{\mu}  \left( \left\| \nabla^2_{xy} g(x, y^*(x)) - \nabla^2_{xy} g_{\gD_G}(x, y^*(x))\right\|+ \rho \| y^*(x) -  y_{\gD_G}^*(x)\| \right),
\ead
\end{align}
where the first equality is by \eqref{solve-partial-y*-partial-x} and the last step follows the fact that $\|X^{-1} - Y^{-1}\| \le \|X^{-1}\| \|X - Y\| \|Y^{-1}\|$, Assumption \ref{assu:bilevel} and Proposition \ref{pro:bounds}. Combining \eqref{eq:term2-stoc} and \eqref{eq:pydiff-stoc}, we get
\be\bad
(II)  \le &\frac{2\ell}{\mu}\left\| \nabla^2_{yx} f(x, y^*(x)) -  \nabla^2_{yx} f_{\gD_F}(x, y^*(x)) \right\| + \frac{2\ell^2}{\mu^2}  \left\| \nabla^2_y g(x, y^*(x)) - \nabla^2_y g_{\gD_G}(x, y^*(x)) \right\| \\
&+ \frac{2\ell}{\mu}  \left\| \nabla^2_{xy} g(x, y^*(x)) - \nabla^2_{xy} g_{\gD_G}(x, y^*(x))\right\| + \left( \frac{4\ell}{\mu} + \frac{2\ell^2}{\mu^2} \right) \rho \| y^*(x) -  y_{\gD_G}^*(x)\|.
\ead\ee
The third term $(III)$ can be written as 
\be\bad\label{eq:term3-stoc}
(III) = &  \left\| \frac{\partial^2 y^*(x)}{\partial^2 x} \cdot  \nabla_y f(x, y^*(x))  - \frac{\partial^2 y_{\gD_G}^*(x)}{\partial^2 x} \cdot  \nabla_y f_{\gD_F}(x, y_{\gD_G}^*(x)) \right\| \\
\le &  \left\| \frac{\partial^2 y^*(x)}{\partial^2 x} \right\| \left\| \nabla_y f(x, y^*(x))  - \nabla_y f_{\gD_F}(x, y_{\gD_G}^*(x)) \right\| \\
+  &\|\nabla_y f_{\gD_F}(x, y_{\gD_G}^*(x)) \| \cdot \left\|\frac{\partial^2 y^*(x)}{\partial^2 x}  - \frac{\partial^2 y_{\gD_G}^*(x)}{\partial^2 x} \right\|\\
\le & \frac{\rho}{\mu} \left( 1 + \frac{\ell}{\mu} \right)^2  \left\| \nabla_y f(x, y^*(x))  - \nabla_y f_{\gD_F}(x, y_{\gD_G}^*(x)) \right\| +  M \left\|\frac{\partial^2 y^*(x)}{\partial^2 x}  - \frac{\partial^2 y_{\gD_G}^*(x)}{\partial^2 x} \right\|,
\ead\ee
where the last inequality follows \eqref{eq:y-hess-bound} and Proposition \ref{pro:bounds}. To bound $\left\|\frac{\partial^2 y^*(x)}{\partial^2 x}  -\frac{\partial^2 y_{\gD_G}^*(x)}{\partial^2 x} \right\|$, we follow the computation of the second to last inequality in \eqref{eq:y-hess-diff} and get 
\begin{align}
    \bad\label{eq:y-hess-dif-stoc}
&\left\|\frac{\partial^2 y^*(x)}{\partial^2 x}  -\frac{\partial^2 y_{\gD_G}^*(x)}{\partial^2 x} \right\|\\
\le & \left[ \frac{1}{\mu} \|\nabla^3_{xxy} g(x, y^*(x))  - \nabla^3_{xxy} g_{\gD_G}(x, y_{\gD_G}^*(x)) \|\right.\\
& +  \frac{2}{\mu} \left\| \frac{\partial y^*(x)}{\partial x} \nabla^3_{yxy}g(x, y^*(x))  - \frac{\partial y_{\gD_G}^*(x)}{\partial x} \nabla^3_{yxy} g_{\gD_G}(x, y_{\gD_G}^*(x))  \right\| \\ 
& \left. + \frac{1}{\mu} \left\| \frac{\partial y^*(x)}{\partial x}\cdot \nabla_y^3 g(x, y^*(x))  \cdot  \frac{\partial y^*(x)}{\partial x}^\top - \frac{\partial y_{\gD_G}^*(x)}{\partial x}\cdot \nabla_y^3 g_{\gD_G}(x, y_{\gD_G}^*(x))  \cdot  \frac{\partial y_{\gD_G}^*(x)}{\partial x}^\top  \right\| \right]\\
& + \rho \left( 1 + \frac{\ell}{\mu} \right)^2 \left\|  \nabla_y^2 g(x, y^*(x))^{-1}  -  \nabla_y^2 g_{\gD_G}(x, y_{\gD_G}^*(x))^{-1}  \right\|.
\ead
\end{align}
To bound the term $\left\| \frac{\partial y^*(x)}{\partial x} \nabla^3_{yxy} g(x, y^*(x))  - \frac{\partial y_{\gD_G}^*(x)}{\partial x} \nabla^3_{yxy} g_{\gD_G}(x, y_{\gD_G}^*(x))  \right\|$, we follow the computation of $(II)$ and get
\be\bad
&\left\| \frac{\partial y^*(x)}{\partial x} \nabla^3_{yxy} g(x, y^*(x))  - \frac{\partial y_{\gD_G}^*(x)}{\partial x} \nabla^3_{yxy} g_{\gD_G}(x, y_{\gD_G}^*(x))  \right\| \\
\le &\frac{2\ell}{\mu}\left\|\nabla^3_{yxy} g(x, y^*(x)) - \nabla^3_{yxy} g_{\gD_G}(x, y^*(x)) \right\| + \frac{2\ell\rho}{\mu^2}  \left\| \nabla^2_y g(x, y^*(x)) - \nabla^2_y g_{\gD_G}(x, y^*(x)) \right\| \\
&+ \frac{2\rho}{\mu}  \left\| \nabla^2_{xy} g(x, y^*(x)) - \nabla^2_{xy} g_{\gD_G}(x, y^*(x))\right\| + \left( \frac{2\ell\nu}{\mu} + \frac{2\ell\rho^2}{\mu^2}  + \frac{2\rho^2}{\mu}\right) \| y^*(x) -  y_{\gD_G}^*(x)\|.
\ead\ee
For the third term, we follow similar computation in \eqref{eq:term-pgp} and give the following bound
\be\bad \label{term-pgp2}
&\left\| \frac{\partial y^*(x)}{\partial x}\cdot \nabla_y^3 g(x, y^*(x))  \cdot  \frac{\partial y^*(x)}{\partial x}^\top - \frac{\partial y_{\gD_G}^*(x)}{\partial x}\cdot \nabla_y^3 g_{\gD_G}(x, y_{\gD_G}^*(x))  \cdot  \frac{\partial y_{\gD_G}^*(x)}{\partial x}^\top  \right\| \\
\le & 2 \cdot \frac{\rho \ell}{\mu} \left\| \frac{\partial y^*(x)}{\partial x} -  \frac{\partial y_{\gD_G}^*(x)}{\partial x} \right\| + \left( \frac{ \ell}{\mu}\right)^2 \left\| \nabla_y^3 g(x, y^*(x)) -  \nabla_y^3 g_{\gD_G}(x, y_{\gD_G}^*(x))  \right\| \\
\le & \frac{2\rho \ell^2}{\mu^3} \left\| \nabla^2_y g(x, y^*(x)) - \nabla^2_y g_{\gD_G}(x, y^*(x)) \right\| + \frac{2 \rho \ell}{\mu^2}  \left\| \nabla^2_{xy} g(x, y^*(x)) - \nabla^2_{xy} g_{\gD_G}(x, y^*(x))\right\|\\
&+  \frac{\ell^2}{\mu^2} \left\| \nabla_y^3 g(x, y^*(x)) -  \nabla_y^3 g_{\gD_G}(x, y^*(x))  \right\| + \left( \frac{2\rho^2 \ell^2}{\mu^3} +  \frac{2 \rho^2 \ell}{\mu^2}  + \frac{\ell^2\nu}{\mu^2} \right)  \| y^*(x) -  y_{\gD_G}^*(x)\|,
\ead\ee
where the last inequality is by \eqref{eq:pydiff-stoc}, Assumption \ref{assu:bilevel} and Proposition \ref{pro:bounds}. Combining \eqref{eq:y-hess-dif-stoc} - \eqref{term-pgp2} together yields
\begin{align}
    \bad\label{eq:y-hess-dif-stoc-final}
& \left\|\frac{\partial^2 y^*(x)}{\partial^2 x}  -\frac{\partial^2 y_{\gD_G}^*(x)}{\partial^2 x} \right\|\\
 \le & \frac{1}{\mu} \|\nabla^3_{xxy} g(x, y^*(x))  - \nabla^3_{xxy} g_{\gD_G}(x, y^*(x)) \|+ \frac{4\ell}{\mu^2}\left\|\nabla^3_{yxy} g(x, y^*(x)) - \nabla^3_{yxy} g_{\gD_G}(x, y^*(x)) \right\|\\
& + \frac{4\ell\rho}{\mu^3}  \left\| \nabla^2_y g(x, y^*(x)) - \nabla^2_y g_{\gD_G}(x, y^*(x)) \right\| + \frac{4\rho}{\mu^2}  \left\| \nabla^2_{xy} g(x, y^*(x)) - \nabla^2_{xy} g_{\gD_G}(x, y^*(x))\right\|\\
& + \frac{2\rho \ell^2}{\mu^4} \left\| \nabla^2_y g(x, y^*(x)) - \nabla^2_y g_{\gD_G}(x, y^*(x)) \right\| + \frac{2 \rho \ell}{\mu^3}  \left\| \nabla^2_{xy} g(x, y^*(x)) - \nabla^2_{xy} g_{\gD_G}(x, y^*(x))\right\|\\
& +  \frac{\ell^2}{\mu^3} \left\| \nabla_y^3 g(x, y^*(x)) -  \nabla_y^3 g_{\gD_G}(x, y^*(x))  \right\| +  \frac{\rho}{\mu^2} \left( 1 + \frac{\ell}{\mu} \right)^2 \left\|  \nabla_y^2 g(x, y^*(x))  -  \nabla_y^2 g_{\gD_G}(x, y^*(x))  \right\|\\
& + \left(\frac{\rho^2}{\mu^2} \left( 1 + \frac{\ell}{\mu} \right)^2 + \frac{\nu}{\mu} + \frac{4\ell\nu + 4\rho^2}{\mu^2} + \frac{6\ell\rho^2 + \ell^2\nu}{\mu^3}  +  \frac{2\rho^2 \ell^2}{\mu^4}  \right) \| y^*(x) -  y_{\gD_G}^*(x)\|,
\ead
\end{align}
where the last inequality uses the fact that $\|X^{-1} - Y^{-1}\| \le \|X^{-1}\| \|X - Y\| \|Y^{-1}\|$. Plug \eqref{eq:y-hess-dif-stoc-final} into \eqref{eq:term3-stoc} leads to
\begin{align}
\bad
&(III)\\
\le & \frac{\rho}{\mu} \left( 1 + \frac{\ell}{\mu} \right)^2  \left\| \nabla_y f(x, y^*(x))  - \nabla_y f_{\gD_F}(x, y^*(x)) \right\| +  \frac{M}{\mu}  \|\nabla^3_{xxy} g(x, y^*(x))  - \nabla^3_{xxy} g_{\gD_G}(x, y^*(x)) \| \\
+ &\frac{4\ell M}{\mu^2}\left\|\nabla^3_{yxy} g(x, y^*(x)) - \nabla^3_{yxy} g_{\gD_G}(x, y^*(x)) \right\| +  \frac{\ell^2M}{\mu^3} \left\| \nabla_y^3 g(x, y^*(x)) -  \nabla_y^3 g_{\gD_G}(x, y^*(x))  \right\|  \\
 + &\left( \frac{4\ell \rho M}{\mu^3}  +  \frac{2\rho \ell^2M}{\mu^4} +\frac{\rho M}{\mu^2} \left( 1 + \frac{\ell }{\mu} \right)^2  \right)  \left\| \nabla^2_y g(x, y^*(x)) - \nabla^2_y g_{\gD_G}(x, y^*(x)) \right\| \\
+ &\left(  \frac{4\rho M}{\mu^2}    +\frac{2 \rho \ell M}{\mu^3}      \right) \left\| \nabla^2_{xy} g(x, y^*(x)) - \nabla^2_{xy} g_{\gD_G}(x, y^*(x))\right\| \\
+ &\left( \left(\frac{\rho \ell }{\mu} +  \frac{\rho^2 M}{\mu^2} \right)\left( 1 + \frac{\ell }{\mu} \right)^2 + \frac{\nu M}{\mu} + \frac{4\ell \nu M + 4\rho^2 M}{\mu^2} + \frac{6\ell \rho^2M+ \ell^2\nu M}{\mu^3}+ \frac{2\rho^2 \ell^2M}{\mu^4}  \right) \| y^*(x) -  y_{\gD_G}^*(x)\|.
\ead
\end{align}
Finally, similar to the computation in \eqref{term-pgp2}, we bound the last term as
\be\bad
(IV) = & \left\|  \frac{\partial y^*(x)}{\partial x}\cdot \nabla_y^2 f(x, y^*(x))  \cdot  \frac{\partial y^*(x)}{\partial x}^\top  - \frac{\partial y_{\gD_G}^*(x)}{\partial x}\cdot \nabla_y^2 f_{\gD_F}(x, y_{\gD_G}^*(x))  \cdot  \frac{\partial y_{\gD_G}^*(x)}{\partial x}^\top \right\| \\
\le & \frac{2\ell^3}{\mu^3} \left\| \nabla^2_y g(x, y^*(x)) - \nabla^2_y g_{\gD_G}(x, y^*(x)) \right\| +  \frac{2\ell^2}{\mu^2}  \left\| \nabla^2_{xy} g(x, y^*(x)) - \nabla^2_{xy} g_{\gD_G}(x, y^*(x))\right\|\\
& +   \frac{ \ell^2}{\mu^2} \left\| \nabla_y^2 f(x, y^*(x)) -  \nabla_y^2 f_{\gD_F}(x, y^*(x))  \right\| + \left(\frac{2\ell^3\rho}{\mu^3}  + \frac{3\ell^2\rho}{\mu^2} \right) \| y^*(x) -  y_{\gD_G}^*(x)\|.
\ead\ee
Combining terms $(I) - (IV)$ together, we have
\be\bad\label{eq:hessphi-err}
& \| \nabla^2 \Phi (x) - \nabla^2 \Phi_\gD (x) \|\\
\le &  \left\| \nabla_x^2 f(x, y^*(x)) -  \nabla_x^2 f_{\gD_F}(x, y^*(x)) \right\| + \frac{2\ell}{\mu}\left\| \nabla^2_{yx} f(x, y^*(x)) -  \nabla^2_{yx} f_{\gD_F}(x, y^*(x)) \right\|\\
& + \frac{ \ell^2}{\mu^2} \left\| \nabla_y^2 f(x, y^*(x)) -  \nabla_y^2 f_{\gD_F}(x, y^*(x))  \right\|+  \frac{\rho}{\mu} \left( 1 + \frac{\ell}{\mu} \right)^2  \left\| \nabla_y f(x, y^*(x))  - \nabla_y f_{\gD_F}(x, y^*(x)) \right\|\\
& + \left( \frac{2\ell^2}{\mu^2}  + \frac{4\ell\rho M + 2\ell^3}{\mu^3}  +  \frac{2\rho \ell^2M}{\mu^4} +\frac{\rho M}{\mu^2} \left( 1 + \frac{\ell}{\mu} \right)^2   \right)  \left\| \nabla^2_y g(x, y^*(x)) - \nabla^2_y g_{\gD_G}(x, y^*(x)) \right\| \\
&+ \left( \frac{2\ell}{\mu} + \frac{4\rho M + 2\ell^2}{\mu^2}    +\frac{2 \rho \ell M}{\mu^3}     \right)  \left\| \nabla^2_{xy} g(x, y^*(x)) - \nabla^2_{xy} g_{\gD_G}(x, y^*(x))\right\|\\
& + \frac{M}{\mu}\|\nabla^3_{xxy} g(x, y^*(x))  - \nabla^3_{xxy} g_{\gD_G}(x, y^*(x)) \|+  \frac{\ell^2M}{\mu^3} \left\| \nabla_y^3 g(x, y^*(x)) -  \nabla_y^3 g_{\gD_G}(x, y^*(x))  \right\|  \\
&+  \frac{4\ell M}{\mu^2}\left\|\nabla^3_{yxy} g(x, y^*(x)) - \nabla^3_{yxy} g_{\gD_G}(x, y^*(x)) \right\| \\
& + \left( \rho +  \left(\frac{\rho \ell}{\mu} +  \frac{\rho^2 M}{\mu^2} \right)\left( 1 + \frac{\ell}{\mu} \right)^2 + \frac{\nu M + 4\ell\rho}{\mu} + \frac{4\ell\nu M + 4\rho^2 M + 5\ell^2\rho}{\mu^2}  \right.\\
& \left. \quad \quad  + \frac{6\ell\rho^2M+ \ell^2\nu M + 2\ell^3\rho}{\mu^3}   +  \frac{2\rho^2 \ell^2M}{\mu^4}  \right) \| y^*(x) -  y_{\gD_G}^*(x)\|.
\ead\ee
To deal with the last term $\| y^*(x) -  y_{\gD_G}^*(x)\|$, we utilize the strongly convexity of $g(x, y)$ and $g_{\gD_G}(x, y)$ with respect to $y$ and have 
\be\bad\label{eq126}
&0 \ge \la \nabla_y g(x, y_{\gD_G}^*(x)) - \nabla_y g(x, y^*(x)),  y^*(x) -  y_{\gD_G}^*(x) \ra + \mu \| y^*(x) -  y_{\gD_G}^*(x)\|^2,\\
&0 \ge \la \nabla_y g_{\gD_G}(x, y_{\gD_G}^*(x)) - \nabla_y g_{\gD_G}(x, y^*(x)),  y^*(x) -  y_{\gD_G}^*(x) \ra + \mu \| y^*(x) -  y_{\gD_G}^*(x)\|^2.\\
\ead\ee
Note that the optimality conditions are $ \nabla_y g(x, y^*(x)) = 0, \nabla_y g_{\gD_G}(x, y_{\gD_G}^*(x)) = 0,$ which combining with \eqref{eq126}, yields
\be\bad \label{eq127}
&0 \ge \la \nabla_y g(x, y_{\gD_G}^*(x)) - \nabla_y g_{\gD_G}(x, y^*(x)),  y^*(x) -  y_{\gD_G}^*(x) \ra + 2\mu \| y^*(x) -  y_{\gD_G}^*(x)\|^2.\\
\ead\ee
Therefore, we can bound $\| y^*(x) -  y_{\gD_G}^*(x)\|$ as
\be\bad\label{eq:y-diff-bound-stoc-sc}
&\hspace{-5mm}\| y^*(x) -  y_{\gD_G}^*(x)\| \\\le & \frac{1}{2\mu}\| \nabla_y g(x, y_{\gD_G}^*(x))  - \nabla_y g_{\gD_G}(x, y^*(x))\|\\
\le &  \frac{1}{2\mu}\| \nabla_y g(x, y_{\gD_G}^*(x)) \| +  \frac{1}{2\mu}\| \nabla_y g_{\gD_G}(x, y^*(x))\|\\
\le & \frac{1}{2\mu}\| \nabla_y g(x, y_{\gD_G}^*(x)) - \nabla_y g_{\gD_G}(x, y_{\gD_G}^*(x))  \| + \frac{1}{2\mu} \| \nabla_y g(x, y^*(x)) - \nabla_y g_{\gD_G}(x, y^*(x))\|,
\ead\ee
where the first inequality is obtained from \eqref{eq127} and the last inequality is due to the optimality conditions. Plugging \eqref{eq:y-diff-bound-stoc-sc} into \eqref{eq:hessphi-err}, we get 11 different terms in \eqref{eq:hessphi-err} and the major parts of them are
\be\bad
&\left\| \nabla_x^2 f(x, y^*(x)) -  \nabla_x^2 f_{\gD_F}(x, y^*(x)) \right\|, \quad \left\| \nabla^2_{yx} f(x, y^*(x)) -  \nabla^2_{yx} f_{\gD_F}(x, y^*(x)) \right\|,\\
&  \left\| \nabla_y^2 f(x, y^*(x)) -  \nabla_y^2 f_{\gD_F}(x, y^*(x))  \right\|,  \quad  \left\| \nabla_y f(x, y^*(x))  - \nabla_y f_{\gD_F}(x, y^*(x)) \right\|,\\
&  \left\| \nabla^2_y g(x, y^*(x)) - \nabla^2_y g_{\gD_G}(x, y^*(x)) \right\| ,  \quad  \left\| \nabla^2_{xy} g(x, y^*(x)) - \nabla^2_{xy} g_{\gD_G}(x, y^*(x))\right\|,\\
& \|\nabla^3_{xxy} g(x, y^*(x))  - \nabla^3_{xxy} g_{\gD_G}(x, y^*(x)) \|, \quad \|\nabla^3_{yxy} g(x, y^*(x))  - \nabla^3_{yxy} g_{\gD_G}(x, y^*(x)) \|,\\
& \left\| \nabla_y^3 g(x, y^*(x)) -  \nabla_y^3 g_{\gD_G}(x, y^*(x))  \right\|,\quad \| \nabla_y g(x, y_{\gD_G}^*(x)) - \nabla_y g_{\gD_G}(x, y_{\gD_G}^*(x))  \|,\\
& \| \nabla_y g(x, y^*(x)) - \nabla_y g_{\gD_G}(x, y^*(x))\|.
\ead\ee
Note that each of the above terms is the difference between the empirical and population derivative. Therefore, each of them can be bounded by one of Lemmas \ref{lem:err-hess} - \ref{lem:err-third}. For example, we consider the term with the largest coefficient, e.g. 
\be\bad
\gO(\kappa^5)   \| \nabla_y g(x, y^*(x)) - \nabla_y g_{\gD_G}(x, y^*(x))\|.
\ead\ee
By Lemma \ref{lem:err-hess}, we can choose $D_g = \gO\left(\kappa^{10} \cdot\frac{\log(2d/\delta')}{\rho_\phi \epsilon}\right)$ so that the above term can be bounded by $\frac{1}{880\logt} \sqrt{\rho_\phi \epsilon}$ with probability $1 - \delta'$. 
For other terms, we can apply the similar techniques, and select batch sizes $D_f = \gO\left(\kappa^6 \cdot\frac{\log(2d/\delta')}{\rho_\phi \epsilon}\right), D_g = \gO\left(\kappa^{10} \cdot\frac{\log(2d/\delta')}{\rho_\phi \epsilon}\right)$ such that each term can be bounded by $\frac{1}{880\logt} \sqrt{\rho_\phi \epsilon}$ with probability $1 - \delta'$. This completes the proof of \eqref{hess-err} with probability $1 - \delta = (1 - \delta')^{11}$. 
\end{proof}

\begin{lemma}\label{lem:batch-grad-err}
Suppose Assumptions \ref{ass:stoc} and \ref{ass:variance-bound} hold. Set batch sizes $D_g = \gO(\kappa^6\epsilon^{-2}), D_f = \gO(\kappa^2\epsilon^{-2}), $ with probability at least $1 - \delta$, we have the following inequality holds:
\be\bad\label{grad-err}
\|\nabla \Phi_\gD(x)  - \nabla \Phi(x) \| \le  \frac{1}{10}\epsilon.
\ead\ee
\end{lemma}
\begin{proof}
The gradient estimate error between $ \nabla \Phi (x) $ and $ \nabla \Phi _\gD(x) $ can be computed as
\begin{align}
    \bad
& \| \nabla \Phi (x) - \nabla \Phi_\gD (x) \|\\
\le & \left\| \nabla_x f(x, y^*(x)) -  \nabla_x f_{\gD_F}(x, y_{\gD_G}^*(x)) \right\|\\
& + \left\|\frac{\partial y^*(x)}{x}\nabla_y f(x, y^*(x)) - \frac{\partial y_{\gD_G}^*(x)}{x}\nabla_y f_{\gD_F}(x, y_{\gD_G}^*(x))\right\|\\
\le &  \left\| \nabla_x f(x, y^*(x)) -  \nabla_x f_{\gD_F}(x, y_{\gD_G}^*(x)) \right\|  +  \left\|\frac{\partial y^*(x)}{x} \right\| \left\| \nabla_y f(x, y^*(x)) -  \nabla_y f_{\gD_F}(x, y^*(x))\right\| \\
&+ \left\|\frac{\partial y^*(x)}{x} -   \frac{\partial y_{\gD_G}^*(x)}{x}\right\| \left\| \nabla_y f_{\gD_F}(x, y^*(x))\right\| \\
\le & \left\| \nabla_x f(x, y^*(x)) -  \nabla_x f_{\gD_F}(x, y^*(x)) \right\| + \ell \|  y^*(x) -  y_{\gD_G}^*(x)\| \\
& +  \frac{\ell}{\mu} \left\| \nabla_y f(x, y^*(x)) -  \nabla_y f _{\gD_F}(x, y^*(x))\right\| + \frac{\ell^2}{\mu}  \|  y^*(x) -  y_{\gD_G}^*(x)\| \\
& + \frac{\ell M}{\mu^2}  \left( \left\| \nabla^2_y g(x, y^*(x)) - \nabla^2_y g_{\gD_G}(x, y^*(x)) \right\| + \rho \| y^*(x) -  y_{\gD_G}^*(x)\|  \right) \\
& + \frac{M}{\mu}  \left( \left\| \nabla^2_{xy} g(x, y^*(x)) - \nabla^2_{xy} g_{\gD_G}(x, y^*(x))\right\|+ \rho \| y^*(x) -  y_{\gD_G}^*(x)\| \right)\\
\le & \left\| \nabla_x f(x, y^*(x)) -  \nabla_x f_{\gD_F}(x, y^*(x)) \right\|  +  \frac{\ell}{\mu} \left\| \nabla_y f(x, y^*(x)) -  \nabla_y f (x, y^*(x))\right\| \\
&+ \frac{\ell M}{\mu^2}  \left\| \nabla^2_y g(x, y^*(x)) - \nabla^2_y g_{\gD_G}(x, y^*(x)) \right\|  + \frac{M}{\mu}  \left\| \nabla^2_{xy} g(x, y^*(x)) - \nabla^2_{xy} g_{\gD_G}(x, y^*(x))\right\|\\
& + \frac{1}{2\mu}\left( \ell + \frac{\ell^2 + M\rho}{\mu} + \frac{\ell M\rho}{\mu^2} \right)\left(\| \nabla_y g(x, y_{\gD_G}^*(x)) - \nabla_y g_{\gD_G}(x, y_{\gD_G}^*(x))  \| \right.\\&\left.+ \| \nabla_y g(x, y^*(x)) - \nabla_y g_{\gD_G}(x, y^*(x))\|\right),
\ead
\end{align}
where the first inequality is obtained by \eqref{phi-grad} and the triangle inequality, the third inequality is due to \eqref{eq:pydiff-stoc},  Assumption \ref{assu:bilevel} and Proposition \ref{pro:bounds}. The last inequality is by \eqref{eq:y-diff-bound-stoc-sc}. Nota that there are 6 different terms in the above gradient estimate error and they are
\be\bad
&\left\| \nabla_x f(x, y^*(x)) -  \nabla_x f_{\gD_F}(x, y^*(x)) \right\|, \quad  \left\| \nabla_y f(x, y^*(x))  - \nabla_y f_{\gD_F}(x, y^*(x)) \right\|,\\
&  \left\| \nabla^2_y g(x, y^*(x)) - \nabla^2_y g_{\gD_G}(x, y^*(x)) \right\| ,  \quad  \left\| \nabla^2_{xy} g(x, y^*(x)) - \nabla^2_{xy} g_{\gD_G}(x, y^*(x))\right\|,\\
& \| \nabla_y g(x, y_{\gD_G}^*(x)) - \nabla_y g_{\gD_G}(x, y_{\gD_G}^*(x))  \|, \quad \| \nabla_y g(x, y^*(x)) - \nabla_y g_{\gD_G}(x, y^*(x))\|.
\ead\ee
For every term, we can apply one of Lemma \ref{lem:err-hess} - Lemma \ref{lem:err-third} and set batch sizes $$D_f = \gO\left(\kappa^2 \cdot\frac{\log(2d/\delta')}{\epsilon^2}\right), D_g = \gO\left(\kappa^6 \cdot\frac{\log(2d/\delta')}{\epsilon^2}\right)$$ such that each term can be bounded by $\frac{1}{60} \epsilon$ with probability $1 - \delta'$. This completes the proof of \eqref{grad-err} with probability $1 - \delta = (1 - \delta')^{6}$. 
\end{proof}

For the StocBiO algorithm (Algorithm \ref{alg:INEONstocbio}), we have the following descent lemma.
\begin{lemma}\label{lem:saddle-decrease}
Suppose Assumption \ref{ass:stoc} holds. We have with high probability the update $x_{k+1} =  x_{k}  - \frac{\bar{\xi}}{80} \sqrt{\frac{ \epsilon}{\rho_\phi}}u$ in Algorithm \ref{alg:INEONstocbio} yields
\begin{align} 
\mathbb{E} \Phi(x_{k+1})\leq&\mathbb{E} \Phi(x_k) - \frac{1}{3 \cdot 80^3\logt^3}\cdot \sqrt{\frac{\epsilon^3}{\rho_\phi}}.
\end{align}
\end{lemma}
\begin{proof}
Combining Lemma \ref{lem:neon} and Lemma \ref{lem:batch-hess-err}, we have 
\begin{align} 
\frac{u_{out}^\top \nabla^2 \Phi (x_k) u_{out}}{\|u_{out}\|^2} \le & \frac{u_{out}^\top \nabla^2 \Phi_\gD (x_k) u_{out}}{\|u_{out}\|^2} + \left\lvert \frac{u_{out}^\top \nabla^2 \Phi_\gD (x_k) u_{out}}{\|u_{out}\|^2} - \frac{u_{out}^\top \nabla^2 \Phi (x_k) u_{out}}{\|u_{out}\|^2} \right\rvert \nonumber\\
\le & \frac{u_{out}^\top \nabla^2 \Phi_\gD (x_k) u_{out}}{\|u_{out}\|^2} + \left\| \nabla^2 \Phi_\gD (x_k)  - \nabla^2 \Phi (x_k) \right\| \nonumber\\
\le & \left(-\frac{1}{40\logt}  + \frac{1}{80\logt} \right) \sqrt{\rho_\phi \epsilon}\nonumber\\
= & - \frac{1}{80\logt} \sqrt{\rho_\phi \epsilon}.
\end{align}
By \cite{xu2018first}[Lemma 1], we have 
\begin{align} 
\mathbb{E} \Phi(x_{k+1})\leq&\mathbb{E} \Phi(x_k) - \frac{1}{3 \cdot 80^3\logt^3}\cdot \sqrt{\frac{\epsilon^3}{\rho_\phi}}.
\end{align}
\end{proof}

\begin{lemma}\label{lem:stocbio-decrease}
Suppose Assumptions \ref{ass:stoc} and \ref{ass:variance-bound} hold. Set parameters as 
\begin{align*}
S = O(\kappa^5 \epsilon^{-2}), B =  O(\kappa^2 \epsilon^{-2}),  D_f= O(\kappa^2 \epsilon^{-2}), D_g = O(\kappa^2 \epsilon^{-2}), \\ Q= O(\kappa \log \frac{1}{\epsilon}), D = O(\kappa \log \frac{1}{\epsilon}), \alpha = \frac{2}{\ell+\mu}, \beta=\frac{1}{4L_\phi}.
\end{align*}
With high probability, the update $x_{k+1} =  x_{k}  - \widehat{\nabla} \Phi(x_k)$ in Algorithm \ref{alg:INEONstocbio} yields
\be\bad\label{eq:f-dec2}
\BE \left[\Phi(x_{k+1}) - \Phi(x_{k}) \right]  \le   -\frac{1}{16L_\phi}\mathbb{E}\|\nabla\Phi(x_k)\|^2 + \frac{1}{400L_\phi} \epsilon^2.
\ead\ee
\end{lemma}

To prove Lemma \ref{lem:stocbio-decrease}, we borrow the following two useful lemmas for stocBiO from \cite{ji2021bilevel}.

\begin{lemma}\cite{ji2021bilevel}[Lemma 7]\label{le:first_m}
Suppose Assumptions \ref{ass:stoc} and \ref{ass:variance-bound} hold. Denote the $\mathbb{E}_k$ as the conditional expectation conditioning on $x_k$ and $y_k^D$, i.e., $\mathbb{E}_k[\cdot] = \mathbb{E}[\cdot\mid x_k, y_k^D]$. We have
\begin{align*}
\big\|\mathbb{E}_k\widehat \nabla \Phi(x_k)-\nabla \Phi(x_k)\big\|^2\leq 2 \Big( \ell +\frac{\ell^2}{\mu} + \frac{M\rho}{\mu}+\frac{LM\rho}{\mu^2}\Big)^2\|y_k^D-y^*(x_k)\|^2 +\frac{2\ell^2M^2(1-\kappa^{-1})^{2Q}}{\mu^2}.
\end{align*}
\end{lemma}

\begin{lemma}\cite{ji2021bilevel}[Lemma 8]\label{le:variancc} 
Suppose Assumptions \ref{ass:stoc} and \ref{ass:variance-bound} hold. Then, we have
\begin{align*}
\mathbb{E}\|\widehat \nabla \Phi(x_k)-\nabla \Phi(x_k)\|^2 \leq &  \frac{4\ell^2M^2}{\mu^2D_g} + \Big(\frac{8\ell^2}{\mu^2} + 2\Big) \frac{M^2}{D_f}+ \frac{16\eta^2  \ell^4M^2}{\mu^2} \frac{1}{B}+\frac{16 \ell^2M^2(1-\kappa^{-1})^{2Q}}{\mu^2} \nonumber
\\&+ \Big( \ell+\frac{\ell^2}{\mu} + \frac{M\rho}{\mu}+\frac{\ell M\rho}{\mu^2}\Big)^2 \mathbb{E}\|y_k^D-y^*(x_k)\|^2.
\end{align*}
\end{lemma}

We now prove Lemma \ref{lem:stocbio-decrease}.

\begin{proof}
By Assumptions \ref{ass:stoc}, $\Phi(x)$ is $L_\phi$ smooth, which yields
\begin{align*}
\Phi(x_{k+1}) \leq & \Phi(x_k)  + \langle \nabla \Phi(x_k), x_{k+1}-x_k\rangle + \frac{L_\phi}{2} \|x_{k+1}-x_k\|^2 \nonumber
\\\leq& \Phi(x_k)  - \beta \langle \nabla \Phi(x_k),\widehat \nabla \Phi(x_k)\rangle + \beta^2 L_\phi \|\nabla\Phi(x_k)\|^2+\beta^2 L_\phi\|\nabla\Phi(x_k)-\widehat\nabla\Phi(x_k)\|^2,
\end{align*}
where the second inequality is by Young's inequality. Taking expectation over the above inequality, we have
\begin{align} \label{eq:expect-phi}
&\mathbb{E}\Phi(x_{k+1})\nonumber \\
 \leq &\mathbb{E}\Phi(x_k)  - \beta \mathbb{E}\langle \nabla \Phi(x_k),\mathbb{E}_k\widehat \nabla \Phi(x_k)\rangle + \beta^2 L_\phi \mathbb{E}\|\nabla\Phi(x_k)\|^2 +\beta^2 L_\phi\mathbb{E}\|\nabla\Phi(x_k)-\widehat \nabla\Phi(x_k)\|^2 \nonumber\\
\leq& \mathbb{E}\Phi(x_k)  +\frac{\beta}{2}\mathbb{E}\|\mathbb{E}_k\widehat \nabla \Phi(x_k)-\nabla \Phi(x_k) \|^2 -\frac{\beta}{4} \mathbb{E}\|\nabla\Phi(x_k)\|^2+\frac{\beta}{4}\mathbb{E}\|\nabla\Phi(x_k)-\widehat \nabla\Phi(x_k)\|^2 \nonumber\\
\leq& \mathbb{E}\Phi(x_k) -\frac{\beta}{4}\mathbb{E}\|\nabla\Phi(x_k)\|^2 +\frac{\beta \ell^2M^2(1-\kappa^{-1})^{2Q}}{\mu^2} \nonumber
\\&+\frac{\beta}{4}
\left(   \frac{4\ell^2M^2}{\mu^2D_g} + \Big(\frac{8\ell^2}{\mu^2} + 2\Big) \frac{M^2}{D_f}+ \frac{16\eta^2  \ell^4M^2}{\mu^2} \frac{1}{B}+\frac{16 \ell^2M^2(1-\kappa^{-1})^{2Q}}{\mu^2}\right) \nonumber
\\&+  \frac{5\beta}{4} \Big( \ell +\frac{L^2}{\mu} + \frac{M\tau}{\mu}+\frac{\ell M\rho}{\mu^2}\Big)^2\mathbb{E}\|y_k^D-y^*(x_k)\|^2.
\end{align}
The second inequality is by Young's inequality and $\beta = \frac{1}{4L_\phi}.$ The last inequality is by Lemmas \ref{le:first_m} and \ref{le:variancc}. Further note that for an integer $t\leq D$
\begin{align}\label{eq:initss}
\|y_k^{t+1}-y^*(x_k) \|^2= &\|y_k^{t+1}-y_k^t\|^2 + 2\langle y_k^{t+1}-y_k^t,  y_k^t-y^*(x_k) \rangle + \| y_k^t-y^*(x_k)\|^2 \nonumber
\\=& \alpha^2\| \nabla_y G(x_k,y_k^{t}; \gS_t) \|^2 -  2\alpha\langle  \nabla_y G(x_k,y_k^{t}; \gS_t),  y_k^t-y^*(x_k) \rangle \nonumber \\
& +\| y_k^t-y^*(x_k)\|^2.
\end{align}
Conditioning on $x_k, y_k^t$ and taking expectation in~\eqref{eq:initss}, we have 
\begin{align}\label{eq:traerr}
&\mathbb{E} [\|y_k^{t+1}-y^*(x_k) \|^2 | x_k, y_k^t ] \nonumber \\
 \leq& \alpha^2 \Big(\frac{\sigma^2}{S} + \|\nabla_y g(x_k,y_k^t)\|^2 \Big) - 2\alpha \langle  \nabla_y g(x_k,y_k^{t}),  y_k^t-y^*(x_k) \rangle +\| y_k^t-y^*(x_k)\|^2 \nonumber
\\\leq&\frac{\alpha^2\sigma^2}{S} + \alpha^2\|\nabla_y g(x_k,y_k^t)\|^2 - 2\alpha\left(  \frac{\ell\mu}{\ell+\mu} \|y_k^t-y^*(x_k)\|^2+\frac{\|\nabla_y g(x_k,y_k^t)\|^2}{\ell+\mu} \right) \nonumber
\\&+\| y_k^t-y^*(x_k)\|^2  \nonumber
\\=& \frac{\alpha^2\sigma^2}{S} - \alpha\left(\frac{2}{\ell +\mu}-\alpha\right)\|\nabla_y g(x_k,y_k^t)\|^2 + \left(1-\frac{2\alpha \ell\mu}{\ell+\mu} \right)\|y_k^t-y^*(x_k)\|^2,
\end{align}
where the first inequality is by Assumption \ref{ass:variance-bound} and the second inequality follows from the strong-convexity (with respect to $y$) and smoothness of the function $g$. Since $\alpha=\frac{2}{\ell +\mu}$, we obtain from~\eqref{eq:traerr} that  
\begin{align}\label{eq:eyk}
\mathbb{E}[\|y_k^{t+1}-y^*(x_k) \|^2| x_k, y_k^t]\leq &\left(\frac{\ell-\mu}{\ell+\mu} \right)^2\|y_k^t-y^*(x_k)\|^2 + \frac{4\sigma^2}{(\ell+\mu)^2S}.
\end{align}
Unconditioning on $x_k, y^t_k$ in \eqref{eq:eyk} and telescoping~\eqref{eq:eyk} over $t$ from $0$ to $D-1$ yield
\be\bad\label{eq:yjtt}
\mathbb{E}\|y_k^{D}-y^*(x_k) \|^2 \leq &  \left(\frac{\ell-\mu}{\ell+\mu}\right)^{2D}\mathbb{E}\|y^0_k-y^*(x_k)\|^2 + \frac{\sigma^2}{\ell\mu S}. \\
\ead\ee 
By setting $D > 1/2 \log(1/4) / \log\left(\frac{1-\kappa}{1+\kappa}\right) = \gO(\kappa)$, we get $\left(\frac{\ell-\mu}{\ell+\mu}\right)^{2D} < 1/4$. Therefore, we have 
\be\bad\label{eq:y-opt-error}
\mathbb{E}\|y_k^{0}-y^*(x_k) \|^2 \leq & 2\mathbb{E}\|y^D_{k-1}-y^*(x_{k-1})\|^2 + 2\mathbb{E}\|y^*(x_k)-y^*(x_{k-1}) \|^2\\
 \le & 2\left(\frac{\ell-\mu}{\ell +\mu}\right)^{2D}\mathbb{E}\|y^0_{k-1}-y^*(x_{k-1})\|^2  + 2\kappa^2 \mathbb{E} \| x_k - x_{k-1}\|^2 + \frac{2\sigma^2}{\ell \mu S}\\
\le & \frac{1}{2} \mathbb{E}\|y^0_{k-1}-y^*(x_{k-1})\|^2 + \left(2\kappa^2 \beta^2 \left(M + \frac{\ell M}{\mu}\right)^2 + \frac{2\sigma^2}{\ell \mu S}\right)\\
\le & \left(\frac{1}{2}\right)^k \mathbb{E}\|y^0_{0}-y^*(x_{0})\|^2 + \sum_{j = 0}^{k-1} \left(\frac{1}{2}\right)^j\left(2\kappa^2\beta^2 \left(M + \frac{\ell M}{\mu}\right)^2 + \frac{2\sigma^2}{\ell \mu S}\right)\\
\le &  \mathbb{E}\|y^0_{0}-y^*(x_{0})\|^2 + 4\kappa^2\beta^2 \left(M + \frac{\ell M}{\mu}\right)^2 + \frac{4\sigma^2}{\ell \mu S}\\
= &  \widehat{\Delta}+ 4\kappa^2 \beta^2 \left(M + \frac{\ell M}{\mu}\right)^2 + \frac{4\sigma^2}{\ell \mu S},
\ead\ee 
where we denoted $\widehat{\Delta} = \|y^0_{0}-y^*(x_{0})\|^2$, the second inequality is true as $y^*(x)$ is $\kappa$-Lipschitz continuous and \eqref{eq:yjtt}, the third inequality is by the fact that $\|\widehat{\nabla}\Phi(x)\|$ can be bounded by $M + \frac{\ell M}{\mu}$. Plug \eqref{eq:yjtt} and \eqref{eq:y-opt-error} into \eqref{eq:expect-phi} yields 
\begin{align} 
&\mathbb{E}\Phi(x_{k+1})  \nonumber\\
\leq& \mathbb{E}\Phi(x_k) -\frac{\beta}{4}\mathbb{E}\|\nabla\Phi(x_k)\|^2 +\frac{\beta \ell^2M^2(1-\kappa^{-1})^{2Q}}{\mu^2} 
\nonumber
\\&+\frac{\beta}{4}
\left(   \frac{4\ell^2M^2}{\mu^2D_g} + \Big(\frac{8\ell^2}{\mu^2} + 2\Big) \frac{M^2}{D_f}+ \frac{16\eta^2  \ell^4M^2}{\mu^2} \frac{1}{B}+\frac{16 \ell^2M^2(1-\kappa^{-1})^{2Q}}{\mu^2}\right) \nonumber
\\&+  \frac{5\beta}{4} \Big( \ell+\frac{\ell^2}{\mu} + \frac{M\tau}{\mu}+\frac{\ell M\rho}{\mu^2}\Big)^2\left(\left(\widehat{\Delta}+ 4\kappa^2 \beta^2 \left(M + \frac{\ell M}{\mu}\right)^2 + \frac{4\sigma^2}{\ell \mu S}\right)\left(\frac{\ell-\mu}{\ell+\mu}\right)^{2D} + \frac{\sigma^2}{\ell\mu S}\right).
\end{align}
Therefore, it suffices to choose the parameters as 
\be\bad
S = O(\kappa^5 \epsilon^{-2}),   D_f= O(\kappa^2 \epsilon^{-2}), D_g = O(\kappa^2 \epsilon^{-2}),\\
B =  O(\kappa^2 \epsilon^{-2}), Q= O(\kappa \log \frac{1}{\epsilon}), D = O(\kappa \log \frac{1}{\epsilon}),
\ead\ee
such that the following inequality holds,
\begin{align} 
&\frac{\beta \ell^2M^2(1-\kappa^{-1})^{2Q}}{\mu^2} +\frac{\beta}{4}
\left(   \frac{4\ell^2M^2}{\mu^2D_g} + \Big(\frac{8\ell^2}{\mu^2} + 2\Big) \frac{M^2}{D_f}+ \frac{16\eta^2  \ell^4M^2}{\mu^2} \frac{1}{B}+\frac{16 \ell^2M^2(1-\kappa^{-1})^{2Q}}{\mu^2}\right) \nonumber
\\&+  \frac{5\beta}{4} \Big( \ell+\frac{\ell^2}{\mu} + \frac{M\tau}{\mu}+\frac{\ell M\rho}{\mu^2}\Big)^2\left(\left(\widehat{\Delta}+ 4\kappa^2 \left(M + \frac{\ell M}{\mu}\right)^2 + \frac{6\sigma^2}{\ell \mu S}\right)\left(\frac{\ell-\mu}{\ell+\mu}\right)^{2D} + \frac{\sigma^2}{\ell\mu S}\right)\nonumber \\
& \le  \frac{1}{400L_\phi} \epsilon^2.
\end{align}
Finally, combining the above two inequalities yields
\be\bad\label{eq:f-dec2-new}
\BE \left[\Phi(x_{k+1}) - \Phi(x_{k}) \right] \le -\frac{\beta}{4}\mathbb{E}\|\nabla\Phi(x_k)\|^2 + \frac{1}{400L_\phi} \epsilon^2.
\ead\ee
\end{proof}

\subsection{Proof of Theorem \ref{thm:stocneon} and Corollary \ref{cor:stocneon}}
\begin{proof}
We consider the following two possible cases
\begin{itemize}
\item \textbf{Case 1:} $\mathbb{E}[\|\nabla\Phi(x_k)\| | x_k ]  > \frac{3}{5}  \epsilon$. Unconditioning on $x_k$, we have $\mathbb{E}\|\nabla\Phi(x_k)\| > \frac{3}{5}  \epsilon$. In this case, Lemma \ref{lem:stocbio-decrease} yields
\[
\BE \left[\Phi(x_{k+1}) - \Phi(x_{k}) \right] \le - \frac{9}{400L_\phi} \epsilon^2 + \frac{1}{400L_\phi} \epsilon^2 = - \frac{1}{50L_\phi} \epsilon^2
\]
and the total iteration number of \textbf{Case 1} can be bounded by 
\be\label{thm45-bound1}
\frac{50 L_\phi (\Phi(x_{0}) - \Phi^* )}{\epsilon^2}.
\ee
\item \textbf{Case 2:}  $\mathbb{E}[\|\nabla\Phi(x_k)\| | x_k ]  \le \frac{3}{5}  \epsilon$, which indicates $$\|\nabla\Phi(x_k)\| =\|\mathbb{E}[\nabla\Phi(x_k)|x_k]\| \le \mathbb{E}[\|\nabla\Phi(x_k)\| | x_k]  \le \frac{3}{5} \epsilon.$$ Unconditioning on $x_k$, we have $\mathbb{E}\|\nabla\Phi(x_k)\| \le \frac{3}{5}  \epsilon$. In this case we run AID in line 12 of Algorithm \ref{alg:INEONstocbio} such that $\|\widehat\nabla \Phi_\gD (x_k) - \nabla \Phi_\gD (x_k)\| \le \frac{1}{10}\epsilon$. According to Lemma \ref{lem:grad-err-aid}, it only requires us to set $D = \gO(\kappa \log \epsilon^{-1})$ and $N = \gO(\sqrt{\kappa} \log \epsilon^{-1})$ in the AID. Moreover, combining with Lemma \ref{lem:batch-grad-err}, we have 
\be\bad\label{thm45-bound-case2-b1}
\|\widehat\nabla \Phi_\gD (x_k)\| \le & \| \nabla \Phi_\gD (x_k) \| + \|\widehat\nabla \Phi_\gD (x_k) - \nabla \Phi_\gD (x_k)\|\\
\le &\| \nabla \Phi (x_k) \| + \| \nabla \Phi_\gD (x_k) -  \nabla \Phi (x_k) \| + \|\widehat\nabla \Phi_\gD (x_k) - \nabla \Phi_\gD (x_k)\|\\
\le & \frac{4}{5}\epsilon.
\ead\ee
Therefore, Algorithm \ref{alg:iNEON} is called with high probability. When the if condition in line 13 of Algorithm \ref{alg:INEONstocbio} is satisfied, we can guarantee 
\be\bad\label{thm45-bound-case2-b2}
\| \nabla \Phi (x_k) \| \le  &\| |\widehat\nabla \Phi_\gD (x_k) \| + \| \nabla \Phi_\gD (x_k) -  \nabla \Phi (x_k) \| + \|\widehat\nabla \Phi_\gD (x_k) - \nabla \Phi_\gD (x_k)\| \le & \epsilon.
\ead\ee
Moreover, when Algorithm \ref{alg:iNEON} is called and it returns a nonzero vector $u_{out}$, combining Lemmas \ref{lem:neon} and \ref{lem:saddle-decrease} yields
\[\mathbb{E} \Phi(x_{k+1}) - \mathbb{E} \Phi(x_k) \le - \frac{1}{3 \cdot 80^3\logt^3}\cdot \sqrt{\frac{\epsilon^3}{\rho_\phi}}.\]
So the total iteration number in this case is bounded by
\be\label{thm45-bound2}
\frac{3\cdot 80^3\logt^3(\Phi(x_0) - \Phi(x^*)) \utime}{ \sqrt{\epsilon^3/ \rho_\phi} }.
\ee
If Algorithm \ref{alg:iNEON} returns a zero vector $u_{out}$, then with high probability we find an $\epsilon$-local minimum.
\end{itemize}
Therefore, combining \eqref{thm45-bound1} and \eqref{thm45-bound2} we know that the total iteration number before we visit an $\epsilon$-local minimum can be bounded by 
\[
K = \frac{3\cdot 80^3\logt^3(\Phi(x_0) - \Phi(x^*)) \utime}{ \sqrt{\epsilon^3/ \rho_\phi} } + \frac{50L_\phi(\Phi(x_0) - \Phi(x^*))}{ \epsilon^2} = \tilde{\gO} \left( \kappa^3 \epsilon^{-2}\right).
\]
We then prove Corollary \ref{cor:stocneon}. Note that we require the sample batch sizes to be 
\be\bad\label{eq:stocneon-batch1}
D_f = \tilde{\gO} \left(\kappa^2 \epsilon^{-2}\right), \quad D_g = \tilde{\gO} \left(\kappa^6 \epsilon^{-2}\right),
\ead\ee
so that Lemma \ref{lem:batch-grad-err} and Lemma \ref{lem:batch-hess-err} hold, and 
\be\bad\label{eq:stocneon-batch2}
&S = O(\kappa^5 \epsilon^{-2}), D_f = O(\kappa^2 \epsilon^{-2}), D_g = O(\kappa^2 \epsilon^{-2}),\\
&Q = O(\kappa \log \epsilon^{-1}), D = O(\kappa \log \epsilon^{-1}), B =  O(\kappa^2 \epsilon^{-2}), 
\ead\ee
so that Lemma \ref{lem:stocbio-decrease} holds. Combining \eqref{eq:stocneon-batch1} and \eqref{eq:stocneon-batch2}, we have for the iNEON calls in Algorithm \ref{alg:INEONstocbio} (Line 12 - 21), the gradient, Jacobian- and Hessian-vector product complexities are 
\be\bad\label{eq:ineon-sample-com}
&G(f, \epsilon) = \tilde{\gO}(\kappa^5 \epsilon^{-4}), \quad G(g, \epsilon) = \tilde{\gO}(\kappa^{10} \epsilon^{-4}),\\
 &JV(g, \epsilon) = \tilde{\gO}(\kappa^{9} \epsilon^{-4}), \quad HVc(g, \epsilon) = \tilde{\gO}(\kappa^{9.5} \epsilon^{-4}).
\ead\ee
For those stocBiO iterations in Algorithm \ref{alg:INEONstocbio} (Line 3-11), the gradient, Jacobian- and Hessian-vector product complexities are 
\be\bad\label{eq:stoc-sample-com}
&G(f, \epsilon) = \tilde{\gO}(\kappa^5 \epsilon^{-4}), \quad G(g, \epsilon) = \tilde{\gO}(\kappa^{9} \epsilon^{-4}), \\
&JV(g, \epsilon) = \tilde{\gO}(\kappa^{5} \epsilon^{-4}), \quad HVc(g, \epsilon) = \tilde{\gO}(\kappa^6 \epsilon^{-4}).
\ead\ee
We take the maximum between \eqref{eq:ineon-sample-com} and \eqref{eq:stoc-sample-com}, which completes the proof of Corollary \ref{cor:stocneon}.
\end{proof}

\end{document}